\documentclass[10pt]{article}

\usepackage[left=2.5cm,bottom=2.5cm,right=2.5cm,top=2.5cm]{geometry}

\usepackage{amsmath}
\usepackage{amssymb}
\usepackage{amsfonts}
\usepackage{multirow}
\usepackage{amsthm}

\newtheorem{theorem}{Theorem}
\newtheorem{lemma}{Lemma}
\newtheorem{corollary}{Corollary}
\newtheorem{proposition}{Proposition}
\theoremstyle{definition}

\newtheorem{remark}{Remark}

\usepackage{float}
\usepackage{array}
\usepackage[caption=false,font=normalsize,labelfont=sf,textfont=sf]{subfig}
\usepackage[ruled,vlined,linesnumbered]{algorithm2e}
\usepackage{textcomp}
\usepackage{stfloats}
\usepackage{url}
\usepackage{verbatim}
\usepackage{graphicx}
\usepackage{mathtools}
\usepackage{caption}
\usepackage{diagbox}
\usepackage{cases}
\usepackage{enumerate}
\usepackage{tabularx}
\usepackage{multirow}
\usepackage{bm}
\usepackage{authblk}
\usepackage{indentfirst}
\usepackage{setspace}
\usepackage{color}
\usepackage[utf8]{inputenc} 
\usepackage[T1]{fontenc}    
\usepackage{hyperref}       
\usepackage{url}            
\usepackage{booktabs}       
\usepackage{amsfonts}       
\usepackage{nicefrac}       
\usepackage{microtype}      
\usepackage{lipsum}		
\usepackage{graphicx}
\usepackage{doi}
\usepackage{amssymb}                  
\usepackage{mathrsfs} 
\usepackage[resetlabels]{multibib}
\usepackage[bottom,perpage,flushmargin]{footmisc}
\newtheorem{assumption}{Assumption}

\usepackage[utf8]{inputenc}
\usepackage{natbib}
\SetKwInput{KwData}{Input}
\SetKwInput{KwResult}{Output} 

\usepackage[table]{xcolor}

\usepackage{tikz}
\usetikzlibrary{patterns}
\usetikzlibrary{decorations.pathreplacing,calligraphy}
\usetikzlibrary{patterns.meta}
\usetikzlibrary{arrows.meta, positioning, shapes.multipart, calc, fit}
\usetikzlibrary{shapes.geometric}
\usetikzlibrary{shadows}
\usetikzlibrary{intersections,calc}

\begin{document}

\title{Adversarial Contamination Meets Hard Thresholding: An Iterative Algorithm with Signal Adaptivity and Minimax Optimality}

\author[1]{Shixiang Liu}
\author[2]{Hanming Yang}

\affil[1]{\footnotesize School of Statistics, Renmin University of China}
\affil[2]{\footnotesize Institute of Statistics and Big Data, Renmin University of China}

\date{}
\maketitle \sloppy

\begin{abstract}
Pervasive data contamination---stemming from measurement errors, outliers, or adversarial corruption---has motivated the development of robust statistical methods. 
In this context, we propose a two-stage Adversarial Contamination-resistant Iterative Hard Thresholding (AC-IHT) algorithm for high-dimensional regression with contamination.
Our nonconvex algorithm achieves minimax near-optimal (up to logarithmic terms) estimation by iteratively updating the coefficient vector and the contamination vector with different thresholding scales.
We further demonstrate that our AC-IHT estimator is signal-adaptive: under proper signal conditions, it adaptively attains a sharper estimation rate and more accurate support recovery.
Moreover, it enjoys the strong oracle property, laying a theoretical foundation for asymptotic inference.
Numerical experiments confirm its superior finite-sample performance.
Finally, we discuss theoretical extensions of the proposed procedure to generalized linear models and to heavy-tailed noise settings.
\end{abstract}


\section{Introduction}

Adversarial contamination has become a significant concern in both statistical theory and its applications.
Models that explicitly address contamination yield enhanced robustness, improved generalization, and a more realistic representation of real-world data.
From a statistical perspective, this paper considers a basic high-dimensional linear model with adversarial contamination, given by
\begin{equation}\label{adv}
    Y = X \beta^* + \sqrt n \theta^* + \xi ,
\end{equation}
where $Y\in \mathbb R^n$ is the response vector, $X\in\mathbb R^{n\times p}$ is the design matrix, and $\beta^*\in\mathbb R^p$ is the coefficient vector.
The term $\sqrt n \theta^*\in \mathbb R^n$ represents an adversarial contamination that models potentially malicious or outlying observations in the response vector.
The $\sqrt n$ factor is introduced for technical convenience, ensuring that the columns of the augmented design matrix $(X \mid \sqrt{n} I_n) \in \mathbb R^{n \times (p+n)}$ are of comparable magnitude.
The noise term $\xi$ is an independent centered $\sigma$-subGaussian random vector, i.e., $\mathbf{E}(e^{\lambda^\top \xi})\leq \exp(\sigma^2\|\lambda\|_2^2/2)$ for all $\lambda \in \mathbb R^n$, and $\xi$ is independent of $X$.  
We assume that $\beta^*$ is $s$-sparse and $\theta^*$ is ${\tilde o}$-sparse, i.e., $\| \beta^{*} \|_0 = s < p$ and $\| \theta^{*} \|_0 = {\tilde o }< n$. 

Such a formulation provides a unified framework that can potentially benefit various domains, including robust principal component analysis with missing data \citep{chen2021bridging}, figure classification with covariate-shifted samples \citep{heng2025detecting}, and econometric models with heterogeneous treatment effects \citep{pinkham2024contamination}.
To address model \eqref{adv} in the high-dimensional regime $p\gtrsim n$, this paper introduces a two-stage \textbf{A}dversarial \textbf{C}ontamination-resistant \textbf{I}terative \textbf{H}ard \textbf{T}hresholding (AC-IHT) algorithm and analyzes how the signal strengths of $\beta^*$ and $\theta^*$ affect the statistical inference of $\beta^*$.

\subsection{Literature review}\label{Related work/Literature review}
\noindent
{\bf Adversarial contamination~}
To achieve robust statistical inference, one often treats the contamination as additional covariates, thereby reducing its influence \citep{sardy2001robust,gannaz2007robust}. 
In high-dimensional settings, \citet{nguyen2012robust} applied simultaneous $\ell_1$ penalties to both coefficient $\beta$ and contamination $\theta$, deriving a joint error bound.  
Alternative penalty forms have also been explored: \citet{she2011outlier} introduced nonconvex penalties for outlier detection, \citet{lee2012regularization} studied $\ell_2$ regularization on contamination within regression and classification frameworks, and \citet{Dehan2018} proposed an adaptive $\ell_1$ penalty with weights determined from an initial robust fit.
From an algorithmic viewpoint, \citet{bhatia2015robust,bhatia2017consistent,suggala2019adaptive} analyzed a series of iterative hard thresholding methods under sparse contamination, establishing their convergence properties.

Recent work on adversarial contamination has shifted toward the non-asymptotic minimax estimation rates.
Following the general minimax theory for $\epsilon$-Huber contamination model proposed by \citet{gaochao16ejs, chengaoren18aos}, \citet{gaochao20bernoulli} introduced a multivariate-depth estimator that is minimax-optimal (but not efficiently computable). 
To bridge this gap, efficient $\ell_1$-regularized methods achieving near-optimal rates were developed \citep{dalalyan19nips, Chinot20EJS, sasai2020robust}. 
Subsequent contributions produced estimators that adapt to the noise level $\sigma$, sparsity $s$, and contamination $o$ \citep{finocchio2021robust, Ndaoud24slope}.
There are also advanced extensions considering non-sparse settings \citep{PolingLoh24jasa, hammouda2024outlier}, low-rank matrix regression under contamination \citep{thompson2020outlier, Xia2025AOS}, and models where both covariates and responses are contaminated \citep{sasai2025outlier}. 
For a comprehensive overview, we refer readers to \citet{loh25review}.

\noindent
{\bf Iterative hard thresholding and signal adaptivity~}
Iterative optimization has emerged as a central analysis paradigm in statistics \citep{jain2014iterative, huang18jmlr, she2021Bregman} that goes beyond the traditional Empirical Risk Minimization (ERM) framework. This is because the iteration trajectory itself acts as a form of regularization, such as early stopping \citep{Fan2023implicit} and iterative thresholding \citep{blumensath2008iterative, BLUMENSATH2009265, liu2020between}, and provides a finer balance between computational efficiency and statistical precision.
Furthermore, iterative schemes provide superior flexibility by allowing for the dynamic tuning of hyperparameters, such as learning rates \citep{she2023slow, Xia2025AOS}, truncation thresholds \citep{NM20}, compression coefficients \citep{she2023slow}, and so on.
As a widely adopted iterative scheme, Iterative Hard Thresholding (IHT, \citet{blumensath2008iterative}) is particularly noted for applying no additional shrinkage to the selected signals. Notably, \citet{NM20} demonstrated that IHT exhibits a remarkable \textbf{signal adaptivity} property: if the signal scale satisfies $\min_{i : \beta^*_i \ne 0} |\beta_i| \ge C \sigma \sqrt{n^{-1} \log(p/s) }$, the IHT estimator adaptively achieves an $\ell_2$-estimation rate of order $\sigma\sqrt{s/n}$. This refined rate not only improves upon the standard high-dimensional minimax rate $\sigma\sqrt{n^{-1}  s \log(p/s) }$ \citep{Raskutti2011minimax}, but also aligns with the minimax phase transition phenomenon \citep{MN19}.

\subsection{Motivation and inspiration}\label{Motivation and inspiration}
In the high-dimensional contamination setting, most existing works established minimax (near) optimal estimations \citep{dalalyan19nips, finocchio2021robust, She22outlier, Ndaoud24slope, Xia2025AOS}, without studying how the signal strengths of $\beta^*$ and $\theta^*$ affect the estimation accuracy. 
Moreover, support recovery and asymptotic distributions remain unexplored: \citet{MINSKER2025SPL} proved support recovery for the Lasso estimator under the incoherence condition, but did not address its asymptotic behavior.
Therefore, we address the following questions:
\begin{quote}
\emph{In the high-dimensional contamination model \eqref{adv}, do the signal strengths of $\beta^*$ and $\theta^*$ influence the recovery of $\beta^*$?
Does the IHT estimator possess signal-adaptivity to $\beta^*$ estimation, and retain the strong oracle property?
Are these results supported by minimax-optimal guarantees?}
\end{quote}
The following works inspire our study. 
From the upper-bound perspective, \citet{dalalyan19nips, Ndaoud24slope} showed that consistent estimation of $\beta^*$ is possible only if the columns of $X\in\mathbb{R}^{n\times p}$ are ``nearly uncorrelated'' with those of the identity matrix $I_n$.
This insight motivates our incoherence-type requirement in Proposition \ref{prop: rii}. 
From the lower-bound perspective, \citet{chengaoren18aos, Chinot20EJS} described how to construct least-favorable distributions specific to contamination settings; in the uncontaminated sparse model, \citet{butucea18AOS} derived a minimax lower bound via a Bayesian-risk approach, which crucially guides our lower bound construction.

\subsection{Main contributions}
We affirmatively address the questions posed in Section \ref{Motivation and inspiration}.
The main contributions of this paper are threefold:
\begin{itemize}
\item \textbf{Signal adaptivity and minimax near-optimality}~
We propose a two-stage AC-IHT algorithm for sparse linear regression under adversarial contamination. 
We establish that our estimator $\tilde \beta$ exhibits signal adaptivity, and explicitly characterize how the signal strengths of $\beta^*$ and $\theta^*$ affect the statistical inference of $\beta^*$, as summarized in Table \ref{table: Theoretical Result}.
Our results are supported by {minimax near-optimal} guarantees.

\item \textbf{Strong oracle property}~
Under proper signal strength conditions, we prove that the AC-IHT algorithm converges to the oracle estimator, enabling exact support recovery of $\beta^*$. 
Furthermore, we establish asymptotic normality for our estimator $\tilde \beta$, providing a foundation for valid inference.

{\color{black}
\item \textbf{Additional theoretical extensions}~ We extend the AC-IHT algorithm to generalized linear models, accommodating a diverse range of response distributions.
Furthermore, we establish a theoretical connection between adversarial contamination and heavy-tailed noise, demonstrating that AC-IHT achieves minimax near optimality in both settings. 
}

\end{itemize}

\begin{table}[htbp]
\centering
\caption{The signal adaptivity of the AC-IHT estimator $\tilde{\beta}$ in signal estimation and support recovery.}\label{table: Theoretical Result}
\resizebox{\linewidth}{!}{
\begin{tabular}{ccc}
\hline
\diagbox[width=4.5cm,height=2cm]{Signal \\ Condition}{Theoretical\\  Result} & \parbox[c][2.5cm][c]{5.5cm}{\centering \textbf{Convergence Rate} \\ \ \\ $\|\tilde{\beta} - \beta^* \|_2$} &  \parbox[c][2.5cm][c]{5cm}{\centering \textbf{Support Recovery}} \\
\hline
\parbox[c][2.0cm][c]{4.5cm}{\centering None} & \parbox[c][2.0cm][c]{5.5cm}{\centering {\color{black}$\sigma \left( \sqrt{ \frac{s \log p}{n}} + \frac{o \log n}{n}\right)$} \\ \ \\ \centering Theorem \ref{th2: Signal adaptive estimation}} & \parbox[c][2.0cm][c]{5cm}{\centering {\small $\left | \text{supp}(\beta^*)\  \triangle \ \text{supp}(\tilde{\beta})\right | = O(s)$}\\ \  \\ \centering Theorem \ref{th2: Signal adaptive estimation}} \\
\parbox[c][2.0cm][c]{4.5cm}{\centering With Signal Condition :\\ \  \\ $\min_{i : \beta^*_i \ne 0} |\beta_i^*| \ge \beta^{\ddag}$} & \parbox[c][2.0cm][c]{5.5cm}{\centering {\color{black} $\sigma \left( \sqrt{\frac{s + \log(1/\varrho)}{n}} + \frac{s \log p+ o \log n}{n} \right)$} \\ \  \\ \centering Theorem \ref{th2: Signal adaptive estimation}}  & \parbox[c][2.0cm][c]{5cm}{\centering {\small $\left | \text{supp}(\beta^*)\  \triangle \ \text{supp}(\tilde{\beta})\right | \prec s$}\\ \ \\ \centering Corollary \ref{co1: almost full recovery}} \\
\parbox[c][2.0cm][c]{4.5cm}{\centering With Signal Conditions: \\ $\min_{i : \beta^*_i \ne 0} |\beta_i^*| \ge \beta^{\ddag},$\\$\min_{k : \theta^*_k \ne 0} |\theta_k^*| \ge \theta^{\ddag}$ } & \parbox[c][2.0cm][c]{5cm}{\centering {\color{black}$\sigma \sqrt{\frac{s+ \log(1/\varrho) }{n-o}}$ } \\ \ \\ \centering Theorem \ref{th3: Strong Oracle estimation}} & \parbox[c][2.0cm][c]{5cm}{\centering {\small $\text{supp}(\tilde{\beta}) = \text{supp}(\beta^*)$}\\ \ \\ \centering Theorem \ref{th3: Strong Oracle estimation}} \\
\hline
\end{tabular}
}
\parbox{\linewidth}{\footnotesize
\raggedright
\textit{Note:}
Here $\beta^{\ddag} := C_\beta   \sigma \left\{\frac{\log p}{n} +  \frac{o^2 \log^{2 }n}{ n^2 s }\right\}^{1/2}$, $\theta^{\ddag} := C_\theta   \sigma \left\{\frac{\log n}{n} +  \frac{s^2 \log^2 p}{n^2 o}\right\}^{1/2}$, where {\color{black}$o= \|\theta^* \|_0 \vee 1$}, and $C_\beta,C_\theta>0 $ are some absolute constants.
{\color{black} All results hold with a probability greater than $1- \varrho - O(p^{-2} + n^{-3})$.}
} 
\end{table}
However, solely estimating $\beta^*$ is far more challenging than jointly estimating $ \beta^{*}$ and $\theta^{*}$.
To overcome this, we innovatively employ separate thresholding levels for $\tilde \beta$ and $\tilde \theta$ in each iteration. 
This separation enables a refined analysis of how the nuisance component $\theta^*$ impacts the estimation of $\beta^*$, and yields both delicate sparsity patterns and sharp $\ell_2$-error bounds for our estimator $\tilde\beta$.

\subsection{Organization of the paper}
The present paper is organized as follows. 
Section~\ref{section:algorithm} establishes the procedure of the two-stage AC-IHT algorithm. 
Section~\ref{section:Theoretical guarantee} presents the theoretical guarantees of our algorithm.
Section~\ref{section:Numerical experiment} presents numerical experiments that illustrate our theoretical findings.
Section~\ref{section:Discussion} discusses extensions of AC-IHT to generalized linear models, its connection to heavy-tailed regression, and related future directions.
Detailed proofs and additional simulations are available in the supplementary material.

\textbf{Notation}
For sequences $a_n$ and $b_n$, we write $a_n = O(b_n)$ (or $a_n \lesssim b_n$) if $a_n \le Cb_n$ for some constant $C>0$ and all large $n$, 
and $a_n \prec b_n$ if $a_n/b_n \to 0$ as $n \to \infty$.
We write that $a_n \asymp b_n$ if $a_n = O(b_n)$ and $b_n  = O(a_n)$.
Let $[m] = \{1,2,\dots,m\}$, and $\mathbf{1} (\cdot)$ be the indicator function. Define $x \vee y = \max\{x,y\}$. Let $S^* = \{i: \beta^*_i \ne 0\} \subseteq [p]$ and $O^* = \{k: \theta^*_k \ne 0\} \subseteq [n]$ denote the support sets of $\beta^*$ and $\theta^*$, respectively.
For sets $A$ and $B$ with sizes $|A|$ and $|B|$, let $\beta_A = (\beta_j)_{j \in A} \in \mathbb{R}^{|A|}$, $X_{\cdot, A} = (X_j)_{j \in A} \in \mathbb{R}^{n \times |A|}$, and $X_{A,B} \in \mathbb{R}^{|A| \times |B|}$ be the submatrix of $X \in \mathbb R^{ n \times p}$ with rows and columns in $A$ and $B$. 
The symmetric difference of $A$ and $B$ is defined as $A \triangle B = (A \setminus B) \cup (B \setminus A)$.
For a vector $\beta$, denote $\|\beta\|_2$ as its Euclidean norm, $\|\beta\|_0$ as the number of its nonzero entries, and $\operatorname{supp}(\beta)$ as its support set. For matrix $X$, denote $\| X \|_2 $ as its operator norm.

\section{Two-stage AC-IHT Algorithm}\label{section:algorithm}
\setcounter{equation}{0}
This section presents the two-stage AC-IHT algorithm, with its first and second stages detailed in Sections \ref{subsection:First stage} and \ref{subsection:Section stage}, respectively. 
The first-stage procedure provides an initial estimator with a near-optimal estimation rate. 
The second-stage algorithm refines this estimate to obtain a final estimator with the desirable theoretical properties summarized in Table \ref{table: Theoretical Result}.

\subsection{The first stage: dynamic thresholding iteration}\label{subsection:First stage}

We start by defining the squared $\ell_2$ loss for model \eqref{adv} as
$$
L(\beta,\theta) = \frac1{2n}\big\lVert Y - X\beta - \sqrt{n}\,\theta\big\rVert_2^2.
$$
Based on $L(\beta,\theta)$, we propose the first-stage AC-IHT Algorithm \ref{alg1}. 
Each iteration in Algorithm \ref{alg1} can be summarized in three steps: \textbf{gradient update}, \textbf{threshold parameters update}, and \textbf{hard thresholding operation}.

\textbf{Gradient update:} We derive the partial derivatives of $L(\beta,\theta)$ with respect to $\beta$ and $\theta$ as
\begin{align*}
    \frac{\partial L}{\partial \beta} (\beta^t, \theta^t) &= -\frac1n X^\top\bigl(Y - X\beta^t - \sqrt{n}\,\theta^t\bigr),\\
    \frac{\partial L}{\partial \theta} (\beta^t, \theta^t) &= -\frac1{\sqrt{n}}\,\bigl(Y - X\beta^t - \sqrt{n}\,\theta^t\bigr),
\end{align*}
and employ the gradient descent approach to update the parameters:
\begin{equation}\label{alg1:step 1}
\textbf{Step 1.1: }\quad
\begin{aligned}
    H^{t+1}_{\beta} &\gets \beta^t - {\color{black} \eta} \frac{\partial L}{\partial \beta} (\beta^t, \theta^t),
    \\ H^{t+1}_{\theta} &\gets \theta^t - {\color{black} \eta} \frac{\partial L}{\partial \theta} (\beta^t, \theta^t),
\end{aligned}
\end{equation}
where ${\color{black} \eta} > 0$ denotes the learning rate, the explicit choice of which will be specified in Theorem \ref{th1: Initial estimation}.

\textbf{Threshold parameters update:} 
For $\lambda>0$, define the hard thresholding operator $\mathcal{T}^m_{\lambda}:\mathbb R^m\to\mathbb R^m$, such that
\begin{equation}\label{eq: ht}
\Big(\mathcal{T}^m_{\lambda}(z)\Big)_j = z_j \times \mathbf{1}\left( |z_j|\geq \lambda \right), \quad \text{for every } z\in\mathbb R^m \text{ and } j \in [m].
\end{equation}
In Algorithm \ref{alg1}, the thresholding parameters in operator $\mathcal{T}^m_{\lambda}$ are dynamically updated at each iteration.
Specifically, the thresholds $\lambda_{\beta,0}$ and $\lambda_{\theta,0}$ are initialized to sufficiently large values.
These thresholds are then iteratively decreased and used at each iteration for hard-thresholding operations, until they reach their respective universal statistical levels $\lambda_{\beta,\infty}$ and $\lambda_{\theta,\infty}$ (see Theorem \ref{th1: Initial estimation} for their specific forms).
The following scheme guarantees that the sequences $\{ \lambda_{\beta,t} \}_{t \ge 0}$ and $\{ \lambda_{\theta,t} \}_{t \ge 0}$ are monotonically non-increasing:
\begin{equation}\label{alg1:step 2}
\textbf{Step 1.2: }\quad
    \begin{aligned}
        \lambda_{\beta,t+1} &\gets (\kappa \times\lambda_{\beta,t}) \vee\lambda_{\beta,\infty},\\
        \lambda_{\theta,t+1} &\gets (\kappa \times\lambda_{\theta,t}) \vee\lambda_{\theta,\infty}.
    \end{aligned}
\end{equation}
Here, $\kappa \in (0,1)$ controls the decay rate, typically chosen as $0.9$ in practice.
This process not only provides an explicit stopping time for the procedure, but also enables sparsity control in each iteration.

\textbf{Hard thresholding operation:} 
In this step, we apply the hard-thresholding operator \eqref{eq: ht}, together with the updated thresholds \eqref{alg1:step 2}, to the raw updates in \eqref{alg1:step 1}, ensuring a dynamic regularization at each iteration: 
\begin{equation}\label{alg1:step 3}
\textbf{Step 1.3: }\quad
    \begin{aligned}
        \beta^{t+1} &\gets \mathcal{T}^p_{\lambda_{\beta,t+1}}(H^{t+1}_{\beta}),\\
        \theta^{t+1} &\gets \mathcal{T}^n_{\lambda_{\theta,t+1}}(H^{t+1}_{\theta}).
    \end{aligned}
\end{equation}

\begin{algorithm}[t!]
\setstretch{1.5}
\SetAlgoLined
\KwData{$\beta^0 = \mathbf 0_p,\ \theta^0 = \mathbf 0_n,\ Y,\ X,\ \lambda_{\beta,0},\ \lambda_{\theta,0},\ \lambda_{\beta,\infty},\ \lambda_{\theta,\infty},\ \kappa,\ t = 0$}
\KwResult{$\widehat{\beta},\ \widehat{\theta}$}
\While{$t < \max\left\{\log_{1/\kappa}(\lambda_{\beta,0}/\lambda_{\beta,\infty}),\log_{1/\kappa}(\lambda_{\theta,0}/\lambda_{\theta,\infty})\right\}$}{
    \textbf{Step 1.1:} Update $H^{t+1}_{\beta}$ and $H^{t+1}_{\theta}$ using \eqref{alg1:step 1}\;
    \textbf{Step 1.2:} Update $\lambda_{\beta,t+1}$ and $\lambda_{\theta,t+1}$ using \eqref{alg1:step 2}\;
    \textbf{Step 1.3:} Update $\beta^{t+1}$ and $\theta^{t+1}$ using \eqref{alg1:step 3}\;
    $t \gets t + 1$\;
}
    $\widehat{\beta}\gets\beta^{t}, \quad \widehat{\theta}\gets\theta^{t}$ 
\caption{The first stage of the AC-IHT algorithm}
\label{alg1}
\end{algorithm}

Steps 1.2 and 1.3 are inspired by the strategy of \citet{NM20}, using gradually decreasing hard-threshold levels to limit variable inclusion in each iteration.
It also aligns with the motivation of the LARS algorithm \citep{efron04lLARS}.
This approach ensures computational efficiency and sparsity in the outputs.
A key novelty of our Algorithm \ref{alg1} is the separate handling of $\beta^t$ and $\theta^t$ updates, which allows us to derive a delicate sparse pattern and a {\color{black} minimax near-optimal} $\ell_2$-error bound for $\widehat\beta$, as shown in Theorem \ref{th1: Initial estimation}. 
In contrast, choosing common threshold levels $\lambda_{\beta,\infty} \asymp \lambda_{\theta, \infty} \asymp \sigma \{\log(p+n)/n\}^{1/2}$ may not deliver comparably precise estimates.

\subsection{The second stage: fixed thresholding
iteration}\label{subsection:Section stage}

While the first-stage Algorithm \ref{alg1} delivers a {\color{black}minimax near-optimal} initial estimate $\widehat{\beta}$ (shown in Theorem \ref{th1: Initial estimation}), it tends to omit some true support variables and falls short in estimation accuracy in practice, even when the support signals of $\beta^*$ are relatively strong.
Therefore, a 
refinement step is required: Starting from $\widehat{\beta}$ and $\widehat{\theta}$, we execute successive iterations, as detailed in Algorithm \ref{alg2}.
The second-stage Algorithm \ref{alg2} differs from Algorithm \ref{alg1} in two ways:
First, it initializes at $\tilde{\beta}^0 = \widehat{\beta}$ and $\tilde{\theta}^0= \widehat{\theta}$, the estimators produced by the first-stage Algorithm \ref{alg1}.
Second, instead of updating the threshold parameters in each iteration, in the second stage, we use two \textbf{fixed} values $\lambda_{\beta}$ and $\lambda_{\theta}$ (see Theorem \ref{th2: Signal adaptive estimation} for their specific forms).
Each iteration in the second stage consists of two steps: a {\bf gradient update} (identical to \textbf{Step 1.1} \eqref{alg1:step 1}), and a {\bf hard thresholding operation} using the fixed thresholds:
\begin{equation}\label{alg2:step2}
\textbf{Step 2.2:}\quad 
    \begin{aligned}
        \tilde{\beta}^{t+1} &\gets \mathcal{T}^p_{\lambda_{\beta}}(H^{t+1}_{\beta}),\\
        \tilde{\theta}^{t+1} &\gets \mathcal{T}^n_{\lambda_{\theta}}(H^{t+1}_{\theta}).
    \end{aligned}
\end{equation}

The complete second-stage algorithm is presented in Algorithm \ref{alg2}, and our two-stage AC-IHT algorithm combines Algorithm \ref{alg1} (initial estimation) and Algorithm \ref{alg2} (refinement):  
{\color{black} Algorithm~\ref{alg1} iteratively updates the estimator while decreasing the thresholding levels $\lambda_{\beta,t}$ and $\lambda_{\theta,t}$ until they reach their limiting values $\lambda_{\beta,\infty}$ and $\lambda_{\theta,\infty}$. Algorithm~\ref{alg2} then continues the iteration with fixed thresholding levels, yielding a refined final output. In this sense, the two algorithms together form a two-step ``debiased'' procedure, which does not require data splitting.
}

\begin{algorithm}[t!]
\setstretch{1.5}
\SetAlgoLined
\KwData{$\tilde{\beta}^0 = \widehat{\beta},\ \tilde{\theta}^0 = \widehat{\theta},\ Y,\ X,\ \lambda_{\beta},\ \lambda_{\theta},\ t = 0$}
\KwResult{$\tilde{\beta},\ \tilde{\theta}$}
\While{$t \le  C \log n $}{
    \textbf{Step 2.1:} Update $H^{t+1}_{\beta}$ and $H^{t+1}_{\theta}$ using \eqref{alg1:step 1}\;
    \textbf{Step 2.2:} Update $\tilde{\beta}^{t+1}$ and $\tilde{\theta}^{t+1}$ using \eqref{alg2:step2}\;
    $t \gets t + 1$\;
}
    $\tilde{\beta}\gets\tilde{\beta}^{t}, \quad \tilde{\theta}\gets\tilde{\theta}^{t}$
\caption{The second stage of the AC-IHT algorithm}
\label{alg2}
\end{algorithm}

\begin{remark}[Practical Role of Algorithm~\ref{alg1}]
The second-stage Algorithm \ref{alg2} performs a refined bias correction on initial estimates $\widehat \beta$.
Moreover, any initial estimator satisfying the guarantees of Theorem \ref{th1: Initial estimation}, such as the Lasso estimator \citep{dalalyan19nips} or the square-root Slope estimator \citep{Ndaoud24slope}, can be used as input to Algorithm \ref{alg2}.
This demonstrates the generality of the proposed IHT framework. 
{\color{black}
However, Algorithm~\ref{alg1} still has its practical advantages: It provides a computationally efficient initial estimation, while the refinement of estimation accuracy is handled by Algorithm~\ref{alg2}. 
Moreover, as shown in Theorem \ref{th2: Signal adaptive estimation}, the limiting thresholds $\lambda_{\beta,\infty}$ and $\lambda_{\theta,\infty}$ used in Algorithm~\ref{alg1} can be directly used in Algorithm~\ref{alg2}, avoiding additional tuning and reducing computational cost.
}
\end{remark}

\section{Theoretical Guarantees}\label{section:Theoretical guarantee}
\setcounter{equation}{0}

In this section, we study the statistical properties of the two-stage AC-IHT algorithm.
{\color{black} Define $o :=\| \theta^* \|_0\vee 1$, and assume $\| \beta^*\|_0 =s \ge 1$ without loss of generality. 
Recall that the noise $\xi$ in \eqref{adv} is assumed to be $\sigma$-subGaussian.}
We further impose two key assumptions.

{\color{black}
\begin{assumption}[Sub-Gaussian design]\label{assumption: RIP}
The design matrix $X\in\mathbb{R}^{n\times p}$ is row-wise independent and sub-Gaussian: each row $X_{i, \cdot} \overset{d}{=} Z_{i} \Sigma^{1/2} $, where $Z_{i } = (Z_{i1}, \cdots, Z_{ip}) \in \mathbb R^{1 \times p}$ and each $Z_{ij}$ is i.i.d. centered 1-sub-Gaussian random variable such that $\mathbf E ( Z_{i}^\top Z_{i} ) = I_p$.
The population covariance $\Sigma \in \mathbb R^{p \times p}$ satisfies
$$
M^{-1} \le  \Lambda_{\min}(\Sigma ) \le \Lambda_{\max}(\Sigma ) \le M,
$$
where $M>1$ is a universal constant.
\end{assumption}
}

\begin{assumption}[Sample size]\label{assumption: sample}
The sample size $n$ satisfies
$
\max \left( s \log p ,~ o \log n \right) \lesssim n.
$
\end{assumption}

{ \color{black}
Assumption \ref{assumption: RIP} controls the correlation among covariates and is commonly used in the literature on iterative algorithms \citep{Fan2023implicit, Han2025adaptivedebiased}.
These assumptions yield the following proposition, which underpins the theoretical guarantees of our iterative procedure.

\begin{proposition}[Restricted isometry and incoherence]\label{prop: rii}
For any fixed constant \(C_1>0\), assume Assumption~\ref{assumption: RIP} holds and Assumption~\ref{assumption: sample} holds in the specific form
$
n\ge 30M^2C_1C_\Sigma \max(s\log p,o\log n),
$
where \(C_\Sigma>0\) is a constant depending only on \(\|\Sigma\|_2\).
Then the following properties hold with probability greater than \(1-4\exp(-2C_1s\log p)\):
\begin{enumerate}
\item \textbf{(Restricted isometry)} For every index set $S \subset [p]$ with $|S| \le C_1 s$, the sample covariance matrix satisfies:
\begin{equation}\label{eq: rii}
 \frac{1}{2M} \|u\|_2^2 ~\le ~ u^\top \left( \frac{ X^\top X}{n} \right)_{S,S} u ~\le ~2M\|u\|_2^2, \text{  for every } u \in \mathbb R^{|S|}.
\end{equation}
\item \textbf{(Restricted incoherence)} There exists a constant $C_M>0$ depending only on $M$ such that:
\begin{equation}\label{eq: inco}
\sup_{S\subset[p]:~ |S| \le {\color{black}C_1} s} ~
\sup_{O\subset[n]:~ |O| \le {\color{black}C_1}o}
\left\| X_{O,S} \right\|_2 \le {\color{black}\sqrt{C_1}} C_M \sqrt{s \log p + o \log n}.
\end{equation}
\end{enumerate}
\end{proposition}

This proposition characterizes both the restricted isometry property of the design matrix $X$ and the restricted correlation between columns of $X$ and the identity matrix $I_n$ (since $X_{O,S}^\top = X_{ \cdot,S}^\top~ I_{ \cdot,O}$). The latter controls how adversarial contamination can distort the estimation of $\beta^*$, and such an incoherence-type condition is standard in outlier analyses \citep{dalalyan19nips, Ndaoud24slope}. 
If \eqref{eq: inco} fails, for example, in the case $n=p$ and $X = \sqrt{n}\,I_n$, then model~\eqref{adv} reduces to $Y = \sqrt{n}(\beta^* + \theta^*) + \xi$, making it impossible to estimate $\beta^*$ alone consistently.

\begin{remark}[Sub-Gaussian limitation]\label{remark: subgaussian}
Proposition \ref{prop: rii} relies on a sub-Gaussian design and a spectrally bounded covariance matrix $\Sigma$, and thus cannot be directly generalized to heavy-tailed designs addressed in some robust statistics literature \citep{Sun02012020, PolingLoh24jasa}. 
We speculate that such extensions may require some technical tools like truncation (Section 4 in \cite{Sun02012020}), and leave this for future work. 
\end{remark}
}

\subsection{Property of the first stage estimation}
We begin with the statistical guarantee of the first-stage Algorithm \ref{alg1}.

{\color{black}
\begin{theorem}[Initial Estimation]\label{th1: Initial estimation}
Assume that Assumptions \ref{assumption: RIP} and \ref{assumption: sample} hold.
For tuning parameters in Algorithm~\ref{alg1}, suppose that the learning rate $\eta \in \left[\frac{2M}{4M^2 + 1}, \frac{4M}{4M^2 + 1} \right]$, the decay rate $\kappa \in \left(\frac{4M^2}{4M^2 +1}, 1 \right)$, and the initial thresholds satisfy $\sqrt{s}\lambda_{\beta,0} > \| \beta^* \|_2$ and $\sqrt{o} \lambda_{\theta,0} > \| \theta^* \|_2$.
Let
\begin{equation}\nonumber
\begin{aligned}
\lambda_{\beta, \infty} &=C_{\beta,1} \sigma \sqrt{\frac{M \log p}{n}} + C_{\beta,2} \sigma\sqrt{\frac{s\log p+ o\log n}{ns}} \sqrt{\frac{o \log n}{n}}, \\
\lambda_{\theta, \infty} &= \frac{3C_{\beta,1}\sigma}{4} \sqrt{\frac{\log n}{n}} + \frac{4C_{\beta,2}\sigma \sqrt M}{3} \sqrt{\frac{s\log p+ o\log n}{no}} \sqrt{\frac{s \log p}{n}},
\end{aligned}
\end{equation}
where $C_{\beta,1},~C_{\beta,2} $ are two constants depending on $M, \kappa$, and $\eta$.
Let $(\widehat{\beta}, \widehat{\theta})$ denote the output of Algorithm~\ref{alg1}.
Then with probability at least $1 - O(p^{-2}+n^{-3})$, both estimators are $\ell_0$-sparse, i.e.,
$
\|\widehat{\beta}\|_0 \lesssim s,~\|\widehat{\theta}\|_0 \lesssim o,
$
and satisfy
\begin{equation}\label{eq: upper1}
\begin{aligned}
\| \widehat{\beta} - \beta^* \|_2^2 \lesssim &\sigma^2 \left( \frac{s \log p}{n} + \frac{o^2 \log^{2} n}{n^2} \right),\\
\| \widehat{\theta} - \theta^* \|_2^2 \lesssim& \sigma^2 \left( \frac{o\log n}{n} + \frac{s^2 \log^{2} p}{n^2} \right). 
\end{aligned}
\end{equation}
\end{theorem}

The $\ell_2$ error of $\widehat{\beta}$ consists of two terms: (i) the estimation rate of an $s$-sparse vector $ (\sigma^2 s\log p) / n$, and (ii) the contamination proportion $ (\sigma^2 o^2 \log^2 n) / n^2$.
Furthermore, by interpolating the $\ell_0$ sparsity and $\ell_2$ rate, we conclude that
$$
\| \widehat{\beta} - \beta^* \|_q^q \lesssim \sigma^q \left\{ s \left(\frac{ \log p}{n}\right)^{q/2}   + s^{1-q/2} \left( \frac{ o \log n}{n } \right)^q 
 \right\}, \text{ for every } q \in [1,2],
$$
demonstrating that $\widehat{\beta}$ is minimax near-optimal in terms of the $\ell_q$ error for all $q \in [1, 2]$ (see Theorem \ref{th4: lower estimation} for the lower bound).

\begin{remark}[Two phases of estimation accuracy]
The bound for $\widehat \beta$ in \eqref{eq: upper1} can be rewritten as
$$
\| \widehat{\beta} - \beta^* \|_2^2 \lesssim 
\sigma^2  \max\left( \frac{ s\log p}{n},~ \frac{ o^2\log^2 n}{n^2} \right) 
=
\begin{cases}
\sigma^2 \frac{ s\log p}{n},  & \text{ if }  o \le \frac{\sqrt{ns \log p}}{\log n},\\
 \sigma^2 \frac{ o^2\log^2 n}{n^2},  & \text{ if } \frac{\sqrt{ns \log p}}{\log n} < o \lesssim \frac{n}{\log n}.
\end{cases}
$$
Thus, there are two distinct regimes: 
When the number of contaminated samples $o$ is relatively small, the estimation error attains the uncontaminated near-optimal rate $(\sigma^2 s \log p )/ n$ \citep{Raskutti2011minimax}.
As $o$ increases, the error of $\widehat\beta$ becomes dominated by the squared contamination proportion (up to a logarithmic term).
\end{remark}

\begin{remark}[Symmetric rate and joint bound]
Within the bounds of \eqref{eq: upper1}, the terms $o\log n$ and $s\log p$ play symmetric roles: swapping them in the bound of $ \widehat{\beta}$ yields the corresponding bound of $ \widehat{\theta}$.
This symmetry reflects a duality between $\beta^*$ and $\theta^*$: If one views the model \eqref{adv} as a sparse linear regression, then $\sqrt{n}\theta^*$ constitutes an $\ell_0$-sparse contamination of the observations; conversely, if one views \eqref{adv} as a (sub-Gaussian) location model, the term $X\beta^*$ acts as contamination, representing the image of an $\ell_0$-sparse vector $\beta^*$ into the observation space via $X$.  

Summing the individual bounds in \eqref{eq: upper1} gives the joint error control: 
\begin{equation}\label{eq: joint upper}
\|\widehat \beta - \beta^* \|_2^2 + \|\widehat \theta - \theta^* \|_2^2 \lesssim \sigma^2\left(\frac{s \log p + o \log n}{n}  \right) ,
\end{equation}
which is minimax near-optimal (see Theorem 6 in \citet{She22outlier} for the joint lower bound ).
A comparison of \eqref{eq: upper1} and \eqref{eq: joint upper} shows that focusing solely on $\beta^*$ delivers a more refined guarantee. 
\end{remark}
}

\subsection{Property of the two-stage estimation}

This subsection establishes signal adaptivity and the strong oracle property of the final estimator $\tilde{\beta}$, which is obtained from the second-stage Algorithm \ref{alg2}. 
We first specify the signal condition imposed on $\beta^*$.
\begin{assumption}[Signal condition for $\beta^*$]\label{assumption: betamin}
There exists a constant $C_{\beta}>0$ such that
{\color{black}
\begin{equation}\label{eq: beta min}
    \min_{i \in S^*} |\beta_i^*| \ge C_\beta   \sigma \left( \sqrt{\frac{\log p}{n}} +  \frac{o \log n}{n \sqrt s} \right), 
\end{equation}
}
\end{assumption}
This signal condition guarantees that all nonzero components of $\beta^*$ are well separated from zero, thereby facilitating sharper estimation rates.

\begin{theorem}[Signal-adaptive estimation]\label{th2: Signal adaptive estimation}
 Assume that Assumptions \ref{assumption: RIP} and \ref{assumption: sample} hold. 
{\color{black} We set the tuning parameters in Algorithm \ref{alg2} as $\eta \in \left[\frac{2M}{4M^2 + 1}, \frac{4M}{4M^2 + 1} \right]$, $\lambda_{\beta} = \lambda_{\beta,\infty}$, and $\lambda_{\theta} = \lambda_{\theta,\infty}$, where $\lambda_{\beta,\infty}, \lambda_{\theta,\infty}$ are specified in Theorem~\ref{th1: Initial estimation}.
Let $(\tilde{\beta}, \tilde{\theta})$ denote the output of Algorithm~\ref{alg2}.
Then, for any given $\varrho \in (0,1)$, with probability at least $1 -\varrho- O(p^{-2} + n^{-3})$, we have $\|\tilde{\beta} \|_0 \lesssim s,~\|\tilde{\theta} \|_0 \lesssim o$, 
and the estimation error is signal-adaptive:
\begin{equation}\label{eq: SA}
    \begin{aligned}
  \|\tilde{\beta} - \beta^* \|_2^2  \lesssim \begin{dcases}
  \sigma^2 \left( {\frac{s + \log(1/\varrho)}{n}} + \frac{(s \log p+o\log n)^2}{n^2}\right), & \text{ if Assumption \ref{assumption: betamin} holds,} \\ 
  \sigma^2 \left(  {\frac{s \log p}{n}} + \frac{o^2 \log^{2} n}{n^2}\right),  & \text{ otherwise.}
\end{dcases}
\end{aligned}
\end{equation}
}
\end{theorem}
Theorem \ref{th2: Signal adaptive estimation} establishes the signal adaptivity of the second-stage AC-IHT algorithm: 
If the signal condition in Assumption \ref{assumption: betamin} fails, $\tilde \beta$ achieves a near-optimal rate; if it holds, $\tilde \beta$ adaptively achieves a sharper estimation rate than the {\color{black} minimax near-optimal} rate.
This signal adaptivity is a key advantage of hard-thresholding-based methods and is generally difficult to achieve by convex procedures such as the Lasso \citep{bellec2018noise}. 
{\color{black}Furthermore, if Assumption \ref{assumption: betamin} holds and $n \gtrsim s \log^2 p$, from \eqref{eq: SA} we can get a sharper bound:
$$
\|\tilde{\beta} - \beta^*\|_2^2 \lesssim \sigma^2 \left( {\frac{s+\log(1/\varrho)}{n}} + \frac{o^2 \log^{2} n}{n^2}\right),
$$}
which matches the estimation rate (up to a $\log n$ factor) as if the support of $\beta^*$ were known \citep{hammouda2024outlier}.

Signal-adaptive estimation has been well-studied in uncontaminated models \citep{Fan2018ILAMM, MN19, NM20, Fan2023implicit}.
However, to our knowledge, no existing work addresses this property under adversarial contamination, and Theorem \ref{th2: Signal adaptive estimation} fills this gap.
Moreover, by utilizing this property, we can derive a more refined variable selection guarantee than that in Theorem \ref{th2: Signal adaptive estimation}. 

\begin{corollary}[Selection error]\label{co1: almost full recovery}
Under the conditions of Theorem \ref{th2: Signal adaptive estimation} and Assumption \ref{assumption: betamin}, {\color{black} if $\frac{o^2\log^{2}n}{n } \prec s\log p \prec  n $, then, as $s \succ 1$, with probability at least $1 - O(p^{-2} + n^{-3})$ we have 
$
\left | \text{supp}(\beta^*)\  \triangle \ \text{supp}(\tilde{\beta})\right |\prec s.
$
}
\end{corollary}

Utilizing the interplay between estimation and selection, {\color{black}Corollary~\ref{co1: almost full recovery} effectively controls the variable selection error of the estimator \(\tilde{\beta}\) (obtained from two-stage AC-IHT).
Moreover, Theorem \ref{th5: lower selection} establishes that the required signal strength Assumption \ref{assumption: betamin} is nearly necessary, demonstrating that our procedure is minimax near-optimal (up to logarithmic terms) in the variable selection task.
}

We further analyze the oracle property of \(\tilde{\beta}\).
An additional signal strength condition on $\theta^*$ is presented as follows. 
\begin{assumption}[Signal condition for $\theta^*$]\label{assumption: thetamin}
There exists a sufficiently large (absolute) constant $C_\theta>0$ such that
{\color{black}
\begin{equation}\label{eq: theta min}
    \min_{k \in O^*} |\theta_k^*| \ge C_\theta   \sigma \left( \sqrt{\frac{\log n}{n}} + \frac{s\log p}{n \sqrt o} \right) .
\end{equation}
}
\end{assumption}
{\color{black}
This assumption is instrumental and sufficient for identifying contaminated samples, which allows our procedure to further remove the bias induced by corruption.
Similar requirements appear in the recent literature (see Theorem 3 in \citet{hammouda2024outlier}).
Moreover, it implies that a larger number of contaminated samples $o$ relaxes the required signal strength for outlier identification.
This phenomenon is empirically validated through numerical experiments in Supplementary Material S1.5.
}
Define the oracle estimator $\beta^{\dag}$ as
\begin{equation}\label{oracle definition}
\beta^{\dagger} := \mathrm{Proj}_\beta \left\{\operatorname*{arg\,min}_{\substack{\beta_{(S^{*})^c} = \mathbf{0} \\ \theta_{(O^{*})^c} = \mathbf{0}}} ~\frac{1}{2n} \left\|Y - X\beta - \sqrt{n}\theta \right\|^2_2\right\},
\end{equation}
where $\text{Proj}_{\beta}$ denotes the projection onto the $\beta$ component of the joint vector $(\beta^\top,\theta^\top)^\top$. 
The next theorem shows that, under suitable signal conditions, $\tilde\beta$ converges to the oracle estimator $\beta^{\dag}$.

\begin{theorem}[Oracle estimation and selection consistency]\label{th3: Strong Oracle estimation}
Assume that all conditions in Theorem \ref{th2: Signal adaptive estimation} hold, and Assumptions \ref{assumption: betamin} and \ref{assumption: thetamin} hold.
Let $\{\tilde{\beta}^t\}_{t \ge 0}$ be the sequence of iterates from Algorithm \ref{alg2}. 
{\color{black} Then with probability at least $1 - O(p^{-2} + n^{-3})$, there exist two absolute constants $r\in (0, 1)$ and $C>0$ such that
\begin{equation}\label{eq: arbitrary conv}
\| \tilde{\beta}^t - \beta^\dag \|_2 \leq C \times  r^t, \quad\text{for every } t \ge 0.
\end{equation}
Additionally, by terminating Algorithm \ref{alg2} after $t \ge C' \log n$ iterations, for any given $\varrho \in (0,1)$, with probability at least $1 -\varrho -  O(p^{-2} + n^{-3} )$ the output $\tilde{\beta}$ reaches the oracle rate and achieves selection consistency:
$$
\| \tilde{\beta} - \beta^* \|_2^2 \lesssim  \sigma^2 \left(\frac{s+ \log(1/\varrho)}{n-o} \right),
\quad \text{supp}(\tilde{\beta}) = \text{supp}(\beta^*),
\quad \text{supp}(\tilde{\theta}) = \text{supp}(\theta^*).
$$ 
}
\end{theorem}

Under suitable signal conditions, $\tilde \beta$  (obtained from two-stage AC-IHT) converges geometrically to the oracle estimator $\beta^\dag$, therefore reaching the $\ell_2$ estimation rate $\sigma \{s/(n-o)\}^{1/2}$ as if the true support of $\beta^*$ were known and the contaminated data were excluded. 
Leveraging this fact, one can further establish the asymptotic property of each linear functional of $\tilde \beta$.

{\color{black}
\begin{corollary}[Asymptotic normality]\label{co2: asymp}
Assume that all conditions in Theorem~\ref{th3: Strong Oracle estimation} hold, and assume $n \succ \max \left(s^2 \log^2 p ,~ s^3 \right)$. 
Define $c_{\xi} :=  \text{Var}(\xi_1)/\sigma^2$.
Then for every $ \gamma \in \mathbb R^s$ with $0<\| \gamma \|_2 < \infty$, as $n,p\to \infty$, we have
    $$
    \sqrt{n}\gamma^{\top}(\tilde{\beta}_{S^*} - \beta^*_{S^*}) \overset{D}\to \mathcal{N} \left(0,\ c_{\xi}\sigma^2\gamma^{\top} \Sigma_{S^*, S^*}^{-1} \gamma  \right).
    $$
\end{corollary}
}
Corollary \ref{co2: asymp} demonstrates that $\tilde \beta$ possesses asymptotic normality, thereby enabling statistical inference. 
This property distinguishes our method from existing methods tailored for adversarial contamination \citep{dalalyan19nips, Ndaoud24slope}.

{\color{black}
\begin{remark}[Revisit signal adaptivity]
Assumptions \ref{assumption: betamin} and \ref{assumption: thetamin} provide the technical prerequisites for the AC-IHT procedure to converge to the oracle estimator, ensuring both selection consistency and asymptotic normality.
When these signal conditions are not satisfied, our estimator may no longer enjoy these strong guarantees; nevertheless, it at least maintains minimax near-optimality and $\ell_0$ sparsity, as established in Theorem \ref{th2: Signal adaptive estimation}.
This underscores that our procedure does not require prior knowledge of the true signals for practical execution, and thus is inherently signal-adaptive: it delivers the (nearly) best possible estimation accuracy for the given data, with stronger signals adaptively yielding sharper results. 
\end{remark}

For clarity, we summarize the conditions on sample size, signal strength, and corruption required by the above theoretical results in Table \ref{table: Assumption list}.
}

\begin{table}[htbp]
    \centering
    \caption{Summary of sample size, signal strength, and corruption assumptions for main theoretical results.}\label{table: Assumption list}
    \resizebox{\linewidth}{!}{
    \begin{tabular}{c|c|c|c}
    \hline
    \diagbox[width=5.5cm,height=1.5cm]{Theoretical  Result}{Assumption} & \parbox[c][1.5cm][c]{5cm}{\centering \small{\textbf{Sample Size}}} &  \parbox[c][1.5cm][c]{4cm}{\centering \small{\textbf{Strength of $\beta^*$ and $\theta^*$}}} & \parbox[c][1.5cm][c]{2cm}{\centering \small{\textbf{Corruption Number}}} \\
    \hline
    \parbox[c][1.5cm][c]{5cm}{\centering \textbf{Theorem 1}\\\centering {\footnotesize (Initial Estimation)}} &\multirow{4}{5cm}{
    \parbox[c][4cm][c]{5cm}{\centering $\max (s\log p,o\log n) \lesssim n$\\ \ \\ \centering {\footnotesize(Assumption 2)}}} & \parbox[c][1.5cm][c]{4cm}{\centering None} &\multirow{2}{3cm}{ \parbox[c][1cm][c]{3cm}{\centering $o\lesssim \frac{n}{\log n}$}} \\
    \cline{1-1}\cline{3-3}
    \parbox[c][1.5cm][c]{5cm}{\centering \textbf{Theorem 2}\\ {\footnotesize(Signal-adaptive estimation)}} &  & \multirow{2}{4cm}{\parbox[c][1.5cm][c]{4cm}{$\min_{i : \beta^*_i \ne 0} |\beta_i^*| \ge \beta^{\ddag}$\\ \ \\ \centering {\footnotesize(Assumption 3)}}} &  \\
    \cline{1-1}\cline{4-4}
    \parbox[c][1.5cm][c]{5cm}{\centering \textbf{Corollary 1}\\ {\footnotesize(Selection error)}} &  &  & \parbox[c][1.5cm][c]{3cm}{\centering $o \prec \frac{\sqrt{n  s\log p}}{\log n} $} \\
    \cline{1-1}\cline{3-4}
    \parbox[c][1.5cm][c]{5cm}{\centering \textbf{Theorem 3}\\ {\footnotesize(Oracle estimation and selection consistency)}} &  & \multirow{2}{4cm}{\parbox[c][2cm][c]{4cm}{\centering $\min_{i : \beta^*_i \ne 0} |\beta_i^*| \ge \beta^{\ddag}$\\ \ \\ \centering $\min_{k : \theta^*_k \ne 0} |\theta_k^*| \ge \theta^{\ddag}$\\\ \\ \centering {\footnotesize(Assumption 3 and 4)}}} & \multirow{2}{3cm}{\parbox[c][2cm][c]{3cm}{\centering $o\lesssim \frac{n}{\log n}$}} \\
    \cline{1-2}
    \parbox[c][1.5cm][c]{5cm}{\centering \textbf{Corollary 2}\\ {\footnotesize(Asymptotic normality)}} &
    \parbox[c][1.5cm][c]{6cm}{\centering $  \max \left(s^2 \log^2 p,s^3,o\log n\right)\prec n$} &  &  \\
    \hline
    \end{tabular}
    }
    \parbox{\linewidth}{\footnotesize
    \raggedright
    \textit{Note:}
    Here $\beta^{\ddag} := C_\beta   \sigma \left\{\frac{\log p}{n} +  \frac{o^2 \log^{2 }n}{ n^2 s }\right\}^{1/2}$, $\theta^{\ddag} := C_\theta   \sigma \left\{\frac{\log n}{n} +  \frac{s^2 \log^2 p}{n^2 o}\right\}^{1/2}$, where $o= \|\theta^* \|_0 \vee 1$, and $C_\beta,C_\theta>0 $ are some absolute constants.
    }
\end{table}

\subsection{Minimax lower bounds}
This subsection provides minimax lower bound guarantees for estimation and support recovery.
For ease of display, we consider the model with Gaussian noise: 
assume that each $Y_i|X_{i\cdot }$ is drawn from $\mathcal N(X_{i\cdot }\beta^*,~ \sigma^2)$, where $X_{i\cdot} \in \mathbb R^{1\times p}$ follows from the random design as introduced in Assumption \ref{assumption: RIP}. Here, we only require a sparse eigenvalue assumption as 
\begin{equation}\label{eq: random rip}
\sup_{S \subset [p], |S| \le 2s} \left\| \Sigma_{S,S} \right\|_2 \le C_{2s},  
\end{equation}
where $C_{2s} >0$ is an absolute constant.
Denote the uncontaminated distribution of $(X_{i\cdot}, Y_i)$ as $\mathbf P_{X,Y}$ and write our model space as 
\begin{equation}\label{eq: model setting}
\begin{aligned}
\mathcal M( \beta, o) :=& \left\{ (n-k) \text{ observations are drawn from } \mathbf P_{X,Y}, \right.\\
& \quad k \text{ observations are drawn from arbitrary } \mathbf Q ,\text{ where } 0 \le k \le o\left. \right\} .
\end{aligned}
\end{equation}

Much of the robust estimation literature considers the $\epsilon$-Huber contamination model $(X_{i\cdot}, Y_i) \sim (1- \epsilon)\mathbf P + \epsilon \mathbf Q$ \citep{chengaoren18aos, gaochao20bernoulli, Chinot20EJS}.
In contrast, our model setting \eqref{eq: model setting} explicitly constrains the number of outliers ($o$) rather than the contamination probability ($\epsilon$).
This difference yields that our lower bound results hold independent significance.

{\color{black}
\begin{theorem}[Estimation lower bound]\label{th4: lower estimation}
Assume that Assumption~\ref{assumption: sample} and equation \eqref{eq: random rip} hold and $o \ge 9$. Then for any $q \in [1,2]$ with an absolute constant $c_{q} \in \left( 0, \frac{(8C_{2s})^{-q/2}}{160} \right)$, we have
\begin{equation*}
\inf_{\widehat \beta} \sup_{\beta^*: \|\beta^*\|_0 \le s} ~\sup_{\mathbf R \in \mathcal M (\beta^*,o)} \mathbf E_{(X,Y) \sim\mathbf R} \left( \left\| \widehat \beta - \beta^* \right\|_q^q\right)
\ge c_q \sigma^q \left\{ s \left(\frac{ \log (ep/s)}{n}\right)^{q/2}   + s^{1-q/2} \frac{  o^q}{n^q} \right\},
\end{equation*}
where $\mathbf R \in \mathcal M (\beta^*,o) $ is a joint distribution of $(X_{i\cdot}, Y_i)_{i\in [n]}$.
\end{theorem}
}
Combined with Theorem \ref{th2: Signal adaptive estimation}, it implies that our two-stage AC-IHT algorithm is {\color{black} minimax near-optimal (up to logarithmic factors)}, and can even surpass this minimax rate under a proper signal strength condition.

We next establish the minimax lower bound for variable selection.
Define 
$$
\mathcal B(s,a):= \left\{ \beta \in \mathbb R^p: ~ \|\beta\|_0 \le s , ~\min_{i: \beta_i \ne 0} |\beta_i| \ge a  \right\}
$$
as the $s$-sparse vector space with a minimal signal strength $a>0$.

\begin{theorem}[Selection lower bound]\label{th5: lower selection}
Assume that Assumption~\ref{assumption: sample} and equation \eqref{eq: random rip} hold and $o \ge 8$, then with an absolute constant $c_2 \in \left( 0, 1/5 \right)$, we have
\begin{equation}
\inf_{\widehat S} ~\sup_{\beta^* \in \mathcal B(s, a)} ~\sup_{\mathbf R \in \mathcal M (\beta^*,o)} \mathbf E_{(X,Y) \sim\mathbf R} \left( \left| \widehat S ~ \triangle~  \text{supp}(\beta^*)\right| \right)
\ge c_2 s
\end{equation}
holds for $a \le \frac{\sigma}{4\sqrt{2C_{2s}}}\left( \sqrt{\frac{\log(ep/s)}{n}} + \frac1{\sqrt s}   \frac on \right) $.
\end{theorem}

This result shows that no procedure can recover $\text{supp}(\beta^*)$ with $o(s)$ selection error if the signals of $\beta^*$ are relatively weak.
{\color{black} It further demonstrates that the signal strength Assumption \ref{assumption: betamin} is minimax near-optimal (up to logarithm terms) for the variable selection task in a contamination setting. }

\section{Simulation Studies}\label{section:Numerical experiment}

This section presents numerical experiments that complement our theoretical findings.
We set $p = 1000$, $n = 300$, $s=o= 10$, and assume $X$ is generated from the multivariate normal distribution $\mathcal{N}(\mathbf{0},\Sigma)$, where $\Sigma_{ij} = \rho^{|i-j|}$ for $i,j\in[p]$ and $\rho = 0.25$. The true parameter vectors $\beta^*$ and $\theta^*$ have their first 10 entries set to 0.5 and all remaining entries set to 0. 
All simulations in this section are executed with 300 replications. 

Our two-stage AC-IHT algorithm involves two tuning parameters, $\lambda_{\beta,\infty}$ and $\lambda_{\theta,\infty}$, as suggested by Theorems \ref{th1: Initial estimation} and \ref{th2: Signal adaptive estimation} {\color{black}(see Supplementary S3.4 for the initial tuning of $(\lambda_{\beta,0}, \lambda_{\theta,0})$)}. 
Here, we determine their values using a Massart-type information criterion by minimizing the following objective
$$
\left\| Y - X \tilde \beta - \sqrt n \tilde \theta \right\|_2^2 + A (\tilde s \log p + \tilde o \log n),
$$
where $\tilde \beta = \tilde \beta (\lambda_{\beta,\infty}, \lambda_{\theta,\infty})$ and $\tilde \theta = \tilde \theta( \lambda_{\beta,\infty}, \lambda_{\theta,\infty})$ are the estimators obtained with the candidate parameters $(\lambda_{\beta,\infty}, \lambda_{\theta,\infty})$, and $\tilde s = | \text{supp}(\tilde \beta) |$ and $\tilde o = | \text{supp}(\tilde \theta) |$. 
In this section, we set $A=2$ and search for the best pair $(\lambda_{\beta,\infty}, \lambda_{\theta,\infty})$ over the region $(10^{-2},1)\times(10^{-2},1)$. We set the decay rate $\kappa = 0.9$ and the learning rate $\eta = 0.75$.

For benchmarking, we compare our AC-IHT method with several existing robust estimation methods: IHT-${\ell_1}$ \citep{Xia2025AOS}, {\color{black}the Progressive Iterative Quantile-Thresholding (PIQ) estimator \citep{She22outlier}}, the Adaptive Huber (Ada-Huber) estimator \citep{Sun02012020}, AC-LASSO \citep{thompson2020outlier}, and AC-SCAD (which replaces the $\ell_1$ penalty in AC-LASSO with the SCAD penalty). 
In addition, the Oracle estimator (defined in \eqref{oracle definition}) is included as an ideal reference for comparison.

Five metrics are used to evaluate the estimation performance: (i) the $\ell_2$-error $\|\beta - \beta^*\|_2$; (ii) the $\ell_\infty$-error $\|\beta - \beta^*\|_\infty$; 
(iii) the $\Sigma$-norm-error $\|\beta - \beta^*\|_\Sigma = \big\{ (\beta - \beta^*)^\top\Sigma(\beta - \beta^*) \big\}^{1/2}$; (iv) Matthews correlation coefficient (MCC) and (v) Symmetric difference (Sym\_diff). 
Here, MCC and Sym\_diff are defined as
$$
    \text{MCC} = \frac{\text{TP} \times \text{TN} - \text{FP} \times \text{FN}}{\{(\text{TP}+\text{FP})(\text{TP}+\text{FN})(\text{TN}+\text{FP})(\text{TN}+\text{FN})\}^{1/2}},
$$
$$
    \text{Sym\_diff} = \text{FP} + \text{FN},
$$
where $\text{TP} = |\widehat{S} \cap S^*|,\ \text{TN} = |\widehat{S}^c \cap (S^*)^c|,\ \text{FP} = |\widehat{S} \cap (S^*)^c|, \ \text{FN} = |\widehat{S}^c \cap S^*|$ and $\widehat{S}$ is the support index of the estimator. Among these metrics, a smaller value indicates better estimation or variable selection performance for all except MCC. For MCC, values closer to 1 indicate better selection accuracy.

\setcounter{equation}{0}
\subsection{Estimation accuracy and selection consistency}\label{Estimation}
In this subsection, we consider the contamination model with noise term $\xi$ independently generated from the following three distributions, each with unit variance: the standard Gaussian distribution $\mathcal{N}(0,1)$, the Rademacher distribution, and the uniform distribution $\mathcal{U}(-\sqrt{3},\sqrt{3})$.
Table \ref{tab:estimation} shows that AC-IHT achieves competitive performance: 
It yields the lowest estimation errors and the best support recovery, with results close to the Oracle benchmark $\beta^\dag$.
Compared with other methods such as IHT-$\ell_1$, PIQ, and Adaptive Huber, AC-IHT is more stable across different noise settings. 
Additionally, the AC-SCAD estimator shows intermediate performance between AC-IHT and AC-LASSO, consistent with the SCAD penalty being sandwiched between hard-thresholding and soft-thresholding penalties.

\begin{table}[htbp]
\caption{Comparison of estimation accuracy across different methods under contaminated data.}
\label{tab:estimation}
\resizebox{\linewidth}{!}{%
{\renewcommand{\arraystretch}{1.2} 
\begin{tabular}{lccccc} \hline
\textbf{Method} & $\|\beta - \beta^*\|_2$ & $\|\beta - \beta^*\|_\infty$ & $\|\beta - \beta^*\|_{\Sigma}$ & \textbf{MCC} & \textbf{Sym\_diff} \\[3pt] \hline

\multicolumn{6}{c}{\textbf{Gaussian}} \\ \hline
AC-IHT     & 0.222 (0.004)    & 0.140 (0.003) & 0.213 (0.004) & 0.989 (0.001) & 0.243 (0.031)  \\
IHT-$\ell_1$ & 0.489 (0.014)  & 0.347 (0.010) & 0.458 (0.013) & 0.953 (0.003) & 0.890 (0.050) \\
PIQ        & 0.535 (0.010)    & 0.341 (0.007) & 0.522 (0.010) & 0.912 (0.004) & 2.007 (0.091)  \\
Ada-Huber  & 0.531 (0.005)    & 0.276 (0.003) & 0.616 (0.006) & 0.697 (0.003) & 10.760 (0.185)  \\
AC-LASSO  &  0.484 (0.004)    & 0.229 (0.003) & 0.540 (0.005) & 0.556 (0.011) & 38.993 (2.412)  \\
AC-SCAD    & 0.350 (0.005)    & 0.182 (0.004) & 0.333 (0.004) & 0.561 (0.011) & 32.257 (1.383)   \\
Oracle    &  0.196 (0.003)    & 0.118 (0.002) & 0.186 (0.002) & 1.000 (0.000) & 0.000 (0.000)   \\ \hline

\multicolumn{6}{c}{\textbf{Rademacher}} \\ \hline
AC-IHT     &0.214 (0.004)    & 0.136 (0.003) & 0.202 (0.004) & 0.992 (0.001) & 0.170 (0.025)  \\
IHT-$\ell_1$ & 0.799 (0.004) & 0.499 (0.001) & 0.752 (0.004) & 0.899 (0.002) & 1.900 (0.031)   \\
PIQ          & 0.543 (0.010) & 0.346 (0.007) & 0.530 (0.010) & 0.913 (0.004) & 1.987 (0.097)  \\
Ada-Huber    & 0.528 (0.005) & 0.279 (0.003) & 0.607 (0.005) & 0.659 (0.003) & 13.083 (0.181) \\
AC-LASSO  & 0.489 (0.004)    & 0.238 (0.003) & 0.548 (0.004) & 0.588 (0.010) & 32.317 (2.222)    \\
AC-SCAD   & 0.368 (0.004)    & 0.198 (0.004) & 0.345 (0.004) & 0.591 (0.011) & 30.363 (1.650)  \\
Oracle   &  0.197 (0.003)    & 0.121 (0.002) & 0.185 (0.002) & 1.000 (0.000) & 0.000 (0.000)  \\ \hline

\multicolumn{6}{c}{\textbf{Uniform}} \\ \hline
AC-IHT    &  0.211 (0.004)    & 0.135 (0.003) & 0.202 (0.004) & 0.992 (0.001) & 0.173 (0.025)  \\
IHT-$\ell_1$  & 0.680 (0.010) & 0.463 (0.006) & 0.638 (0.009) & 0.919 (0.002) & 1.533 (0.044) \\
PIQ        & 0.534 (0.010)    & 0.341 (0.007) & 0.523 (0.010) & 0.913 (0.004) & 2.003 (0.091)  \\
Ada-Huber  & 0.532 (0.005)    & 0.276 (0.003) & 0.615 (0.006) & 0.669 (0.003) & 12.463 (0.206) \\
AC-LASSO   & 0.487 (0.004)    & 0.234 (0.003) & 0.547 (0.004) & 0.586 (0.010) & 32.340 (2.141)  \\
AC-SCAD    & 0.367 (0.004)    & 0.196 (0.004) & 0.348 (0.004) & 0.580 (0.011) & 30.857 (1.556)  \\
Oracle    &  0.193 (0.003)    & 0.117 (0.002) & 0.183 (0.002) & 1.000 (0.000) & 0.000 (0.000)    \\ \hline
\end{tabular}}
}
\footnotesize
The numbers in parentheses denote the standard errors.
\end{table}

\subsection{Oracle property}\label{Signal adaptive}
This subsection demonstrates that our AC-IHT method exhibits the oracle property.
Under the standard Gaussian distribution, we first consider the convergence to the oracle estimator with increasing sample size $n$.
Figure \ref{pic:Signal adaptive} shows that, as $n$ increases from 200 to 800, AC-IHT performs well and gradually approaches the oracle estimator, illustrating high accuracy and precise support recovery.
This empirical result aligns with the convergence guarantees established in Theorem \ref{th3: Strong Oracle estimation}.
{\color{black}The simulation results illustrating the asymptotic normality and iteration-wise convergence behavior of AC-IHT are provided in Section~S1.2 and S1.4 of the Supplementary Material.}

\begin{figure}[H] 
    \centering
    \includegraphics[width=1\textwidth]{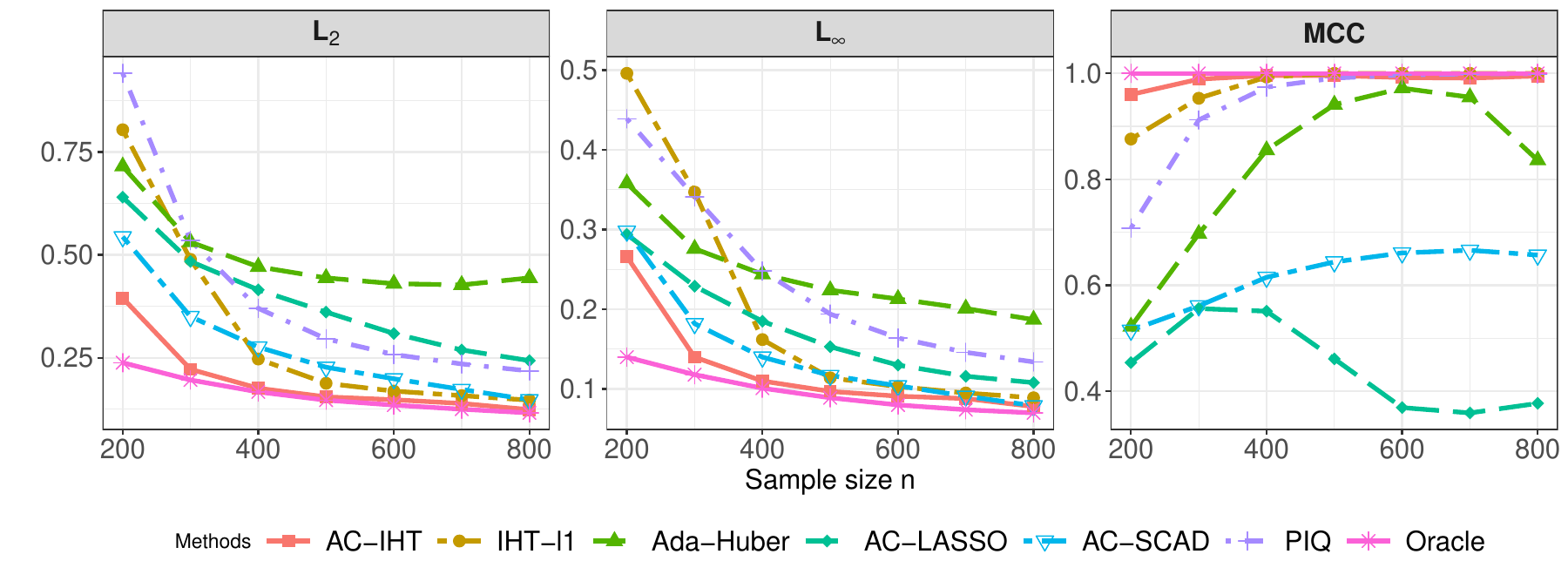}
    \caption{Estimation accuracy and support recovery performance with increasing sample size.}
    \label{pic:Signal adaptive}
\end{figure}

\section{Discussion}\label{section:Discussion}
We conclude by discussing the generalizability of our procedure and several directions for future research.

\subsection{Extension to GLM} 
Here we extend our two-stage AC-IHT algorithm to the Generalized Linear Models (GLMs) setting, thereby allowing for a broader class of response distributions.
Consider a GLM setup with $n$ independent observations $\{(X_{i,\cdot },Y_i)\}^{n}_{i=1}$, where $X_{i,\cdot } \in \mathbb R^{1 \times p}$ is the $i$-th row of $X$. 
The distribution of each $Y_i$ is assumed to follow an exponential family characterized by the natural parameter $\zeta^*_i$ and a scale parameter $a$, with the density function:
\begin{equation}\label{GLM}
    f(Y_i;\zeta^*_i) = \exp\left( \frac{Y_i\zeta^*_i - b(\zeta_i^*)}{a} + c(Y_i,a) \right),
\end{equation}
where $b(\cdot)$ denotes a cumulant function.
We adopt the canonical link function \( b'^{-1} (\cdot) \) and define the linear predictor as \( \zeta_i^* = X_{i,\cdot}  \beta^* + \sqrt{n} \theta_i^* \) under our adversarial contamination setting.
We consider minimizing the negative log-likelihood
$$
-\log \prod_{i\in [n]} f(Y_i;\zeta^*_i) 
\propto  \frac1n \sum_{i \in [n]}\Big\{ b(\zeta_i^*) - Y_i\zeta_i^* - a \times c(Y_i,a) \Big\}.
$$ 
Therefore, only the {\bf gradient update} step in our two-stage AC-IHT algorithm requires modification, while all other steps remain unchanged. The updated gradient step is as follows:
\begin{equation}\label{eq: GLM GD}
\textbf{Step G.1:}\quad
\begin{aligned}
    H^{t+1}_{\beta} &\gets \beta^t - \frac{{\color{black}\eta}}{n}\sum_{i = 1}^{n}X_{i,\cdot}^\top\left(b^{'}(X_{i,\cdot}\beta^t + \sqrt{n}\theta^t_i ) - Y_i\right), \\
    H^{t+1}_{\theta} &\gets \theta^t - \frac{{\color{black}\eta}}{\sqrt{n}}\sum_{i = 1}^{n}e_i \left(b^{'}(X_{i,\cdot}\beta^t + \sqrt{n}\theta^t_i ) - Y_i\right),
\end{aligned}
\end{equation}
where {\color{black}$\eta > 0$ denotes the learning rate,} $e_i \in \mathbb R^{n \times 1}$ is a vector with its $i$-th component equals 1 and all other components are 0.
To implement this variant, we replace {\bf Step~1.1} in Algorithm~\ref{alg1} and Step~2.1 in Algorithm~\ref{alg2} with {\bf Step~G.1} \eqref{eq: GLM GD}.
We introduce the following regularity assumption to establish the theoretical guarantees for the GLM extension. 

\begin{assumption}\label{LU}
The function $b(\cdot)$ is twice-differentiable, and assume that each $\zeta_i^* \in \Theta$, where the parameter space $\Theta \subseteq \mathbb R$ is a closed (finite or infinite) interval.
There exist two constants $0<L \le U < \infty $ such that the function $b'' (\cdot)$ satisfies $L \le \inf_{t \in \Theta} b''(t) \le \sup_{t \in \mathbb R} b''(t) \le U$. 
\end{assumption}
This assumption controls the variance $\text{Var}(Y_i) =ab''(\zeta_i^*)$ and provides uniform strong convexity (on $\Theta$).
It is a standard condition in high-dimensional GLM analyses \citep{abramovich2016model}.

\begin{theorem}[GLM]\label{th6: GLM}
Suppose that Assumptions \ref{assumption: RIP}, \ref{assumption: sample}, and \ref{LU} hold. Let $\tilde{\beta}^{\text{GLM}}$ be the estimator obtained from the two-stage AC-IHT algorithm in the GLM setting. 
For appropriately chosen algorithmic parameters (refer to Supplementary Material S10), it holds with probability at least $1 - \varrho - O(p^{-2} + n^{-3})$ that $\|\tilde{\beta}^{\text{GLM}}\|_0 \lesssim s$. Furthermore, the squared error $\|\tilde{\beta}^{\text{GLM}} - \beta^* \|_2^2$ achieves the same signal-adaptive rate as established in \eqref{eq: SA}.  
\end{theorem}
Therefore, in the GLM setting, $\tilde{\beta}^{\mathrm{GLM}}$ maintains both sparsity and signal adaptivity as in \eqref{eq: SA}, establishing the generality of our two-stage AC-IHT algorithm.

{\color{black}
\subsection{Relationship with heavy-tailed regression}
Here, we establish a formal connection between the contaminated model \eqref{adv} and heavy-tailed regression.
Consider the high-dimensional heavy-tailed model
\begin{equation}\label{eq: heavy}
    Y = X \beta^* + \epsilon \in \mathbb R^n,
\end{equation}
where we assume the white noise $\epsilon$ is independent of $X$, and each $\epsilon_i$ is zero-mean and has a bounded $(1+\delta)$-th moment, i.e., $\mathbf E( |\epsilon_i|^{1+\delta}) \le  v_\delta$ for every $i \in [n]$, where $\delta >0$.
For such a heavy-tailed regression, \citet{Sun02012020} established the minimax $\ell_2$ estimation rate as
$$
\begin{cases}
v_\delta^{\frac{1}{1+\delta}} \sqrt s \left( \frac{\log p}{n} \right)^{\frac{\delta}{1+\delta} },  & \text{ if } \delta \le1, \\
v_1^{\frac{1}{2}} \sqrt s \left( \frac{\log p}{n} \right)^{\frac12 },  & \text{ if } \delta >1,
\end{cases}
$$
and achieved this rate via a regularized Huber estimator.

Now we reformulate model \eqref{eq: heavy} as an adversarial contamination model through the following decomposition:
\begin{equation} 
    Y = X \beta^* + \big( \epsilon - \psi_\tau(\epsilon) \big) +  \psi_\tau(\epsilon)=  X \beta^* + \sqrt n \theta^*  + \xi.
\end{equation}
Here, each $\xi_i := \psi_\tau(\epsilon_i), ~i \in [n]$ represents the $\tau$-truncated noise, and $\psi_\tau(\cdot) := \max( \min(\cdot, \tau), -\tau)$ is the truncation operator. 
The remaining term, $\sqrt n \theta^*_i := \epsilon_i - \psi_\tau(\epsilon_i)$, is treated as an outlier component.
Leveraging this decomposition, we apply our AC-IHT algorithm to robust heavy-tailed regression and establish the following theoretical guarantees.

\begin{theorem}[Heavy tailedness]
Consider model \eqref{eq: heavy} and suppose Assumption \ref{assumption: RIP} holds with sample size $n \gtrsim (s + \log n) \log p$.
Run Algorithm~\ref{alg1} and \ref{alg2} with the learning rate $\eta \in \left[\frac{2M}{4M^2 + 1}, \frac{4M}{4M^2 + 1} \right]$, the decay rate $\kappa \in \left(\frac{4M^2}{4M^2 +1}, 1 \right)$, and the initial threshold $\lambda_{\beta,0} > \| \beta^* \|_2/\sqrt{s}$.
Denote by $(\tilde{\beta}, \tilde{\theta})$ the output of Algorithm~\ref{alg2}. Then:  
\begin{enumerate}
\item \textbf{Case $\delta \in (0,1)$.} With $\lambda_{\theta,0} \gtrsim  n^{-1/2} (nv_\delta/\varrho)^{\frac1{1+\delta}}$, $\lambda_{\beta, \infty} = \lambda_\beta \asymp  v_\delta^{\frac{1}{1+\delta}} \left( \frac{\log p}{n} \right)^{\frac{\delta}{1+\delta}} \left( 1+ \sqrt{\frac{\log n}{s}}\right)$, and $\lambda_{\theta, \infty} =  \lambda_\theta \asymp n^{\frac{1-\delta}{2(1+\delta)}} \left( \frac{v_\delta}{\log p} \right)^{\frac{1}{1+\delta}}$, under a probability at least $1 -\varrho-  O(p^{-2})$ we have:
$$
\|\tilde{\beta} \|_0 \lesssim s,
\quad
\|\tilde{\beta} - \beta^* \|_2 \lesssim v_\delta^{\frac{1}{1+\delta}}  \left( \frac{\log p}{n} \right)^{\frac{\delta}{1+\delta}}\sqrt{s+\log n}.
$$

\item \textbf{Case $\delta \ge 1$.} Take $v_1 = \mathbf E(\epsilon_i^2)$. With $\lambda_{\theta,0} \gtrsim \sqrt{v_1/\varrho }$, $\lambda_{\beta, \infty} = \lambda_\beta \asymp  \sqrt{\frac{v_1 \log p}{n} } \left( 1+ \sqrt{\frac{\log n}{s}}\right)$, and $\lambda_{\theta, \infty} =  \lambda_\theta \asymp \sqrt{\frac{v_1}{\log p} }$, under a probability at least $1 -\varrho-  O(p^{-2})$ we have: 
$$
\|\tilde{\beta} \|_0 \lesssim s,
\quad
\|\tilde{\beta} - \beta^* \|_2 \lesssim \sqrt{\frac{v_1 (s+ \log n) \log p}{n} }.
$$
\end{enumerate} 
\end{theorem}

The above result confirms that our two-stage AC-IHT algorithm remains minimax near-optimal (up to a $\log n$ term) under heavy-tailed settings, demonstrating the generality of the proposed procedure. 
Numerical simulations of our algorithm under heavy-tailed noise are presented in Supplementary Material S1.3, and are consistent with the theoretical result.
Moreover, this theorem shows that heavy-tailed regression could be connected to the adversarial contamination framework through a truncation-based decomposition, thereby revealing the broader theoretical and practical significance of the contamination framework \eqref{adv}. 
}

\subsection{ {\color{black}Limitations and future directions } }

This paper proposes an algorithmic regularization procedure that preserves signal adaptivity and achieves the strong oracle property. 
In our theoretical analysis, we explicitly balance optimization and statistical errors,
yielding a computationally efficient procedure with near-optimal guarantees up to logarithmic factors. 
{\color{black}However, several limitations remain, including the lack of theoretical guarantees for the adaptive tuning (used in Section \ref{section:Numerical experiment}), 
the remaining logarithmic gap between the upper and lower bounds, 
and the restriction to sub-Gaussian designs for $X$ (discussed in Remark \ref{remark: subgaussian}). 
These theoretical improvements are left for future investigation.
}

\bibliographystyle{unsrtnat} 
\bibliography{sample.bib}

\clearpage

\begin{appendix}
\section{Additional Simulations}\label{supp: simulation}

In this section, we present additional simulation results that are not included in the main text due to space limitations. These supplementary results provide further support for the theoretical findings reported in the paper.
\subsection{Varying the level of sparsity and contamination}
We fix $p=1000,~ n=600$, vary sparsity level $s$ from 7 to 35 and contamination level $o$ from 12 to 120.
All non-zero signals take a value of 1.
The $\ell_2$ estimation errors of $\beta^*$ are shown in Figure \ref{fig: toy}.

\begin{figure}[htbp]
    \centering
    \includegraphics[width=0.7\textwidth]{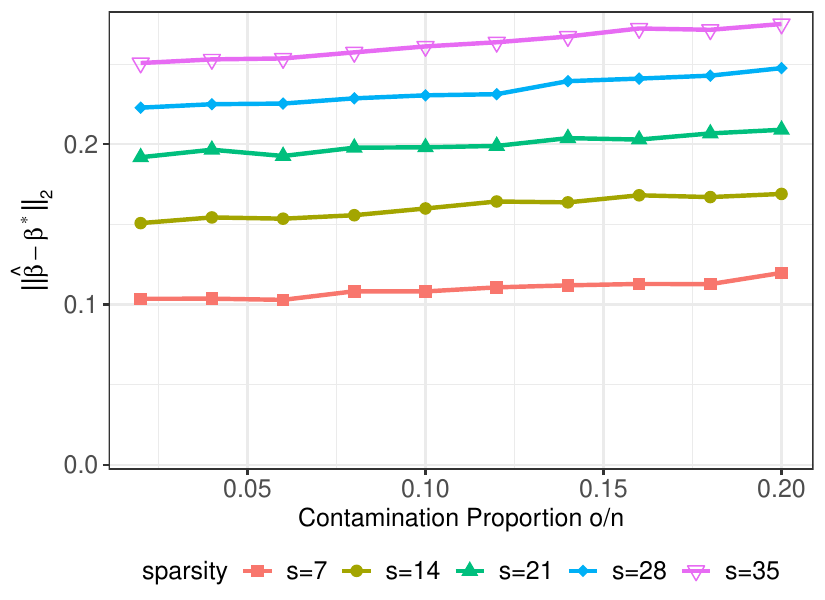}
    \caption{The $\ell_2$ estimation error of the two-stage AC-IHT algorithm with varying $s$ and $o$ under 300 replications.}
    \label{fig: toy}
\end{figure} 

As illustrated in Figure \ref{fig: toy}, the $\ell_2$ estimation error $\|\tilde\beta - \beta^*\|_2$ grows linearly with the contamination proportion $o/n$, yet it grows only on the order of $\sqrt{s}$ of the sparsity level $s$.  
This empirical behavior matches our theoretical guarantees (summarized in Table 1 in our manuscript): 
No matter in which signal cases, for fixed $o$, the $\ell_2$ error bound increases linearly in $\sqrt{s}$; for fixed $s$, it increases linearly in $o$.
In particular, when $o/n \le 1/5$, the $\ell_2$ error bound in Theorem 3 can be rewritten as
$$
\|\tilde\beta - \beta^*\|_2 
\;\asymp\; \sigma\sqrt{\frac{s}{\,n - o\,}}
\;\asymp\; \sigma\sqrt{\frac{s}{n}} \left(1 + \frac on \right),
$$
which explains the linear relationship with $o$.

\subsection{Asymptotic normality}
We then consider the asymptotic normality of the estimator obtained from two-stage AC-IHT by constructing the z-score based on Corollary~2 with 300 replications.
Specifically, we set the 5th and 6th elements of $\gamma$ to 1 and all others to 0. 
We assess the asymptotic normality of AC-IHT, IHT-$\ell_1$, and AC-SCAD by using histograms, Q-Q plots, and the $R^2$ values from the no-intercept linear fit of these Q-Q plots.

As illustrated in Figure~\ref{pic:Section_2 n=300,(5,6).pdf}, AC-IHT exhibits the best asymptotic normality performance: Its histogram closely aligns with the normal density curve, and the points in its Q-Q plot lie almost perfectly along the diagonal, with an $R^2 = 0.9891$.
In contrast, IHT-$\ell_1$ and AC-SCAD show noticeable deviations and lower $R^2$ values.
\begin{figure}[H]
    \centering
    \includegraphics[width=0.8\textwidth]{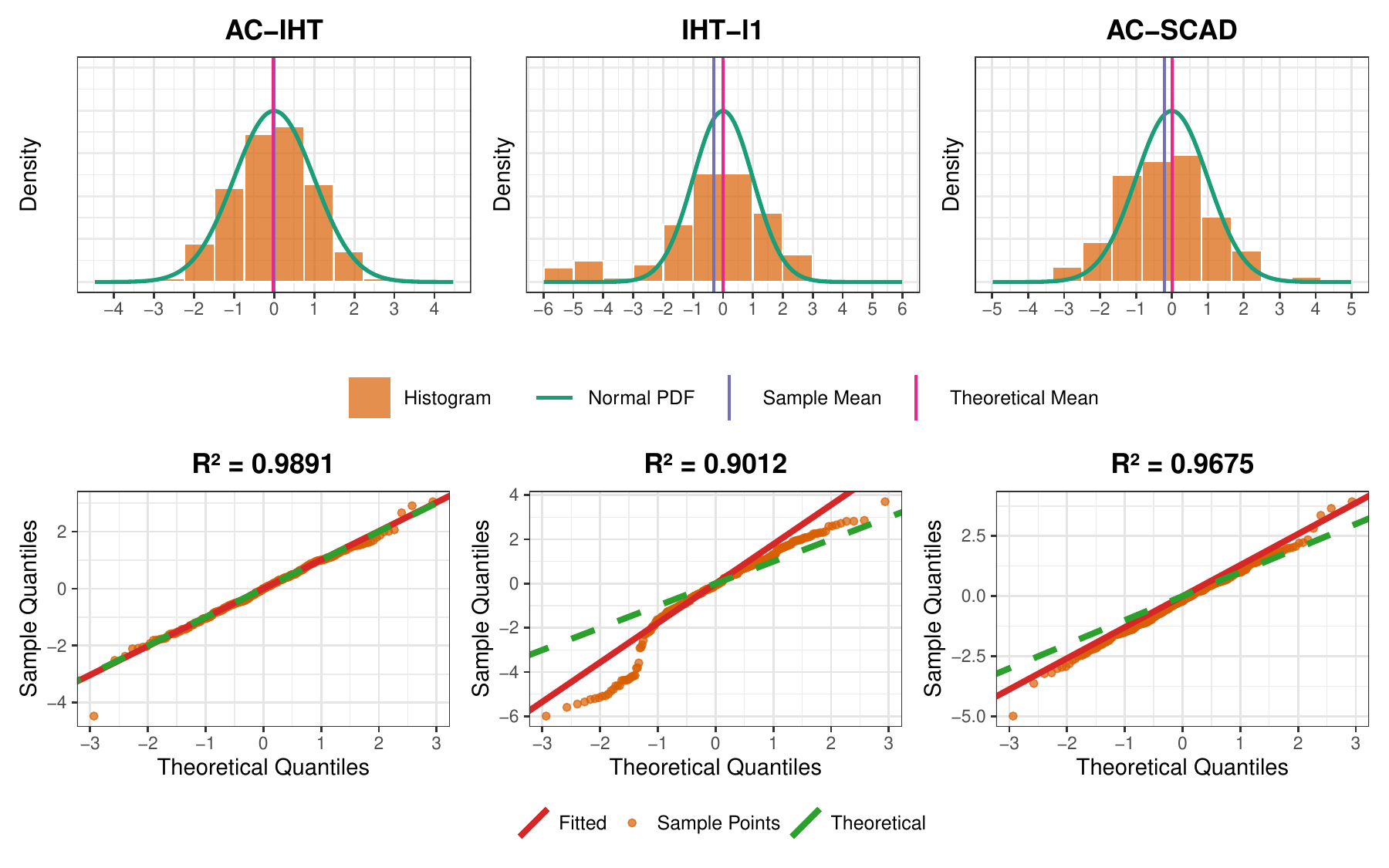}
    \caption{Comparison of asymptotic normality across AC-IHT, IHT-$\ell_1$, and AC-SCAD methods.}
    \label{pic:Section_2 n=300,(5,6).pdf}
\end{figure}

\subsection{Heavy tailedness} 
Following the theoretical extension in the discussion section,  we conduct supplementary experiments under heavy-tailed noise settings. 
Specifically, we generate i.i.d. white noise from Student's t-distributions with 2 and 3 degrees of freedom and assume no adversarial contamination, while keeping all other model configurations identical to those in Section 4.
And we include the Oracle Huber estimator as an ideal reference, which is the Huber estimator with the known support set.
Results over 100 replications are reported in Table \ref{tab: heavy}, with standard errors shown in parentheses. 
The findings demonstrate that our AC-IHT method achieves favorable performance in both estimation accuracy and support recovery, which aligns with its minimax near-optimal guarantee in Theorem 7.

\begin{table}[htbp]
\caption{Comparison of estimation accuracy under heavy-tailed noise.}
\label{tab: heavy}
\resizebox{\linewidth}{!}{%
{\renewcommand{\arraystretch}{1.2}
\begin{tabular}{lccccc} \hline
\textbf{Method} & $\|\beta - \beta^*\|_2$ & $\|\beta - \beta^*\|_\infty$ & $\|\beta - \beta^*\|_{\Sigma}$ & \textbf{MCC} & \textbf{Sym\_diff} \\[3pt] \hline
\multicolumn{6}{c}{\textbf{t(df=2)}} \\ \hline
AC-IHT       & 0.610 (0.027) & 0.378 (0.014)          & 0.569 (0.024)          & 0.907 (0.009) & 1.940 (0.193)  \\
IHT-$\ell_1$ & 0.757 (0.013) & 0.494 (0.005)          & 0.713 (0.014)          & 0.895 (0.005) & 1.970 (0.083)  \\
Ada-Huber     & 0.683 (0.011) & 0.354 (0.007)         & 0.770 (0.012)          & 0.563 (0.006) & 21.440 (0.549) \\
Oracle Huber & 0.269 (0.007) & 0.162 (0.005)          & 0.254 (0.006)          & 1.000 (0.000) & 0.000 (0.000) \\ \hline
\multicolumn{6}{c}{\textbf{t(df=3)}} \\ \hline
AC-IHT       & 0.435 (0.020) & 0.301 (0.016)          & 0.403 (0.018)            & 0.959 (0.005) & 0.850 (0.107)  \\
IHT-$\ell_1$ & 0.624 (0.023) & 0.416 (0.015)          & 0.578 (0.021)            & 0.926 (0.005) & 1.400 (0.096)  \\
Ada-Huber    & 0.598 (0.010) & 0.310 (0.006)          & 0.683 (0.011)            & 0.632 (0.005) & 15.120 (0.409) \\
Oracle Huber & 0.252 (0.006) & 0.153 (0.004)          & 0.235 (0.005)            & 1.000 (0.000) & 0.000 (0.000)  \\ \hline
\end{tabular}}
}
\end{table}

\subsection{Dynamic of convergence}
In this subsection, we examine the iteration-wise convergence behavior of AC-IHT under varying sample sizes and contamination levels. Specifically, we consider sample sizes $n\in\{300, 500, 700\}$ and numbers of contaminated observations $o\in\{10, 30\}$, while keeping all other simulation settings the same as in the beginning of Section~4. Estimation accuracy is assessed using the $\ell_2$ error ($L_2$) and the $\ell_\infty$ error ($L_{\max}$), and support recovery accuracy is evaluated by the Matthews Correlation Coefficient (MCC). The corresponding iteration trajectories are displayed in Figure \ref{pic:Convergence dynamic}.

Figure \ref{pic:Convergence dynamic} shows a consistent convergence pattern across all configurations. 
The errors $L_2$ and $L_{\max}$ decrease during the initial iterations and then level off after about 20 iterations, whereas MCC rises toward one after about 15 iterations and remains stable thereafter, indicating reliable support recovery.
Moreover, larger sample sizes improve estimation and support recovery performance, as evidenced by smaller $L_2/L_{\max}$ and higher MCC. When the contamination level increases from $o=10$ to $o=30$, convergence becomes slightly slower, and the terminal errors are mildly larger, particularly for smaller $n$, which aligns with our theoretical results.

\begin{figure}[H]
\centering
\includegraphics[width=1\textwidth]{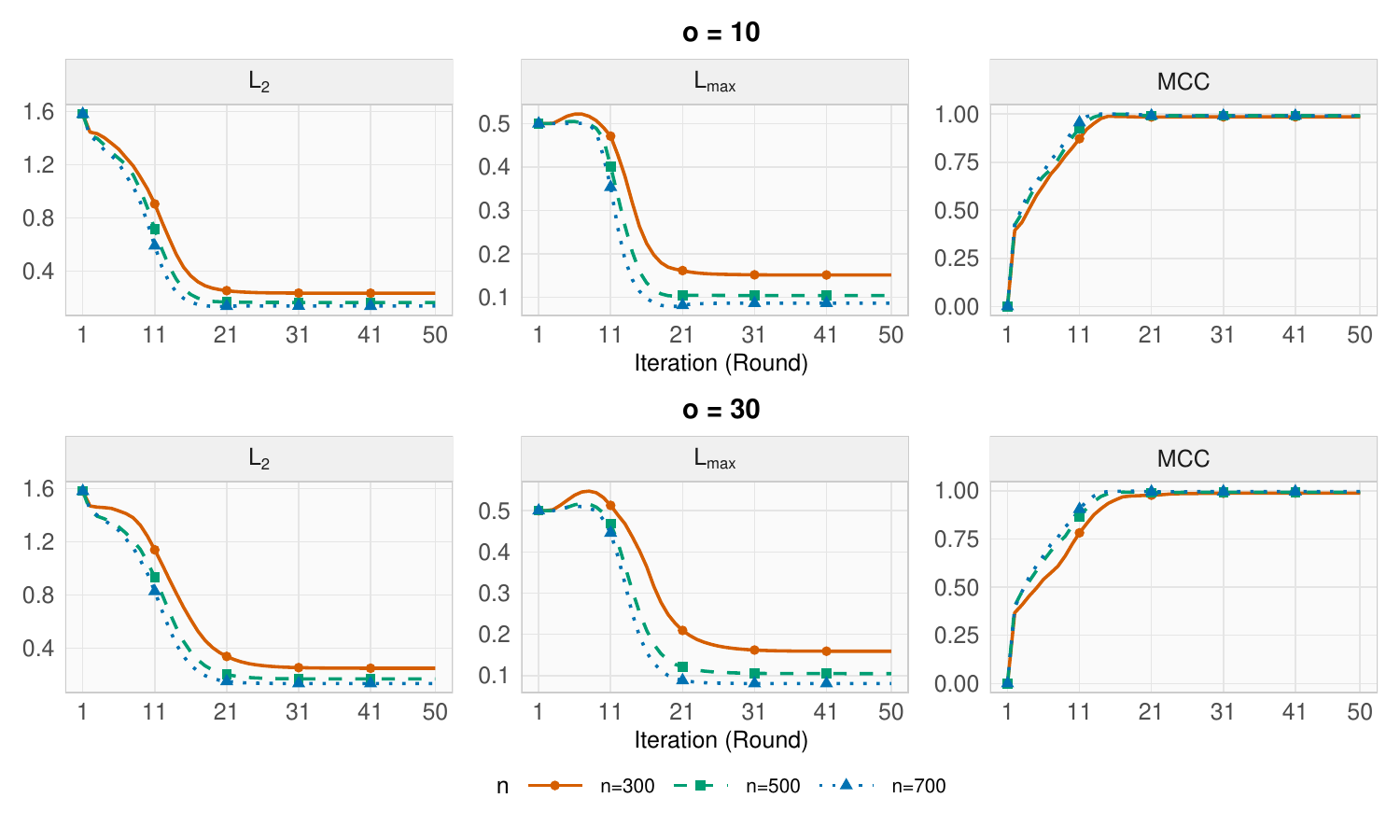}
\caption{Convergence dynamics of AC-IHT over iterations for different values of $n$ and $o$. Each $(o,n)$ setting undergoes 100 repeated simulations.}
\label{pic:Convergence dynamic}
\end{figure}

\subsection{Relationship between signal strength and sparsity.}
Assumption 4 indicates that the required signal strength (for \(\theta^*\)) increases as \(o\) decreases. The following simulation validates this phenomenon.
Following the simulation settings in Section 4, we fix the non-zero entries of \(\theta^*\) at 0.2, and examine the support recovery of \(\theta^*\) as \(o\) increases under standard Gaussian noise. 
The simulation result in Figure \ref{fig: increase o} (with 100 replications) shows that as \(o(= \| \theta^*\|_0) \) increases, the MCC (Matthews Correlation Coefficient) rises and the averaged Hamming loss \(\left(\frac{\left | \text{supp}(\theta^*) \triangle \text{supp}(\hat{\theta})\right |}{o}\right)\) decreases.
This indicates that a larger \(o\) makes the ``theta-min'' condition easier to satisfy, yielding progressively more accurate outlier identification. 
And this result aligns with the relationship between $o$ and the signal strength in Assumption 4.

\begin{figure}
    \centering
    \includegraphics[width=0.8\linewidth]{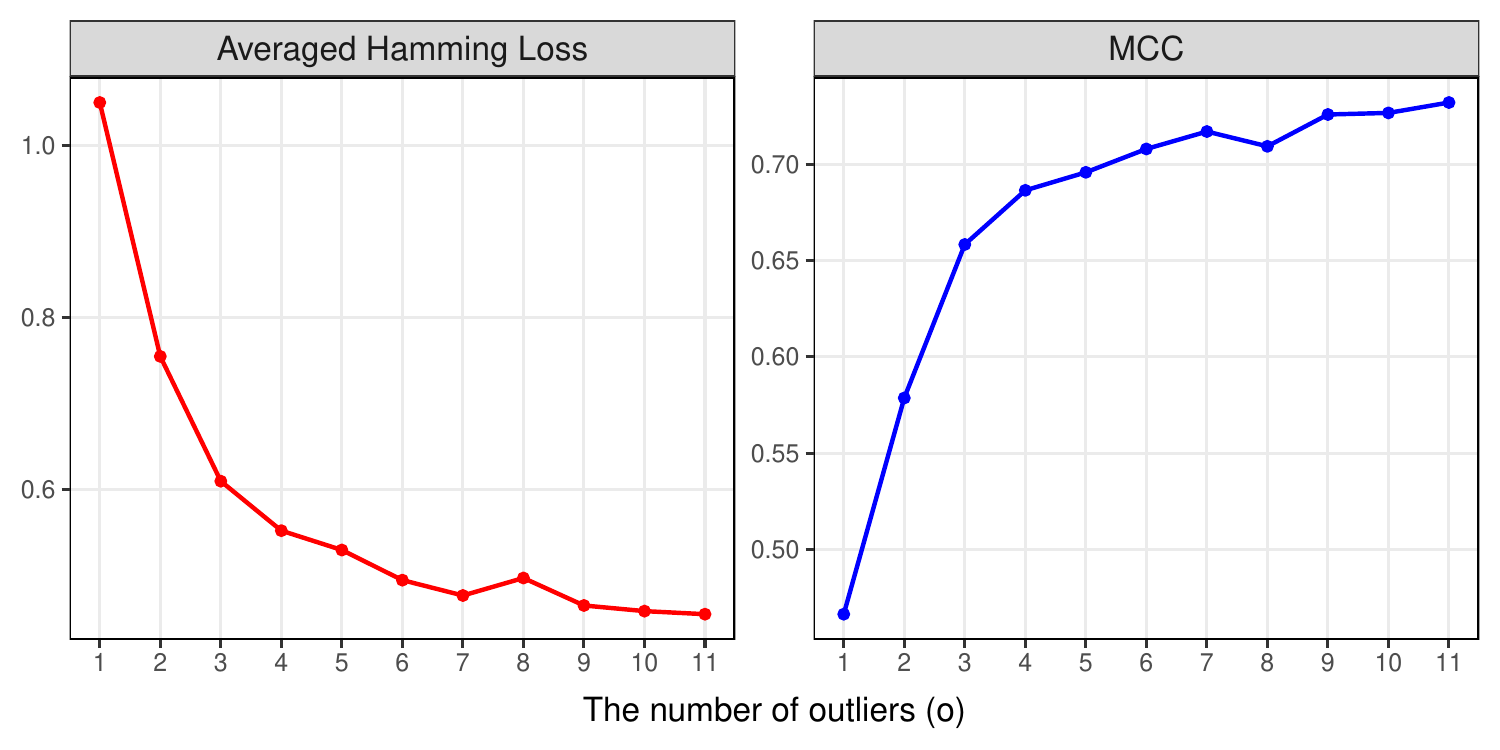}
    \caption{Support recovery performance with increasing $o$.}
    \label{fig: increase o}
\end{figure}

\section{Proof of Proposition 1}\label{sec: proof of Prop1}
Throughout this proof, \(C_1>0\) is treated as a fixed constant, and Assumption~2 is used in the concrete form stated in Proposition~1.
We first prove the restricted isometry of \(X\).
For a fixed set \(S\subset[p]\) with \(|S|=C_1s\), by Remark 5.40 in \cite{vershynin2010introduction}, we obtain
\begin{equation}\label{eq: spec norm}
\mathbf P \left( \left\|\frac1n X_{\cdot S}^\top X_{\cdot S}- \Sigma_{SS} \right\|_2 > \max(\iota , \iota^2)  \right) \le 2e^{- u},
\end{equation}
where
\(
\iota =
\left(C \vee \frac1{\sqrt c}\right)
\left(
\sqrt{\frac{C_1s}{n}}
+
\sqrt{\frac{u}{n}}
\right),
\)
and \(C,c\) are constants depending only on \(\|\Sigma\|_2\). Let
\(
C_\Sigma := C^2 \vee \frac1c \vee 1 .
\)
By taking \(u=3C_1s\log p\) and using
\(
n\ge 30M^2C_1C_\Sigma \max(s\log p,o\log n),
\)
we have \(\iota\vee\iota^2=\iota\le 1/(2M)\), which yields that
\begin{equation}\label{eq: X Sigma}
\begin{aligned}
&\mathbf P \left( \sup_{S\subset [p]:~ |S| \le C_1 s} \left\|\frac1n X_{\cdot S}^\top X_{\cdot S}- \Sigma_{SS} \right\|_2 > \frac1{2M}\right)\\
\le& \sum_{ S\subset [p]:~ |S| = C_1 s} \mathbf P \left( \left\|\frac1n X_{\cdot S}^\top X_{\cdot S}- \Sigma_{SS} \right\|_2 > \max(\iota , \iota^2)  \right)\\
\le& 2\binom p{ C_1 s} \exp\left(- 3C_1 s\log p\right)\\
\le&  2 \exp\left( -2C_1 s\log p \right).
\end{aligned}
\end{equation}
By Weyl's inequality, with probability at least $1 - 2 \exp(-2C_1 s \log p)$ we have that,
\begin{equation}\label{eq: weyl}
\max_{S\subset [p]:~ |S| \le C_1 s}~ \max_{1\le k \le |S|}~ \left| \Lambda_k\left( \frac1n X_{\cdot S}^\top X_{\cdot S}\right) - \Lambda_k (\Sigma_{SS})\right| \le \frac{1}{2M}.
\end{equation}
Consequently, the restricted isometry property follows immediately from this uniform eigenvalue bound and the definition of eigenvalue.

We then prove the restricted incoherence of $X$.
Similar to \eqref{eq: spec norm}, for fixed sets $S\subset[p]$ with $|S|= C_1 s$ and $O \subset [n]$ with $|O |=C_1 o $, we have
$$
\mathbf P \left( \left\|\frac1{C_1 o}  X^\top_{O, S} X_{O, S} - \Sigma_{SS}  \right\|_2 >\max(\tau , \tau^2)  \right) \le 2e^{-v},
$$
where $\tau = \left(C \vee \frac1{\sqrt c}\right) \cdot \left(\sqrt{\frac{ s}o } +\sqrt{\frac{v}{C_1 o} } \right)$.
By setting $v =3C_1 s\log p+3C_1 o \log n$ and applying a union bound, we obtain 
\begin{equation*} 
\begin{aligned}
&\mathbf P \left\{  \sup_{O \subset [n]: ~|O| \le C_1 o }~ \sup_{S\subset[p]:~ |S| \le C_1 s} \left\|\frac1{C_1 o}  X^\top_{O, S} X_{O, S} - \Sigma_{SS}  \right\|_2 >\max(\tau , \tau^2)  \right\} \\
\le&  \binom{p}{C_1 s} \binom{n}{C_1 o } 2e^{-v}\\
\le& 2 e^{-2 C_1 s \log p}.
\end{aligned}
\end{equation*}
Consequently, for any $S \subset [p]$ with $|S| \le C_1 s$ and $O \subset [n]$ with $|O| \le C_1 o$, it holds that 
\begin{equation}\label{eq: inco induction}
\begin{aligned}
\|X_{O, S} \|_2  =&\sqrt{\left\| X^\top_{O,S} X_{O,S} \right\|_2 }\\
\le & \sqrt{C_1 o \| \Sigma_{SS}\|_2 + C_1 o (\tau + \tau^2 )  }\\
\le&
\sqrt{
C_1 o M
+C_1 C_\Sigma
\left( \sqrt{so}
+ \sqrt{\frac{vo}{C_1}}
+2s+ \frac{2v}{C_1}
\right)
}\\
\le&
\sqrt{
C_1 o M
+ 10 C_1 C_\Sigma
\left(s \log p + o \log n\right)
}\\
\le& \sqrt{C_1} C_M\cdot \sqrt{ s \log p+ o \log n} 
\end{aligned}
\end{equation}
with probability at least \(1-2\exp(-2C_1s\log p)\).
Here \(C_M:=\sqrt{M+ 10C_\Sigma}>1\) is a constant depending only on \(M\).
Therefore, we prove the restricted incoherence property and complete the proof of Proposition 1.

\begin{remark}[Incoherence condition]
The restricted incoherence condition \eqref{eq: inco induction} plays a pivotal role in our theoretical analysis: it controls how adversarial contamination distorts the estimation of $\beta^*$, and enables us to achieve a sharper estimation accuracy than joint estimation.   
Existing incoherence conditions in the literature, such as Definition 1.(ii) in \cite{dalalyan19nips}, and equations (2.11)-(2.12) in \cite{Ndaoud24slope}, are often tied to the specific penalty forms of Lasso or Slope, and are relatively intricate.
In contrast, by leveraging the $\ell_0$ structure, we introduce a more concise and intuitive incoherence condition, and prove that it holds with high probability under sub-Gaussian designs.
\end{remark}

\section{Proof of Theorem 1}\label{supp: Th1}
We introduce some definitions that will be used in the following proofs. 
Define 
$$\Phi :=  \frac{ \eta }n  X^{\top} X - I_p \in \mathbb R^{p\times p}, \quad ~\Xi := \frac1n X^{\top}\xi \in \mathbb R^{p}.
$$
For ease of display, we assume $\sigma=1$ in the main proof.

Three parts constitute the proof of Theorem 1.
We first introduce some useful preliminaries, and then prove the sparsity and error bounds by using mathematical induction.
Finally, we verify the validity of our proposed threshold sequence.

\subsection{Preliminary}\label{sec: th1 pre}
Define the event
\begin{equation}\label{eq: good event of Th1}
\mathcal E: = \left\{  \left\| \Xi  \right\|_\infty < 4\sqrt{\frac{M\log p}{n}} , \text{ and }  \left\| \frac1{\sqrt n} \xi  \right\|_\infty < 3\sqrt{\frac{ \log n}{n} }\right\},
\end{equation}

and
$$
\mathcal E_X:= \begin{Bmatrix}
\frac1{2M} \le \Lambda_k\left( \frac1n X_{\cdot S}^\top X_{\cdot S}\right) \le 2M, 
\text{ for every }S \subset [p]:~ |S| \le (2B+1)s \text{ and } 1\le k \le |S|,  \\
\sup_{S\subset[p]:~ |S| \le (B+1) s} ~
\sup_{O\subset[n]:~ |O| \le (B+1)o}
\left\| X_{O,S} \right\|_2 \le {(\sqrt B+1)} f \sqrt n 
\end{Bmatrix},
$$
where
\begin{equation}\label{eq: B delta}
B:= \left(\frac{\kappa+\delta}{\kappa- \delta}\right)^2>1, \quad \delta := \frac{4M^2}{4M^2 +1}\in (0,1), \quad f:= C_M \sqrt{\frac{s \log p + o \log n}n}.
\end{equation}
Throughout the proof of Theorem~1, Assumption~2 is used in the concrete form
\[
n\ge C_{\mathrm{Th1}}\{s\log p+o\log n\},
\]
where \(C_{\mathrm{Th1}}\) is a fixed constant depending only on \(M,\eta,\kappa\), chosen large enough so that
\[
C_{\mathrm{Th1}}
\ge
\max\left\{
30M^2(2B+1)C_\Sigma,~
\frac{32\eta^2\kappa^2C_M^2}{(\kappa-\delta)^4},~
4MC_M^2
\right\}.
\]

Therefore, with $n \ge 30M^2(2B+1)C_\Sigma \{s\log p+o\log n\}$, we 
confirm that 
the restricted isometry part of \(\mathcal E_X\) follows from Proposition~1 with \(C_1=2B+1\), 
and the restricted incoherence part follows from Proposition~1 with \(C_1=B+1\).
Since $ B=\left(\frac{\kappa+\delta}{\kappa-\delta}\right)^2$ is fixed once \(M\) and \(\kappa \in(\delta,1)\) are fixed, these choices of \(C_1\) are admissible under Proposition~1, hence we have $\mathbf P(\mathcal E_X) \ge 1- 4p^{-2(1+B)s}$.

Similarly, with $n\ge 30M^2C_\Sigma\max(s\log p,o\log n)$, by Lemma~\ref{lemma: supp error} we have $\mathbf P(\mathcal E) \ge 1-2p^{-2s} - 2p^{-3}- 2n^{-3}$. And we conclude that
\[
\mathbf P(\mathcal E\cap\mathcal E_X)
\ge
1-O(p^{-2}+n^{-3}).
\]
The proof of Theorem~1 is then based on the event \(\mathcal E\cap\mathcal E_X\).

In the iteration algorithm, we take the learning rate $\eta \in \left[\frac{2M}{4M^2 + 1}, \frac{4M}{4M^2 + 1} \right]$ and the decay rate $\kappa \in (\delta,1)$. And we use the thresholds 
\begin{equation}\label{eq: lambda infty}
\begin{aligned}
\lambda_{\beta, \infty} &= \frac{16\eta \kappa }{\kappa- \delta} \sqrt{\frac{M \log p}{n}} + \frac{48\eta^2 \kappa^2 C_M }{(\kappa- \delta)^3} \sqrt{\frac{s\log p+ o\log n}{ns}} \sqrt{\frac{o \log n}{n}} \\
& {\asymp  \sqrt{\frac{\log p}{n}} +  {\frac{o \log n}{n\sqrt s}} }, \\
\lambda_{\theta, \infty} &= \frac{12\eta \kappa }{\kappa- \delta} \sqrt{\frac{\log n}{n}} + \frac{64\eta^2 \kappa^2 C_M \sqrt M }{(\kappa- \delta)^3} \sqrt{\frac{s\log p+ o\log n}{no}} \sqrt{\frac{s \log p}{n}}\\
& {\asymp  \sqrt{\frac{\log n}{n}} +  {\frac{s \log p}{n\sqrt o}} }.
\end{aligned}
\end{equation}
For the given thresholds $\lambda_{\beta,\infty}$ and $\lambda_{\theta,\infty}$ in \eqref{eq: lambda infty}, we can always choose sufficiently large initial thresholds $\lambda_{\beta,0}~ (>\lambda_{\beta,\infty})$ and $\lambda_{\theta,0}~(>\lambda_{\theta,\infty})$ satisfying 
\begin{equation}\label{eq: lambda0s}
    \sqrt s \lambda_{\beta,0} > \|\beta^* \|_2, \quad \sqrt o \lambda_{\theta,0} > \|\theta^* \|_2, \quad
 \frac{\lambda_{\beta,0}}{\lambda_{\theta,0}} 
  =\frac{\lambda_{\beta,\infty}}{\lambda_{\theta,\infty}}.
\end{equation}
The detailed threshold initialization is provided in Section \ref{sec: th1 initial}, where we utilize the conditions of Theorem 1 and the high-probability restricted incoherence property from Proposition 1 to establish a feasible range for the initial thresholds.

Additionally, by the decomposition, we define 
\begin{equation}\label{eq: decomposition}
\begin{aligned}
H_{\beta}^{t+1} :=& {\beta}^t + \frac{\eta}{n}X^{\top}(Y - X {\beta}^t - \sqrt{n} {\theta}^t) \\
= &\beta^* + \Phi(\beta^* - \beta^t) + \frac{\eta}{\sqrt n} X^\top (\theta^* - \theta^t) + \eta\Xi \in \mathbb R^p,\\
H_{\theta}^{t+1} :=& \theta^t + \frac{\eta}{\sqrt{n}}(Y - X\beta^t - \sqrt{n}\theta^t)\\
=& \theta^* + (\eta-1) (\theta^* - \theta^t) + \frac{\eta}{\sqrt n} X (\beta^* - \beta^t) + \frac{\eta}{\sqrt n} \xi \in \mathbb R^n.
\end{aligned}
\end{equation}

\subsection{Mathematical induction}\label{sec: MI}
Recall $S^* = supp(\beta^*)$ and $O^* = supp(\theta^*)$.
We aim to use mathematical induction to prove the following
\begin{align}
\| \beta^t_{(S^*)^c} \|_0  < B s,& \quad  \|\theta^{t}_{(O^*)^c}\|_0 <   Bo,\label{support number}\\
\| \beta^t - \beta^* \|_2  \le   (\sqrt B+1) \cdot \sqrt s \lambda_{\beta, t},& \quad
\| \theta^{t} - \theta^* \|_2  \le  (\sqrt B+1) \cdot \sqrt o \lambda_{\theta, t}.\label{upper bound}
\end{align}
hold for all $t\ge 0$.

First, by \eqref{eq: lambda0s}, we guarantee that both \eqref{support number} and \eqref{upper bound} hold at $t=0$, with $\beta^0 = \mathbf 0_p$ and $\theta^0 = \mathbf0_n$.
By using mathematical induction, assume that \eqref{support number} and \eqref{upper bound} hold at iteration $t$ for some $t \ge 0$, and then we aim to prove that they remain valid at iteration $t+1$.

\textbf{Proof of sparsity}\quad  In the $(t+1)$-th iteration, we first prove \eqref{support number} reasoning by the absurd. 
Assume that $\|\beta^{t+1}_{(S^*)^c}\|_0 \ge Bs$ or $\|\theta^{t+1}_{(O^*)^c}\|_0 \ge Bo$ holds at first. 
It then follows that there exist subsets $\tilde{S} \subset (S^*)^c$ or $\tilde{O} \subset (O^*)^c$ such that $\|\tilde{S}\|_0 = Bs$ or $\|\tilde{O}\|_0 = Bo$, satisfying
\begin{equation}\label{eq: abs1}
    Bs \lambda_{\beta,t+1}^2 < \sum_{i \in \tilde{S}}(H_{\beta,i}^{t+1})^2 \mathbf{1}\{ |H_{\beta,i}^{t+1}|\geq \lambda_{\beta,t+1} \}
\end{equation}
or
\begin{equation}\label{eq: abs2}
    B   o \lambda_{\theta,t+1}^2 < \sum_{j \in \tilde{O}}(H_{\theta,j}^{t+1})^2 \mathbf{1}\{ |H_{\theta,j}^{t+1}|\geq \lambda_{\theta,t+1} \}.
\end{equation}

Since $\beta_{\tilde{S}}^* = \mathbf 0_{Bs}$ and $\theta_{\tilde{O}}^* = \mathbf 0_{Bo}$, under event $\mathcal E$, if \eqref{eq: abs1} holds, by decomposition \eqref{eq: decomposition} we have
\begin{equation}\label{absurd beta}
\begin{aligned}
    \sqrt{Bs}\lambda_{\beta,t+1}&\leq  \sqrt{\sum_{i \in \tilde{S}}\langle \Phi_{\cdot i}, \beta^* - \beta^t \rangle^2} + \sqrt{\sum_{i \in \tilde{S}}\frac{\eta^2}{n}\langle X_{\cdot i}, \theta^* - \theta^t \rangle^2} + \eta \sqrt{Bs} \| \Xi\|_\infty\\
    &\overset{(i)}{\leq} \delta \|\beta^* - \beta^t\|_2+ \eta (\sqrt B +1)f\cdot \|\theta^* - \theta^t\|_2 + \eta \sqrt{Bs} \| \Xi\|_\infty \\
    &\overset{(ii)}{<}  \delta (\sqrt B+1) \cdot \sqrt s \lambda_{\beta, t}  +   \eta   (\sqrt B+1)^2 f \cdot \sqrt o \lambda_{\theta, t}   +  4 \eta (\sqrt B +1) \sqrt{ s} \sqrt{\frac{M\log p}{n}} \\
    &\overset{(iii)}{\leq}  \frac{2 \delta}{\kappa-\delta}\cdot \sqrt s \lambda_{\beta, t+1} + \frac{4 \eta \kappa}{(\kappa-\delta)^2} \cdot f\sqrt o \lambda_{\theta, t+1} +  \frac{8 \eta \kappa}{\kappa-\delta}  \sqrt{ s} \sqrt{\frac{M\log p}{n}}, 
\end{aligned}
\end{equation}
where:
\begin{itemize}
    \item Inequality (i) follows from the fact 
    \begin{equation}\label{eq: inner products}
\begin{aligned}
    &\quad \sum_{i \in \tilde{S}}\langle \Phi_{\cdot i}, \beta^* - \beta^t \rangle^2 \leq \sum_{i \in S^{'}}\langle \Phi_{\cdot i}, \beta^* - \beta^t \rangle^2~ \quad\text{ (define } S^{'} = \tilde{S}\cup\text{supp}(\beta^* - \beta^t) ) \\
    &   =~ (\beta^* - \beta^t)^\top_{S^{'}}\Phi^\top_{S^{'}S^{'}}\Phi_{S^{'}S^{'}}(\beta^* - \beta^t)_{S^{'}}~\leq~ \|\Phi_{S^{'}S^{'}}\|_2^2\|\beta^* - \beta^t\|_2^2\\
    &=  \left\{ \left|\Lambda_{\max} \left( \frac{\eta (X^\top X)_{S' S^{'}}}n \right) - 1 \right| \bigvee \left| 1 - \Lambda_{\min} \left( \frac{\eta (X^\top X)_{S' S^{'}}}n \right) \right| \right\}^2 \|\beta^* - \beta^t\|_2^2\\
    &\leq (1- \eta/(2M) )^2\|\beta^* - \beta^t\|_2^2 \quad\text{(By } \mathcal E_X \text{ and the range of } \eta \text{)}\\
    &\leq \delta^2\|\beta^* - \beta^t\|_2^2, \quad\text{(Equation \eqref{eq: B delta})}
\end{aligned}
\end{equation}
and 
\begin{equation}\label{eq: inner products 2}
\begin{aligned}
    \frac{1}{n}\sum_{i \in \tilde{S}}\langle X_{\cdot i}, \theta^* - \theta^t \rangle^2 &= \frac{1}{n}(\theta^* - \theta^t)^\top_{O'}X_{O'\tilde{S}}X^\top_{O'\tilde{S}}(\theta^* - \theta^t)_{O'}~ \quad\text{ (define } O^{'} = \text{supp}(\theta^* - \theta^t) )\\
    &\leq \frac{1}{n}\|X_{O'\tilde{S}}\|_2^2\cdot \|\theta^* - \theta^t\|_2^2 \\
    &\leq (\sqrt B +1)^2 f^2 \cdot \|\theta^* - \theta^t\|_2^2. \quad\text{(By } \mathcal E_X \text{)}
    \end{aligned}
\end{equation}

    \item Inequality (ii) follows from the event $\mathcal E$ and the assumption \eqref{upper bound} at the $t$-th iteration.

    \item Inequality (iii) follows the relationship:
\begin{equation}\label{eq: relation b lambda}
\sqrt B + 1 = \frac{2 \kappa}{\kappa - \delta}, \quad
\lambda_{\beta,t+1} = \max(\kappa \lambda_{\beta,t}, ~ \lambda_{\beta, \infty}), \quad
\lambda_{\theta,t+1} = \max(\kappa \lambda_{\theta,t}, ~ \lambda_{\theta, \infty}).
\end{equation}
\end{itemize}

Similarly, if \eqref{eq: abs2} holds, we have
\begin{equation}\label{absurd theta}
\begin{aligned}
    \sqrt{Bo}\lambda_{\theta, t+1}&\leq  |\eta -1| \cdot \| \theta^*- \theta^t\|_2 + \sqrt{\sum_{j \in \tilde{O}}\frac{\eta^2}{n}\langle X_{j \cdot}^\top, \beta^* - \beta^t \rangle^2} + \eta \sqrt{\sum_{j \in \tilde{O}}\frac{1}{n}\xi^2_j}\\
    &\leq \delta\| \theta^*- \theta^t\|_2  + \eta (\sqrt B +1)f  \| \beta^* - \beta^t\|_2  +  3\eta \sqrt{ B o} \sqrt{\frac{\log n}{n}}\\
    &< \frac{2 \delta}{\kappa-\delta}\cdot \sqrt o \lambda_{\theta, t+1} + \frac{4 \eta \kappa}{(\kappa-\delta)^2} \cdot f\sqrt s \lambda_{\beta, t+1} +  \frac{6\eta \kappa}{\kappa-\delta} \sqrt{ o} \sqrt{\frac{\log n}{n}}, 
\end{aligned}
\end{equation}
where: 
\begin{itemize}
    \item The second inequality follows from the fact
    \begin{equation}\label{eq: inner products 3}
\begin{aligned}
     \frac{1}{n}\sum_{j \in \tilde{O}}\langle X_{j \cdot}^\top, \beta^* - \beta^t \rangle^2 &\leq \frac{1}{n}(\beta^* -\beta^t)^\top_{S^{''}} X_{\tilde{O} S^{''}}^\top X_{\tilde{O} S^{''}}(\beta^* -\beta^t)_{S^{''}} ~ \quad\text{ (define } S^{''} =  \text{supp}(\beta^* - \beta^t) )\\
     &\leq \frac{1}{n}\|X_{\tilde{O} S^{''}}\|_2^2 \cdot \|\beta^* -\beta^t\|_2^2\\
     &\leq (\sqrt B +1)^2 f^2\cdot \|\beta^* - \beta^t\|_2^2 \quad\text{(By } \mathcal E_X ),
\end{aligned}
\end{equation}
and the fact $|\eta-1| < \delta$ in the case $\eta \in \left[\frac{2M}{4M^2 + 1}, \frac{4M}{4M^2 + 1} \right]$.

\item The derivation of the last inequality is analogous to that of (ii) and (iii) in \eqref{absurd beta}. 
\end{itemize}

Combining \eqref{absurd beta} and \eqref{absurd theta}, to prove $\|\beta^{t+1}_{(S^*)^c}\|_0 < Bs$ and $\|\theta^{t+1}_{(O^*)^c}\|_0 < Bo$ by contradiction, it suffices to show that:
\begin{equation}\label{eq: keypoint}
\begin{cases}
\frac{2 \delta}{\kappa-\delta}\cdot \sqrt s \lambda_{\beta, t+1} + \frac{4 \eta \kappa}{(\kappa-\delta)^2} \cdot f\sqrt o \lambda_{\theta, t+1} +  \frac{8 \eta \kappa}{\kappa-\delta}  \sqrt{ s} \sqrt{\frac{M\log p}{n}} \le \sqrt{ Bs} \lambda_{\beta, t+1},\\
\frac{2 \delta}{\kappa-\delta}\cdot \sqrt o \lambda_{\theta, t+1} + \frac{4 \eta \kappa}{(\kappa-\delta)^2} \cdot f\sqrt s \lambda_{\beta, t+1} +    \frac{6\eta \kappa}{\kappa-\delta} \sqrt{ o} \sqrt{\frac{\log n}{n}} \le \sqrt{ Bo} \lambda_{\theta, t+1}. \\
\end{cases}
\end{equation}
We first assume that there exist suitable sequence $\{(\lambda_{\beta,t},~\lambda_{\theta,t} ) \}_{t \ge 0}$ satisfying this system \eqref{eq: keypoint}, and we will construct its explicit form in Section \ref{sec: threshold seq}.

\textbf{Proof of error bounds}\quad Once the sparsity results are established in the $(t+1)$-th iteration, we turn to derive upper bounds on estimation errors.
We establish that
\begin{equation}\label{eq: th1 upper bound1}
\begin{aligned}
&\| \beta^{t+1} - \beta^*\|_2 \\
&\le \Bigg\{  \sum_{i \in S^*}\left( - H_{\beta,i}^{t+1} \mathbf1(|H_{\beta,i}^{t+1}|< \lambda_{\beta,t+1}) 
   + \langle \Phi_{\cdot i}, \beta^* - \beta^t \rangle + \frac{\eta}{\sqrt n}\langle X_{\cdot i}, \theta^* - \theta^t \rangle + \eta \Xi_i \right)^2  \\
&\quad \quad\quad \left.+ \sum_{i \in S^{t+1}\setminus S^*}\left( \langle \Phi_{\cdot i}, \beta^* - \beta^t \rangle +
   \frac{\eta}{\sqrt n}\langle X_{\cdot i}, \theta^* - \theta^t \rangle + \eta \Xi_i \right)^2 \right\}^{1/2}\\
&\le \sqrt{\sum_{i \in S^*} \left( H_{\beta,i}^{t+1} \right)^2 \mathbf1(|H_{\beta,i}^{t+1}|< \lambda_{\beta,t+1})}
    + \sqrt{\sum_{i \in S^{t+1} \cup S^*} \langle \Phi_{\cdot i}, \beta^* - \beta^t \rangle^2 } \\
&\quad \quad + \eta \sqrt{\sum_{i \in S^{t+1} \cup S^*} \frac{1}{n}\langle X_{\cdot i}, \theta^* - \theta^t \rangle^2 } + \eta\sqrt{\sum_{i \in S^{t+1} \cup S^*} \Xi_i^2 } \\
    &\overset{(i)}\leq \sqrt{s}\lambda_{\beta, t+1} + \delta\| \beta^t - \beta^*\|_2 + \eta (\sqrt B +1)f\cdot \|\theta^* - \theta^t\|_2 + 4 \eta (\sqrt B+1) \sqrt{s} \sqrt{\frac{M\log p}{n}}  \\
    &\overset{(ii)}\leq \sqrt{s}\lambda_{\beta, t+1} + \frac{2 \delta}{\kappa-\delta}\cdot \sqrt s \lambda_{\beta, t+1} + \frac{4 \eta \kappa}{(\kappa-\delta)^2} \cdot f\sqrt o \lambda_{\theta, t+1} +  \frac{8 \eta \kappa}{\kappa-\delta}  \sqrt{ s} \sqrt{\frac{M\log p}{n}},
\end{aligned}
\end{equation}
where recall $S^{t+1} = \text{supp}(\beta^{t+1}),\ S^{*} = \text{supp}(\beta^{*})$.
Here, inequality (i) follows from \eqref{eq: inner products} and \eqref{eq: inner products 2} again, and inequality (ii) follows analogously to (ii) and (iii) in \eqref{absurd beta}.
Similar to \eqref{eq: th1 upper bound1}, we derive that
\begin{equation}\label{eq: th1 upper bound2}
\begin{aligned}
\| \theta^{t+1} - \theta^*\|_2 &\leq \sqrt{o}\lambda_{\theta, t+1} + \delta\| \theta^t - \theta^* \|_2 + \eta (\sqrt B +1) f \| \beta^t - \beta^* \|_2 + 3\eta \sqrt{(B+1) o} \sqrt{\frac{\log n}{n}} \\
    &< \sqrt{o}\lambda_{\theta, t+1} + \frac{2 \delta}{\kappa-\delta}\cdot \sqrt o \lambda_{\theta, t+1} + \frac{4 \eta \kappa}{(\kappa-\delta)^2} \cdot f\sqrt s \lambda_{\beta, t+1} +  \frac{6\eta \kappa}{\kappa-\delta} \sqrt{ o} \sqrt{\frac{\log n}{n}}.
\end{aligned}
\end{equation} 
Therefore, to ensure that the $\ell_2$ error bounds after the $(t+1)$-th iteration satisfies \eqref{upper bound}, it suffices for the sequence $\{ ( \lambda_{\beta,t}, \lambda_{\theta,t} ) \}_{t \ge 0}$ to satisfy
\begin{equation}
\begin{cases}
\sqrt s \lambda_{\beta, t+1} + \frac{2 \delta}{\kappa-\delta}\cdot \sqrt s \lambda_{\beta, t+1} + \frac{4 \eta \kappa}{(\kappa-\delta)^2} \cdot f\sqrt o \lambda_{\theta, t+1} +  \frac{8 \eta \kappa}{\kappa-\delta}  \sqrt{ s} \sqrt{\frac{M\log p}{n}} \le (\sqrt B+1) \sqrt{ s} \lambda_{\beta, t+1},\\
\sqrt{o}\lambda_{\theta, t+1}  + \frac{2 \delta}{\kappa-\delta}\cdot \sqrt o \lambda_{\theta, t+1} + \frac{4 \eta \kappa}{(\kappa-\delta)^2} \cdot f \sqrt s \lambda_{\beta, t+1} +    \frac{6\eta \kappa}{\kappa-\delta} \sqrt{ o} \sqrt{\frac{\log n}{n}} \le (\sqrt B+1) \sqrt{ o} \lambda_{\theta, t+1},\\
\end{cases}
\end{equation}
which is equivalent to the system \eqref{eq: keypoint}. Consequently, we next focus on constructing a suitable set of solutions for $\{ ( \lambda_{\beta,t}, \lambda_{\theta,t} ) \}_{t \ge 0}$.

\subsection{Appropriate threshold sequence}\label{sec: threshold seq}
Define 
$$
W_1 := \frac{4 \eta \kappa}{(\kappa-\delta)^2},
\quad W_2 := \frac{8 \eta \kappa}{\kappa-\delta} \sqrt{\frac{Ms\log p}{n}},
\quad W_3:= \frac{6\eta \kappa}{\kappa-\delta} \cdot\sqrt{\frac{o\log n}{n}},
$$
and note that $\sqrt B = \frac{\kappa+ \delta}{\kappa - \delta}$. Then the system \eqref{eq: keypoint} reduces to the inequalities 
\begin{equation}\label{eq: concise}
\begin{cases}
W_1 \cdot f\sqrt o \lambda_{\theta, t+1} +  W_2 \le \sqrt{s} \lambda_{\beta, t+1},\\
W_1 \cdot f\sqrt s \lambda_{\beta, t+1} +   W_3 \le \sqrt{ o} \lambda_{\theta, t+1}. \\
\end{cases}
\end{equation}
The feasible region for $(\sqrt s\lambda_{\beta, t+1}, \sqrt{ o} \lambda_{\theta, t+1})$ is illustrated as the shaded area in Figure \ref{fig: solution}. 
Given the update rule
\begin{equation*}
    \begin{aligned}
        \lambda_{\beta,t+1} &= (\kappa \times\lambda_{\beta,t}) \vee\lambda_{\beta,\infty},\\
        \lambda_{\theta,t+1} &= (\kappa \times\lambda_{\theta,t}) \vee\lambda_{\theta,\infty},
    \end{aligned}
\end{equation*}
maintaining the ratio $\frac{\lambda_{\beta,0}}{\lambda_{\theta,0}} = \frac{\lambda_{\beta,\infty}}{\lambda_{\theta,\infty}}$ ensures that the iterate trajectory $\{ ( \lambda_{\beta,t}, \lambda_{\theta,t} ) \}_{t \ge 0}$ follows a linear path toward the origin, such as the solid purple or dashed red rays shown in Figure \ref{fig: solution}.
Specifically, in the case $1-W_1^2 f^2 > 1/2$, i.e.,
$$
n > 32 \frac{\eta^2 \kappa^2 C_M^2}{(\kappa-\delta)^4} (s\log p + o \log n),
$$
the target purple endpoint $\left( \sqrt s\lambda_{\beta,\infty} = 2W_2 +2 f W_1 W_3 ,~ \sqrt o \lambda_{\theta, \infty} =2W_3 +2f W_1 W_2 \right)$ lies within the feasible region. 
Consequently, the purple ray (passing through this endpoint and maintaining $\frac{\lambda_{\beta,t}}{\lambda_{\theta,t}} = \frac{\lambda_{\beta,\infty}}{\lambda_{\theta,\infty}}$) constitutes a valid threshold update trajectory. 

Therefore, by using the thresholds in \eqref{eq: lambda infty}, we prove that 
$$
\begin{aligned}
\| \beta^{t+1}_{S^c} \|_0  < B s,& \quad  \|\theta^{t+1}_{O^c}\|_0 <   B o,\\
   \| \beta^{t+1} - \beta^* \|_2  \le   (\sqrt B+1) \cdot \sqrt s \lambda_{\beta, t+1},& \quad
\| \theta^{t+1} - \theta^* \|_2  \le  (\sqrt B+1) \cdot \sqrt o \lambda_{\theta, t+1} 
\end{aligned}
$$
hold simultaneously in the ($t+1$)-th iteration, which completes the proof of Theorem 1.

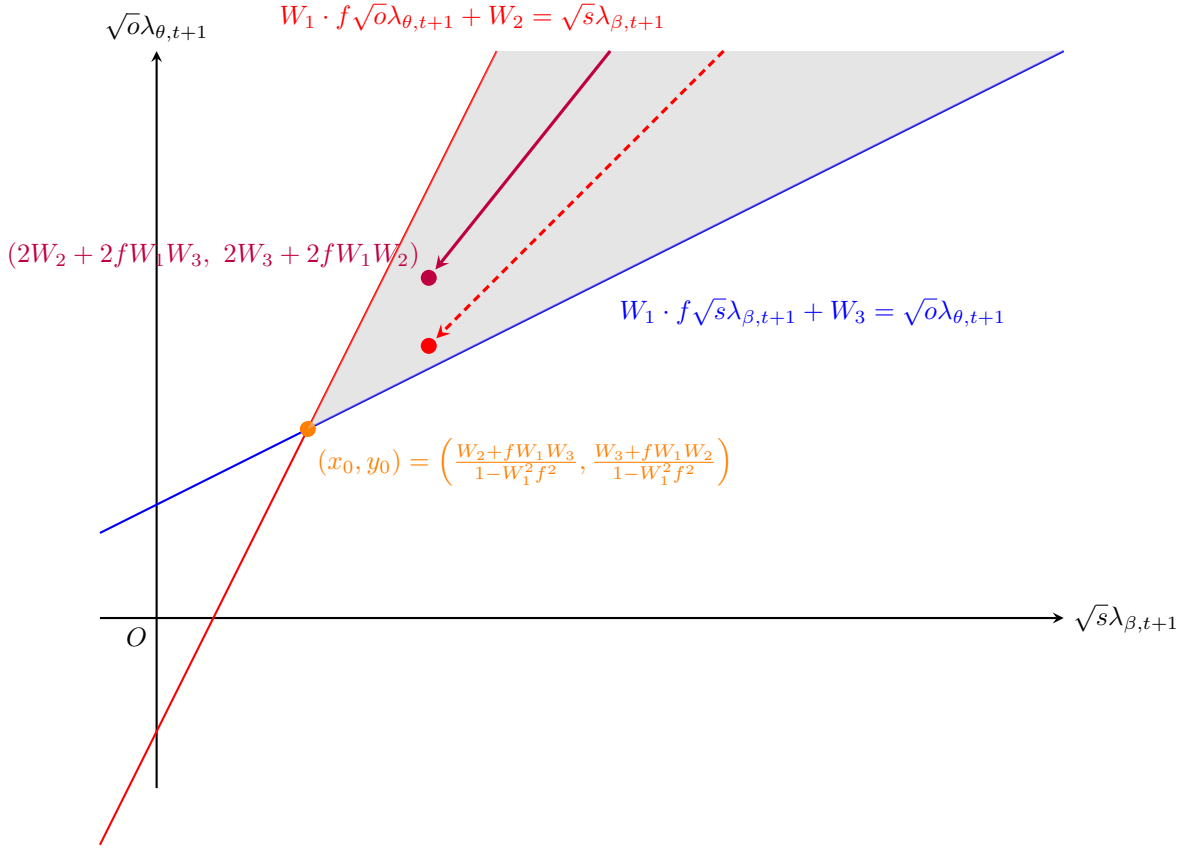
\begin{figure}
    \centering
\begin{tikzpicture}[scale=1.5]
    \draw[->, >=stealth, thick] (-0.5,0) -- (8,0) node[right] {$\sqrt{s} \lambda_{\beta,t+1}$};
    \draw[->, >=stealth, thick]  (0,-1.5) -- (0,5) node[above] {$\sqrt{o} \lambda_{\theta,t+1}$};
    \node at (0,0) [below left] {$O$};
     
    \draw[blue, thick, name path=line1] plot [domain=-0.5:8] (\x, {0.5*\x + 1});
    \node[blue, right] at (4, 2.7) {$W_1 \cdot f\sqrt s \lambda_{\beta, t+1} + W_3 = \sqrt{ o} \lambda_{\theta, t+1} $ };
    
    \draw[red, thick, name path=line2] plot [domain=-0.5:3] (\x, {2*\x - 1});
    \node[red, right] at (1, 5.3) {$W_1 \cdot f\sqrt o \lambda_{\theta, t+1} +  W_2 = \sqrt{s} \lambda_{\beta, t+1}$};
    
    \path[name intersections={of=line1 and line2, by=P}];
    
   \fill[orange] (P) circle (2pt) node[below right] {$(x_0, y_0) = \left(\frac{W_2 + f W_1 W_3}{1-W_1^2 f^2}, \frac{W_3 + f W_1 W_2}{1-W_1^2 f^2} \right)$};

    \begin{scope}
        \clip (P) rectangle (8,5);
        \fill[black!20, opacity=0.5] 
            plot [domain=1.33:8] (\x, {0.5*\x + 1}) -- 
            plot [domain=5:1.33] (\x, {2*\x - 1}) -- cycle;
    \end{scope}

    \coordinate (B) at (4, 5);

    \coordinate (A) at (2.4, 3);
    \fill[purple] (A) circle (2pt);
    
    \node (L) at (0.5, 3.2) [purple, font=\normalsize] {$\left( 2W_2 +2 f W_1 W_3 ,~ 2W_3 +2f W_1 W_2 \right) $};
    
    \draw[->, >=stealth, very thick, purple, shorten >=4pt] (B) -- (A);

    \coordinate (C) at (2.4, 2.4);
    \fill[red] (C) circle (2pt);
    \coordinate (D) at (5, 5);
    \draw[->, >=stealth, very thick, red, densely dashed, shorten >=4pt] (D) -- (C);

\end{tikzpicture}
    \caption{Feasible region (shaded) and two exemplary threshold update trajectories (solid purple and dashed red).}
    \label{fig: solution}
\end{figure}

\subsection{Threshold initialization}\label{sec: th1 initial}
Here we provide some practical guidance for selecting the initial tuning parameters $\lambda_{\beta,0}$ and $\lambda_{\theta,0}$. 
With a probability greater than $1-O(p^{-2}+n^{-3})$, we have the decomposition
\begin{equation}
\begin{aligned}
\sqrt s \left\| \frac1nX^\top Y \right\|_\infty 
\ge & \left\| \left( \frac1nX^\top Y\right)_{S^*} \right\|_2\\
\ge &\left\| \left( \frac1nX^\top X\right)_{S^*,S^*} \beta^*_{S^*} \right\|_2 - \left\| \left( \frac1{\sqrt n}X^\top \theta^*\right)_{S^* } \right\|_2 -  \left\| \left(\Xi\right)_{S^* } \right\|_2 \\
\ge & \frac{1}{2M}\left\| \beta^* \right\|_2 - f\left\| \theta^* \right\|_2 - 4\sqrt M \sqrt{\frac{s\log p}{n}},
\end{aligned}
\end{equation}
and 
\begin{equation}
\begin{aligned}
\sqrt o \left\| \frac{1}{\sqrt{n}} Y \right\|_\infty 
\ge & \left\| \left( \frac{1}{\sqrt{n}} Y\right)_{O^*} \right\|_2\\
\ge &\left\|\theta^*\right\|_2 - \left\| \left( \frac1{\sqrt n} X \beta^*\right)_{O^* } \right\|_2 -  \left\| \frac1{\sqrt n} \xi_{O^* } \right\|_2 \\
\ge & \left\|\theta^*\right\|_2 - f\left\| \beta^* \right\|_2 - 3\sqrt{\frac{o\log n}{n}}.
\end{aligned}
\end{equation}
For both inequalities, the final step relies on the restricted incoherence property given in Proposition 1, with $C_1 = 1$.
And the sample-size assumption on which Proposition 1 depends is fulfilled under the sample condition presented in Section~\ref{sec: th1 pre}. 

By combining these inequalities, we obtain 
$$
\sqrt s \left\| \frac1nX^\top Y \right\|_\infty + 4\sqrt M \sqrt{\frac{s\log p}{n}}
+ f \sqrt o \left\| \frac{1}{\sqrt{n}} Y \right\|_\infty  + 3f\sqrt{\frac{o\log n}{n}}
\ge \left( \frac1{2M} - f^2\right) \left\| \beta^* \right\|_2.
$$
Then, under the assumption $n \ge 4M C_M^2(s\log p + o \log n)$, it follows that 
$$
\frac{\| \beta^*\|_2}{\sqrt s} \le 4M\left( 4 +  \left\| \frac1nX^\top Y \right\|_\infty + \left\|  Y \right\|_\infty \right).
$$
Similarly, by $4Mf^2 \le 1$, $M >1$ and 
$$
\sqrt o \left\| \frac{1}{\sqrt{n}} Y \right\|_\infty +3\sqrt{\frac{o\log n}{n}}
+ 2Mf \sqrt s \left\| \frac1nX^\top Y \right\|_\infty  + 2Mf 4\sqrt M \sqrt{\frac{s\log p}{n}} \ge (1-2Mf^2) \left\|\theta^*\right\|_2,
$$
we have 
$$
\frac{\| \theta^*\|_2}{\sqrt o} \le 8\sqrt M+ 2 \left\| \frac1{\sqrt{n}} Y \right\|_\infty +  \left\| \frac1{\sqrt n}X^\top Y \right\|_\infty .
$$
Consequently, with probability at least $1-O(p^{-2}+n^{-3})$, we can construct valid initial thresholds $\lambda_{\beta,0}$ and $\lambda_{\theta,0}$ based on
$$
\begin{aligned}
    \lambda_{\beta,0} \ge &~ 4M\left( 4 +  \left\| \frac1nX^\top Y \right\|_\infty + \left\|  Y \right\|_\infty \right)+ \lambda_{\beta,\infty},\\
\lambda_{\theta,0} \ge&~ 8\sqrt M+ 2 \left\| \frac1{\sqrt{n}} Y \right\|_\infty +  \left\| \frac1{\sqrt n}X^\top Y \right\|_\infty + \lambda_{\theta,\infty}.
\end{aligned}
$$

\section{Proof of Theorem 2}\label{supp: Th2}
Throughout this proof, Assumption~2 is used in the same concrete form as in the proof of Theorem~1, i.e., $ n\ge C_{\mathrm{Th1}} (s\log p+o\log n)$.
The proof of Theorem~2 proceeds in two steps.
First, we introduce some useful preliminaries.
Then, we prove that, {under Assumption 3}, a carefully chosen stopping time yields an error bound sharper than Theorem 1.
\subsection{Preliminary}\label{sec: prove of th2-1}
By Lemma \ref{lemma: supp error}, we have
\begin{align*}
\mathbf P\left\{\underbrace{ \frac1n \left\|  X_{\cdot S^*}^\top \xi \right\|_2 \le  \sqrt{\frac{4Ms +6M \log(1/\varrho)}n} }_{=:\mathcal E'} \right\}
\ge 1- O\left(p^{-2s} + \varrho \right),
\end{align*}
and this proof is based on the event $\mathcal E' \cap \mathcal E \cap \mathcal E_X$.

In the second-stage iteration, we apply the thresholds of the same values as the endpoints in \eqref{eq: lambda infty}:
$$
\lambda_{\beta} = \lambda_{\beta, \infty},
\quad \lambda_{\theta} = \lambda_{\theta, \infty},
$$ 
therefore, by Theorem 1, with a probability greater than $1 - O(p^{-2} + n^{-3})$ we have the following 
\begin{align}
\| \tilde{\beta}^{t}_{S^{*c}} \|_0 \le B s,& \quad \| \tilde{\theta}^{t}_{O^{*c}} \|_0 \le B o,\label{support number5}\\
\| \tilde{\beta}^t - \beta^* \|_2  \le (\sqrt B+1) \cdot \sqrt s \lambda_{\beta},& \quad
\| \tilde{\theta}^{t} - \theta^* \|_2  \le  (\sqrt B+1) \cdot \sqrt o \lambda_{\theta}, \label{upper bound5}
\end{align}
hold for all $t\ge 0$, regardless of whether Assumption 3 holds (the definition of $B$ follows \eqref{eq: B delta}).
Moreover, as Assumption 3 holds, i.e.,
\begin{equation}\label{eq: betamin in th2}
\min_{i \in S^*} |\beta_i^*| \ge  \left( \frac{4 \kappa}{\kappa-\delta} + \frac{\kappa-\delta}{4 \kappa} \right) \cdot \lambda_\beta ,
\end{equation}
we can further prove a sharper error bound of $\tilde \beta^t$. 
Before we prove this refined bound, we rewrite the decomposition as
\begin{equation} 
\begin{aligned}
\tilde H_{\beta}^{t+1} :=& \tilde {\beta}^t + \frac{\eta}{n}X^{\top}(Y - X \tilde  {\beta}^t - \sqrt{n} \tilde {\theta}^t) \\
= &\beta^* + \Phi(\beta^* - \tilde \beta^t) + \frac{\eta}{\sqrt n} X^\top (\theta^* - \tilde \theta^t) + \eta\Xi \in \mathbb R^p,\\
\tilde H_{\theta}^{t+1} :=&\tilde  \theta^t + \frac{\eta}{\sqrt{n}}(Y - X \tilde \beta^t - \sqrt{n}\tilde \theta^t)\\
=& \theta^* + (\eta-1) (\theta^* - \tilde \theta^t) + \frac{\eta}{\sqrt n} X (\beta^* -\tilde  \beta^t) + \frac{\eta}{\sqrt n} \xi \in \mathbb R^n.
\end{aligned}
\end{equation}

\subsection{Sharper bound under Assumption 3}\label{sec: prove of th2-3}
In the second-stage iteration, by \eqref{support number5} we guarantee the sparsity of solution sequences $\left\{ \tilde \beta^t \right\}_{t \ge 0}$ and $\left\{ \tilde \theta^t \right\}_{t \ge 0}$. 
Then, under the event $\mathcal E' \cap \mathcal E \cap \mathcal E_X$, for every $ t \ge 0$ we have
\begin{equation}\label{eq: th6 beta-min}
    \begin{aligned}
    &\sum_{i \in S^*} \left( \tilde H_{\beta,i}^{t+1} \right)^2 \mathbf1(| \tilde H_{\beta,i}^{t+1}|< \lambda_{\beta})\\
\le & \sum_{i \in S^*} \lambda_{\beta}^2 \cdot \mathbf1 \left( |\beta^*_i| - \left|\langle \Phi_{\cdot i}, \beta^* - \tilde \beta^t \rangle +
    \frac{\eta}{\sqrt n}\langle X_{\cdot i}, \theta^* - \tilde \theta^t \rangle \right| -  \eta\left|\Xi_i \right| < \lambda_{\beta} \right)\\
\le & \sum_{i \in S^*} \lambda_{\beta}^2 \cdot \mathbf1 \left( \left|\langle \Phi_{\cdot i}, \beta^* - \tilde \beta^t \rangle +
    \frac{\eta}{\sqrt n}\langle X_{\cdot i}, \theta^* -\tilde \theta^t \rangle \right| > |\beta_i^*| - \lambda_\beta - \eta \|\Xi\|_\infty \ge \frac{3 \kappa + \delta}{\kappa - \delta}\lambda_\beta \right)\\
\le & \left(\frac{\kappa - \delta}{3 \kappa + \delta} \right)^2 \sum_{i \in S^*} \left( \langle \Phi_{\cdot i}, \beta^* -\tilde \beta^t \rangle +
    \frac{\eta}{\sqrt n}\langle X_{\cdot i}, \theta^* - \tilde \theta^t \rangle \right)^2,
\end{aligned}
\end{equation}
where the second inequality follows from the signal condition \eqref{eq: betamin in th2} and the fact $\|\eta\Xi\|_\infty \le \frac{\kappa - \delta}{4 \kappa } \lambda_\beta$ under the event $\mathcal E$.

On the set $\tilde S^{t+1} \setminus S^*$, by $\|\eta\Xi\|_\infty \le \frac{\kappa - \delta}{4 \kappa } \lambda_\beta$ we have
\begin{equation}\label{absurd beta5}
\begin{aligned}
&\sum_{i \in \tilde S^{t+1} \setminus S^*} (\eta \Xi_i)^2 \mathbf{1}\{ | \tilde H_{\beta,i}^{t+1}|\geq \lambda_{\beta} \}\\
\leq& \sum_{i \in \tilde S^{t+1} \setminus S^* } (\eta \Xi_i)^2 \mathbf{1} \left\{  \left|\langle \Phi_{\cdot i}, \beta^* - \tilde{\beta}^t \rangle + \frac{\eta}{\sqrt n}\langle X_{\cdot i}, \theta^* - \tilde{\theta}^t \rangle \right| \ge  \lambda_\beta - \| \eta\Xi\|_\infty \ge \frac{3\kappa + \delta}{\kappa -\delta} \|\eta \Xi \|_\infty \right\} \\
\leq& \left( \frac{\kappa-\delta}{3 \kappa + \delta}\right)^2 \sum_{i \in \tilde S^{t+1} \setminus S^* } \left|\langle \Phi_{\cdot i}, \beta^* - \tilde{\beta}^t \rangle + \frac{\eta}{\sqrt n}\langle X_{\cdot i}, \theta^* - \tilde{\theta}^t \rangle \right|^2. 
\end{aligned}
\end{equation}
Therefore,
\begin{equation}\label{eq: th6 upper bound1}
\begin{aligned}
&\quad \quad \|\tilde \beta^{t+1} - \beta^*\|_2 \\
&\le \Bigg\{  \sum_{i \in S^*}\left( - \tilde H_{\beta,i}^{t+1} \mathbf1(| \tilde H_{\beta,i}^{t+1}|< \lambda_{\beta}) 
   + \langle \Phi_{\cdot i}, \beta^* - \tilde \beta^t \rangle + \frac{\eta }{\sqrt n}\langle X_{\cdot i}, \theta^* - \tilde \theta^t \rangle + \eta \Xi_i \right)^2  \\
&\quad \quad \left.+ \sum_{i \in \tilde S^{t+1}\setminus S^*}\left( \langle \Phi_{\cdot i}, \beta^* -\tilde \beta^t \rangle +
   \frac{\eta}{\sqrt n}\langle X_{\cdot i}, \theta^* - \tilde \theta^t \rangle + \eta\Xi_i \mathbf{1} \left( | \tilde H_{\beta,i}^{t+1}|\geq \lambda_{\beta} \right)  \right)^2 \right\}^{1/2}\\
&\le \sqrt{\sum_{i \in \tilde S^{t+1} \cup S^*} \langle \Phi_{\cdot i}, \beta^* - \tilde \beta^t \rangle^2 } 
    + \sqrt{\sum_{i \in \tilde S^{t+1} \cup S^*} \frac{\eta^2}n\langle X_{\cdot i}, \theta^* - \tilde \theta^t \rangle^2 } + \eta \sqrt{\sum_{i \in S^*} \Xi_i^2 } \\
&\quad + \sqrt{\sum_{i \in S^*}  \left( \tilde H_{\beta,i}^{t+1} \right)^2 \mathbf1(| \tilde H_{\beta,i}^{t+1}|< \lambda_{\beta})
 + \sum_{i \in \tilde S^{t+1}\setminus S^*} (\eta\Xi_i)^2  \mathbf{1} \left(| \tilde H_{\beta,i}^{t+1}|\ge  \lambda_{\beta} \right)  } \\
&\leq \frac{ 4 \kappa\delta}{3\kappa + \delta}  \| \tilde \beta^t - \beta^*\|_2 + \frac{ 8 \kappa^2 \eta}{(3\kappa + \delta)(\kappa - \delta)} \cdot f \| \tilde\theta^t - \theta^* \|_2
+ \eta \sqrt{\frac{4Ms +6M \log(1/\varrho)}n},
\end{aligned}
\end{equation}
where the last inequality follows from \eqref{eq: th6 beta-min}, \eqref{absurd beta5}, the sparsity results \eqref{support number5} and the event $\mathcal E'$.
Combining \eqref{eq: th6 upper bound1} and \eqref{upper bound5}, with a constant $C_1$ (depending only on $\kappa, \eta$, and $M$), we have
$$
\| \tilde \beta^{t+1} - \beta^*\|_2 
\le \frac{ 4 \kappa\delta}{3\kappa + \delta}  \| \tilde \beta^t - \beta^*\|_2
    + C_1\left( \sqrt{\frac{s + \log (1/\varrho)}{n} } + \frac{s \log p + o \log n}{n} \right),
$$
by $4 \kappa\delta < 3 \kappa + \delta$, we further get that
$$
\| \tilde \beta^{t} - \beta^*\|_2 
\le  \left( \frac{ 4 \kappa\delta}{3\kappa + \delta} \right)^t \cdot \| \tilde \beta^0 - \beta^*\|_2 
   + \frac{C_1}{1- \frac{4\kappa \delta}{3 \kappa + \delta}}  \left( \sqrt{\frac{s + \log (1/\varrho)}{n} } + \frac{s \log p + o \log n}{n} \right)
$$
holds for every $t\ge 0$. 
Additionally, from Theorem 1, we have
$$\mathbf P\left\{ \| \tilde \beta^0 - \beta^*\|_2^2  \lesssim \frac{s\log p}{n} + \frac{o^2 \log^{2 } n}{n^2} \right\} \ge 1- O\left(p^{-2} + n^{-3} \right).
$$
Therefore, with a sufficiently large constant $C_2$, for every $t \ge  C_2 \log n$, with a probability greater than $1-\varrho -O(p^{-2}+n^{-3})$ we get the sharper error bound: 
\begin{equation*}
\| \tilde \beta^{t} - \beta^*\|_2 
\lesssim \sqrt{\frac{s +\log (1/\varrho)}n} + \frac{s\log p+ o \log n}{n}.
\end{equation*}
Combining this bound and \eqref{support number5}, we complete the proof of Theorem 2.

\section{Proof of Corollary 1}\label{supp: Cor1}
\textbf{Sharper error bound:}
Under the beta-min Assumption 3 and taking $\varrho = p^{-2}$, from Theorem 2 we learn that
\begin{equation}\label{eq: sharp bound}
\| \tilde \beta^t  - \beta^*\|_2 
\le C\left( \sqrt{\frac{s + \log p}n} +  \frac{s\log p+ o \log n}{n}  \right),
\end{equation}
holds for every $t \ge  C_2 \log n$, with a probability greater than $1-O(p^{-2}+n^{-3})$.

\textbf{On true support:}
On the true support $S^* = \text{supp}(\beta^*)$, we have
\begin{equation}\label{eq: almost full}
\begin{aligned}
  \left | \text{supp}(\beta^*) - \text{supp}(\tilde{\beta}^t)\right | &= \sum_{i\in S^*}\mathbf{1}\{ | \tilde H^{t}_{\beta,i}| < \lambda_{\beta} \}\\
    &\overset{(i)}\le \sum_{i\in S^*} \mathbf1 \left( \left|\langle \Phi_{\cdot i}, \beta^* - \tilde \beta^t \rangle +
    \frac{\eta}{\sqrt n}\langle X_{\cdot i}, \theta^* -\tilde \theta^t \rangle \right| >  \frac{3 \kappa + \delta}{\kappa - \delta}\lambda_\beta \right)\\
&\lesssim \sum_{i\in S^*} \frac{ \left|\langle \Phi_{\cdot i}, \beta^* - \tilde \beta^t \rangle +
    \frac{\eta}{\sqrt n}\langle X_{\cdot i}, \theta^* -\tilde \theta^t \rangle \right|^2}{\lambda_{\beta}^2}\\
&\lesssim  \frac{\| \beta^* - \tilde{\beta}^t \|_2^2 + f^2 \| \theta^* - \tilde{\theta}^t \|_2^2 }{ \lambda_\beta^2 }\\
    &\overset{(ii)}\lesssim \frac{\frac{s+ \log p}{n}+\left( \frac{s\log p}{n} \right)^2 + \left( \frac{o\log n}{n} \right)^2 }{\frac{s\log p}{n} + \left( \frac{o\log n}{n} \right)^2 }\cdot s\\
    & \overset{(iii)}= \left(\frac1{\log p} + \frac1s + \frac{s\log p}{n}  + o(1) \right) \cdot s\\
    &= o(s),
\end{aligned}
\end{equation}
where inequality (i) comes from \eqref{eq: th6 beta-min}, inequality (ii) comes from the sharper bound \eqref{eq: sharp bound} and \eqref{upper bound5}, and inequality (iii) comes from the condition $ n \succ  s\log p \succ \frac{o^2\log^{2}n}{n }$.

\textbf{On non-support:} 
By $ \sum_{i\notin S^*}\mathbf{1}(\tilde{\beta}^t_i\neq 0)\leq B s$, we conclude that
$$
\begin{aligned}
   & \quad\left|\text{supp}(\tilde{\beta}^t) - \text{supp}(\beta^*) \right| = \sum_{i\notin S^*}\mathbf{1}\{ | \tilde H^{t}_{\beta,i}|\ge  \lambda_{\beta} \}\\
    &\overset{(i)}\leq \sum_{i\notin S^*} \mathbf{1} \left\{  \left|\langle \Phi_{\cdot i}, \beta^* - \tilde{\beta}^t \rangle + \frac{\eta}{\sqrt n}\langle X_{\cdot i}, \theta^* - \tilde{\theta}^t \rangle \right| \ge  \lambda_\beta - \| \eta\Xi\|_\infty \ge \frac{3\kappa + \delta}{4 \kappa} \lambda_\beta \right\}\\
    &=o(s),
\end{aligned}
$$
where inequality follows from \eqref{absurd beta5}, and the last equation comes from a similar process in \eqref{eq: almost full}. 
Therefore, by taking the stopping time $t \ge  C_2 \log n $ we have
$$
\mathbf P \left\{ \left | \text{supp}(\beta^*)\  \triangle \ \text{supp}(\tilde{\beta}^t)\right | = o(s) \right\} \ge 1 - O(p^{-2} + n^{-3}),
$$
which completes the proof of Corollary 1.

\section{Proof of Theorem 3}\label{supp: Th3}
Throughout this proof, Assumption~2 is used in the concrete form
\[
n\ge C_{\mathrm{Th3}}\{s\log p+o\log n\},
\]
where \(C_{\mathrm{Th3}}  \) is a fixed constant depending only on \(M,\eta,\kappa\), chosen large enough to dominate the concrete sample-size constants used in the proof of Theorem~1, Lemma~\ref{lemma: oracle error}, and Lemma~\ref{lemma: rrip}, satisfying:
\[
C_{\mathrm{Th3}}
\ge \max \left\{ C_{\mathrm{Th1}},~30M^2 C_{\Sigma}, ~ 6M C_M^2, ~
\frac{16C_M^2}{(\kappa-\delta)^4} \right\}.
\]
The proof of Theorem~3 proceeds in three steps: 
First, we introduce some useful preliminaries.
Second, by mathematical induction, we prove that, under beta-min and theta-min conditions, i.e., Assumptions 3 and 4, the output of Algorithm 2 converges to the oracle estimation $\beta^\dag$. 
Finally, we derive the variable selection consistency.

\subsection{Preliminary}
Assume $supp(\theta^*) \ne \emptyset$. For ease of exposition, we will denote $(O^*)^c$ as $-O^*$ throughout the following proof.
Recall the formula of the oracle estimation:
\begin{equation}\label{oracle solution}
    \begin{aligned}
        \beta^{\dag}_{S^*} &:= \beta^{*}_{S^*} + (X^{\top}_{-O^*,S^*}X_{-O^*,S^*})^{-1}X_{-O^*,S^*}^{\top}\xi_{-O^*} \in \mathbb R^{|S^*|}, \\
        \beta^{\dag}_{-S^*} &:= \mathbf{0} \in \mathbb R^{p-|S^*|},\\
         \theta^{\dag}_{O^*} &:= \theta^{*}_{O^*} + \frac{1}{\sqrt{n}} \Big\{ 
 \xi_{O^*} - X_{O^*,S^*}(X_{-O^*,S^*}^{\top}X_{-O^*,S^*})^{-1}X^{\top}_{-O^*,S^*}\xi_{-O^*} \Big\}  \in \mathbb R^{|O^*|},\\
 \theta^{\dag}_{-O^*} &:= \mathbf{0} \in \mathbb R^{n-|O^*|}.
    \end{aligned}
\end{equation}

Therefore, we can rewrite the decomposition in the second-stage algorithm as: 
\begin{equation}
    \begin{aligned}
        \tilde H_{\beta}^{t+1} =& \tilde{\beta}^t + \frac{\eta}{n}X^{\top}(Y - X \tilde{\beta}^t - \sqrt{n} \tilde{\theta}^t) \\
= & \beta^{\dag} + \Phi(\beta^{\dag} - \tilde\beta^{t}) + \frac{\eta}{\sqrt{n}} X^{\top}(\theta^{\dag} - \tilde\theta^{t}) + \eta\Xi^{\dag}, \\
\tilde H_{\theta}^{t+1} =& \tilde \theta^t + \frac{\eta}{\sqrt{n}}(Y - X \tilde \beta^t - \sqrt{n}\tilde\theta^t) \\
= &\theta^{\dag} + (\eta-1) (\theta^\dag - \tilde \theta^t) + \frac{\eta}{\sqrt n} X  (\beta^{\dag} - \tilde\beta^{t}) + \frac{\eta}{\sqrt{n}}\xi^{\dag},   
    \end{aligned}
\end{equation}
where 
$$
\begin{aligned}
    \Xi_{S^*}^\dag &:= \mathbf{0} \in \mathbb R^{|S^*|},\\
    \Xi_{-S^*}^\dag &:= \frac{1}{n}X_{-O^*,-S^*}^{\top} \Big\{  I_{n-o} - X_{-O^*,S^*}(X_{-O^*,S^*}^{\top}X_{-O^*,S^*})^{-1}X_{-O^*,S^*}^{\top} \Big\} \xi_{-O^*} \in \mathbb R^{p-|S^*|},\\
    \xi_{O^*}^\dag &: = \mathbf{0} \in \mathbb R^{|O^*|},
\end{aligned}
$$
and
$$
\begin{aligned}
    \xi_{-O^*}^\dag &:= \Big\{  I_{n-o} - X_{-O^*,S^*}(X_{-O^*,S^*}^{\top}X_{-O^*,S^*})^{-1}X_{-O^*,S^*}^{\top} \Big\} \xi_{-O^*} \in \mathbb R^{n-|O^*|}.
\end{aligned}
$$

We specify Assumptions 3 and 4 as
\begin{equation}\label{eq: betathetamin}
\begin{aligned}
    \min_{i \in S^*} |\beta_i^*| \ge  \frac{4 \kappa}{\kappa - \delta}  \lambda_\beta + 5 \sqrt{\frac{M \log p}n}  , \quad 
    \min_{k \in O^*} |\theta_k^*| \ge \frac{4 \kappa}{\kappa - \delta}  \lambda_\theta + 4 \sqrt{\frac{\log n}n},
\end{aligned}
\end{equation}
where $\lambda_\beta = \lambda_{\beta, \infty}$ and $\lambda_\theta = \lambda_{\theta, \infty}$ follows from \eqref{eq: lambda infty}.
Define the event 
$$
\begin{aligned}
    \mathcal E_{oracle}: = \begin{Bmatrix}
\|\Xi^{\dag} \|_\infty < 4 \sqrt{\frac{M \log p}{n} } , & 
\left\|\frac{\xi^{\dag} }{\sqrt{n}} \right\|_\infty < 3 \sqrt{\frac{\log n}n} ;& \\
\left\|\beta^{\dag} -\beta^{*}\right\|_\infty < 5 \sqrt{\frac{M \log p}n} , &
\left\|\theta^{\dag} -\theta^{*}\right\|_\infty < 4  \sqrt{\frac{\log n}n} .
\end{Bmatrix}
\end{aligned}
$$
And the following proof is based on the event $\mathcal E_{oracle} \cap \mathcal E_X$, which holds with a probability greater than $1- O(p^{-2} + n^{-3} )$ by Lemma \ref{lemma: oracle error} and Proposition 1. 

\subsection{Convergence to oracle estimation}
Similar to \eqref{eq: th6 beta-min}, under the event $\mathcal E_{oracle}$ and the beta-min condition in \eqref{eq: betathetamin}, we have
\begin{equation}\label{eq: th8 beta-min}
    \begin{aligned}
    &\sum_{i \in S^*} \left( \tilde H_{\beta,i}^{t+1} \right)^2 \mathbf1(| \tilde H_{\beta,i}^{t+1}|< \lambda_{\beta}) \\
\le & \sum_{i \in S^*} \lambda_{\beta}^2 \cdot \mathbf1 \left( |\beta^{\dag}_i| - \left|\langle \Phi_{\cdot i}, \beta^{\dag} - \tilde \beta^t \rangle +
    \frac{\eta}{\sqrt n}\langle X_{\cdot i}, \theta^{\dag} - \tilde \theta^t \rangle \right| < \lambda_{\beta} \right)\\
\le & \sum_{i \in S^*} \lambda_{\beta}^2 \cdot \mathbf1 \left( \left|\langle \Phi_{\cdot i}, \beta^\dag - \tilde \beta^t \rangle +
    \frac{1}{\sqrt n}\langle X_{\cdot i}, \theta^\dag -\tilde \theta^t \rangle \right| > |\beta^*_i| - \| \beta^\dag - \beta^*\|_\infty - \lambda_\beta \ge \frac{3\kappa + \delta}{\kappa - \delta}\lambda_{\beta} \right)\\
\le & \left(\frac{\kappa-\delta}{3\kappa + \delta }\right)^2 \sum_{i \in S^*} \left( \langle \Phi_{\cdot i}, \beta^\dag -\tilde \beta^t \rangle +
    \frac{\eta}{\sqrt n}\langle X_{\cdot i}, \theta^\dag - \tilde \theta^t \rangle \right)^2,
\end{aligned}
\end{equation}
leading that
\begin{equation}\label{eq: th8 upper bound1}
\begin{aligned}
&\quad \quad \|\tilde \beta^{t+1} - \beta^{\dag}\|_2 \\
&\le \Bigg\{  \sum_{i \in S^*}\left( - \tilde  H_{\beta,i}^{t+1} \mathbf1(| \tilde H_{\beta,i}^{t+1}|< \lambda_{\beta}) 
   + \langle \Phi_{\cdot i}, \beta^{\dag} - \tilde \beta^t \rangle + \frac{\eta}{\sqrt n}\langle X_{\cdot i}, \theta^{\dag} - \tilde \theta^t \rangle \right)^2  \\
&\quad \quad \left.+ \sum_{i \in \tilde S^{t+1}\setminus S^*}\left( \langle \Phi_{\cdot i}, \beta^{\dag} -\tilde \beta^t \rangle +
   \frac{\eta}{\sqrt n}\langle X_{\cdot i}, \theta^{\dag} - \tilde \theta^t \rangle + \eta \Xi_i^\dag \mathbf{1} \left( | \tilde H_{\beta,i}^{t+1}|\geq \lambda_{\beta} \right)  \right)^2 \right\}^{1/2}\\
&\le \sqrt{\sum_{i \in \tilde S^{t+1} \cup S^*} \langle \Phi_{\cdot i}, \beta^{\dag} - \tilde \beta^t \rangle^2 } 
    + \sqrt{\sum_{i \in \tilde S^{t+1} \cup S^*} \frac{\eta^2}n\langle X_{\cdot i}, \theta^{\dag} - \tilde \theta^t \rangle^2 } \\
&\quad + \left\{ \sum_{i \in S^*} \left(  \tilde H_{\beta,i}^{t+1} \right)^2 \mathbf1(| \tilde H_{\beta,i}^{t+1}|< \lambda_{\beta})\right.\\
&\qquad \quad  \left.+ \sum_{i \in \tilde S^{t+1}\setminus S^*} \left(\eta \Xi^{\dag}_i \right)^2  \mathbf{1} \left(  \left|\langle \Phi_{\cdot i}, \beta^{\dag} - \tilde \beta^t \rangle + \frac{\eta \langle X_{\cdot i}, \theta^{\dag} - \tilde \theta^t \rangle}{\sqrt n} \right| \ge \lambda_\beta- \|\eta \Xi^\dag \|_\infty \ge \frac{3\kappa + \delta}{\kappa - \delta} |\eta \Xi^{\dag}_i| \right) \right\}^{1/2}\\
&\leq \frac{ 4 \kappa  }{ 3 \kappa + \delta } \left(\delta \| \tilde \beta^t - \beta^{\dag}\|_2 +  \frac{2\kappa \eta}{\kappa - \delta}  \cdot f \| \tilde\theta^t - \theta^{\dag} \|_2 \right),
\end{aligned}
\end{equation}
where the second inequality follows from the event $\mathcal E_{oracle}$.
The last inequality follows from \eqref{eq: th8 beta-min}, the techniques \eqref{eq: inner products} and \eqref{eq: inner products 2}, and the sparsity result \eqref{support number5} derived in Theorem 2, which holds consistently for every $t \ge 0$ with a probability greater than $1- O\left(p^{-2} +n^{-3} \right)$.

Similarly, the event $\mathcal E_{oracle}$ and the theta-min condition in \eqref{eq: betathetamin} yield that
\begin{equation}\label{eq: th8 theta-min}
    \begin{aligned}
&\sum_{i \in O^*} \left( \tilde H_{\theta,i}^{t+1} \right)^2 \mathbf1(| \tilde H_{\theta,i}^{t+1}|< \lambda_{\theta})\\
\le & \sum_{i \in O^*} \lambda_{\theta}^2 \cdot \mathbf1 \left( \left| ({ \eta }-1) (\theta^\dag_i - \tilde \theta^t_i ) + \frac{ \eta }{\sqrt n} X_{i \cdot} (\beta^{\dag} - \tilde\beta^{t}) \right|  > |\theta_i^*| - \|\theta^\dag - \theta^* \|_\infty  - \lambda_\theta \ge \frac{3\kappa + \delta}{\kappa - \delta}\lambda_{\theta} \right)\\
\le & \left( \frac{\kappa - \delta}{3\kappa + \delta} \right)^2 \sum_{i \in O^*} \left( ({ \eta }-1) (\theta^\dag_i - \tilde \theta^t_i ) + \frac{ \eta }{\sqrt n} X_{i \cdot} (\beta^{\dag} - \tilde\beta^{t})  \right)^2,
\end{aligned}
\end{equation}
which leads to
\begin{equation}\label{eq: th8 upper bound2}
\begin{aligned}
&\quad \quad \|\tilde \theta^{t+1} - \theta^{\dag}\|_2 \\
&\le \Bigg\{  \sum_{i \in O^*}\left( - \tilde H_{\theta,i}^{t+1} \mathbf1(| \tilde H_{\theta,i}^{t+1}|< \lambda_{\theta}) 
   + ({ \eta }-1) (\theta^\dag_i - \tilde \theta^t_i ) + \frac{ \eta }{\sqrt n} X_{i \cdot} (\beta^{\dag} - \tilde\beta^{t}) \right)^2  \\
&\quad \quad \left.+ \sum_{i \in \tilde O^{t+1}\setminus O^*}\left( ({ \eta }-1) (\theta^\dag_i - \tilde \theta^t_i ) + \frac{ \eta }{\sqrt n} X_{i \cdot} (\beta^{\dag} - \tilde\beta^{t}) + \frac{\eta}{\sqrt{n}}\xi_i^\dag \mathbf{1} \left( | \tilde H_{\theta,i}^{t+1}|\geq \lambda_{\theta} \right)  \right)^2 \right\}^{1/2}\\
&\le \sqrt{\sum_{i \in \tilde O^{t+1} \cup O^*} \left( ({ \eta }-1) (\theta^\dag_i - \tilde \theta^t_i ) + \frac{ \eta }{\sqrt n} X_{i \cdot} (\beta^{\dag} - \tilde\beta^{t}) \right)^2 }\\
&\quad + \left\{ \sum_{i \in O^*} \left( \tilde H_{\theta,i}^{t+1} \right)^2 \mathbf1(| \tilde H_{\theta,i}^{t+1}|< \lambda_{\theta}) \right.\\
&\qquad \quad \left.+ \sum_{i \in  \tilde  O^{t+1}\setminus O^*} \left( \frac{\eta \xi^{\dag}_i}{\sqrt n}\right)^2   \mathbf{1} \left( \left| ({ \eta }-1) (\theta^\dag_i - \tilde \theta^t_i ) + \frac{ \eta }{\sqrt n} X_{i \cdot} (\beta^{\dag} - \tilde\beta^{t}) \right| \ge \lambda_\theta -  \left|\frac{\eta \xi^{\dag}_i}{\sqrt{n}}\right| \ge \frac{3\kappa + \delta}{\kappa - \delta} \left|\frac{\eta \xi^{\dag}_i}{\sqrt{n}}\right| \right) \right\}^{1/2}\\
&\leq \frac{4 \kappa }{3 \kappa + \delta } \left( \delta \| \tilde \theta^t - \theta^{\dag}\|_2  + \frac{2 \kappa \eta}{\kappa - \delta} \cdot f \| \tilde \beta^t - \beta^{\dag}\|_2 \right).
\end{aligned}
\end{equation}
Combining \eqref{eq: th8 upper bound1} and \eqref{eq: th8 upper bound2}, we conclude that
\begin{equation}\label{eq: dynamic sol}
\begin{aligned}
\|\tilde \beta^{t} - \beta^{\dag}\|_2
\le&  \left(\frac{4 \kappa \delta }{3 \kappa + \delta } + \frac{8 \kappa^2 \eta \cdot f }{(3 \kappa + \delta)(\kappa - \delta) }  \right)^{t} \cdot \frac{ \|\tilde \beta^{0} - \beta^{\dag}\|_2 + \|\tilde \theta^{0} - \theta^{\dag}\|_2}2\\
&+ \left(\frac{4 \kappa \delta }{3 \kappa + \delta } - \frac{8 \kappa^2 \eta \cdot f }{(3 \kappa + \delta)(\kappa - \delta) }  \right)^{t} \cdot \frac{ \|\tilde \beta^{0} - \beta^{\dag}\|_2 - \|\tilde \theta^{0} - \theta^{\dag}\|_2}2, \\
\le & \left(\frac{4 \kappa \delta }{3 \kappa + \delta } + \frac{8 \kappa^2 \eta \cdot f }{(3 \kappa + \delta)(\kappa - \delta) }  \right)^{t} \cdot\left( \|\tilde \beta^{0} - \beta^{\dag}\|_2 + \|\tilde \theta^{0} - \theta^{\dag}\|_2\right)
\end{aligned}
\end{equation}
simultaneously hold for every $t \ge 0$.
Define $C_0:= \|\tilde \beta^{0} - \beta^{\dag}\|_2 + \|\tilde \theta^{0} - \theta^{\dag}\|_2 < \infty$ and $r : =\frac{4 \kappa \delta }{3 \kappa + \delta } + \frac{8 \kappa^2 \eta \cdot f }{(3 \kappa + \delta)(\kappa - \delta) } $, by the definition of $f$ in \eqref{eq: B delta} and the sample size assumption $n \ge \frac{16 C_M^2}{(\kappa - \delta)^4}\cdot ( s\log p + o \log n )$, we conclude that $r \le \frac{2 \kappa + 4\kappa \delta - 2 \delta}{3 \kappa + \delta} <1$.
Therefore, we prove that, with a probability greater than $1 - O(p^{-2} + n^{-3})$,
\begin{equation}\label{eq: final sol}
\| \tilde{\beta}^t - \beta^\dag \|_2 \leq C_0 \times  r^t
\end{equation}
holds for every $t \ge 0$ in the second-stage algorithm.

\subsection{Oracle estimation rate}
According to \eqref{oracle solution}, with a probability greater than $1- O(p^{-2})$, we have
\begin{equation}\label{oracle rate 1}
\begin{aligned}
    \|\beta^{\dag} - \beta^*\|_2
    &=
    \left\|
    (X^{\top}_{-O^*,S^*}X_{-O^*,S^*})^{-1}
    X_{-O^*,S^*}^{\top}\xi_{-O^*}
    \right\|_2 \\
    &\leq
    \frac{3M}{n}
    \cdot
    \left\|X_{-O^*,S^*}^{\top}\xi_{-O^*}\right\|_2,
\end{aligned}
\end{equation}
where the last inequality follows from Lemma~\ref{lemma: rrip}, which implies
\[
\left\|
(X^{\top}_{-O^*,S^*}X_{-O^*,S^*})^{-1}
\right\|_2
=
\frac{1}{
\Lambda_{\min}
\left(
X^{\top}_{-O^*,S^*}X_{-O^*,S^*}
\right)
}
\le
\frac{3M}{n}.
\]
For a given design $X$, from Theorem 2.1 of \citet{hsu2012tail}, we learn that 
\begin{equation}\label{eq: oracle prob}
\begin{aligned}
\mathbf P_{\xi | X}\Bigg\{ \left\|  X_{-O^*,S^*}^\top \xi_{-O^*} \right\|_2^2  
\ge& tr\left( X_{-O^*,S^*}^\top X_{-O^*,S^*} \right)+2\left\| X_{-O^*,S^*}^\top X_{-O^*,S^*} \right\|_F\sqrt{u}&\\
    &\quad +2\Lambda_{\max}\left( X_{-O^*,S^*}^\top X_{-O^*,S^*} \right) u ~ \Big | ~ X\Bigg\}  \le e^{-u}.
\end{aligned}
\end{equation}
Additionally, if $X \in \mathcal E_X$, we further conclude that
\begin{equation*}
\begin{aligned}
tr\left( X_{-O^*,S^*}^\top X_{-O^*,S^*} \right) \le&
tr\left( X_{\cdot,S^*}^\top X_{\cdot,S^*}  \right)  \le 2Mns,\\
\left\| X_{-O^*,S^*}^\top X_{-O^*,S^*} \right\|_F^2 \le &
s  \left\| X_{-O^*,S^*}^\top X_{-O^*,S^*} \right\|_2^2
\le s \left\| X_{\cdot,S^*}^\top X_{\cdot,S^*} \right\|_2^2 \le 4M^2 n^2 s,\\
\Lambda_{\max}\left( X_{-O^*,S^*}^\top X_{-O^*,S^*} \right)\le & 
\left\| X_{\cdot,S^*}^\top X_{\cdot,S^*} \right\|_2 \le 2M n.
\end{aligned}
\end{equation*}
Therefore, by taking $u= \log(1/ \varrho)$ into \eqref{eq: oracle prob}, we have
\begin{equation}\label{oracle rate 2}
\begin{aligned}
& \mathbf P_{\xi ,X} \left( \left\|  X_{-O^*,S^*}^\top \xi_{-O^*}  \right\|_2^2 \ge 2M n( 2s + 3  \log(1/ \varrho)) \right)\\
\le & \mathbf P_{\xi ,X} \left( \mathcal E_X^c\right)
+  \mathbf E_{X} \left\{ \mathbf1(\mathcal E_X)\cdot \mathbf P_{\xi |X} \left( \left\|  X_{-O^*,S^*}^\top \xi_{-O^*}  \right\|_2^2 \ge 2M n( 2s + 3  \log(1/ \varrho)) ~ \big| X \right)  \right\}\\
\le& 4p^{-2s} + \varrho.
\end{aligned}
\end{equation}
Combining \eqref{eq: dynamic sol}, \eqref{oracle rate 1} and \eqref{oracle rate 2},
with taking the stopping time $t \ge  \frac{\log(C_0 n) }{\log(1/r)}$, by \eqref{eq: final sol} we guarantee that 
\begin{equation}\label{eq: entry error}
\sup_{i \in [p]}|\beta_i^\dag - \tilde \beta_i^t| \le \|\tilde \beta^{t} - \beta^{\dag}\|_2 \le \frac1n,
\end{equation}
and 
$$
\| \tilde{\beta}^t - \beta^* \|_2
\le  \|\tilde \beta^{t} - \beta^{\dag}\|_2 + \|\beta^{\dag} - \beta^*\|_2 
~\le ~ \frac1n +  9M^{3/2} \sqrt{\frac{s + \log(1/\varrho)}{n}} 
~\asymp ~\sqrt{\frac{s+ \log(1/\varrho)}{n-o}},
$$
with a probability at least $1 -\varrho -  O(p^{-2} + n^{-3} )$.

\subsection{Variable selection consistency}
By \eqref{eq: entry error}, we conclude that:
\begin{itemize}
    \item For every $i \in S^*$, under the signal condition \eqref{eq: betathetamin} we have 
$$
|\tilde \beta^t_i| \ge |\beta_i^*|  - |\beta_i^\dag- \beta_i^*| - |\beta_i^\dag - \tilde \beta_i^t| \ge  \frac{4 \kappa}{\kappa - \delta} \lambda_{\beta} - \frac1n > 0,
$$
which proves that the whole support set $S^*$ is recovered.

\item For every $i \notin S^*$, we have 
$$
|\tilde \beta^t_i| \le |\beta_i^\dag| + |\beta_i^\dag - \tilde \beta_i^t| \le  \frac1n < \lambda_\beta,
$$
which proves that we cannot discover any entry on $(S^*)^c$ by using the hard-thresholding parameter $\lambda_\beta$.
\end{itemize}
Therefore, with a probability greater than $1- O(p^{-2} + n^{-3})$, we have
$$ 
\text{supp}(\tilde \beta^t ) = \text{supp} (\beta^\dag) = \text{supp} (\beta^*), \quad \text{for every } t \ge C \log n,
$$
which achieves the variable selection consistency of $\tilde \beta$.
Meanwhile, by utilizing the technique similar to \eqref{eq: dynamic sol}, we also establish the convergence of the estimator $\{\tilde \theta^t\}_{t \ge 0}$ to the oracle $\theta^\dag$, therefore demonstrating the selection consistency of $\tilde \theta$.

\section{Proof of Corollary 2}\label{supp: Cor2}
For every vector $\gamma\in \mathbb R^s$ with $\|\gamma\|_2<\infty$, according to Theorem 3, we have 
$$
\sqrt{n}\gamma^{\top} \tilde \beta_{S^*} = \sqrt{n}\gamma^{\top}\beta^{\dag}_{S^*} + o_p(1).
$$
Therefore, by Slutsky's Lemma, we only need to prove the asymptotic normality of $\beta^{\dag}_{S^*}$.
By \eqref{oracle solution}, 
$$
\begin{aligned}
&\sqrt n \gamma^\top ( \beta^{\dag}_{S^*} - \beta^*_{S^*}) \\
=& \sqrt n \gamma^\top \left(\frac{X^{\top}_{-O^*,S^*}X_{-O^*,S^*}}{n-o} \right)^{-1}   \frac{X^{\top}_{-O^*,S^*} \xi_{-O^*}}{n-o} \\
= & \frac{n}{n-o} \cdot \frac{1}{\sqrt n} \gamma^\top \Sigma_{S^*, S^*}^{-1} X_{-O^*,S^*}^{\top}\xi_{-O^*}\\
& + \sqrt n\gamma^\top \Sigma_{S^*, S^*}^{-1} \left\{ \Sigma_{S^*, S^*} - \frac{X^{\top}_{-O^*,S^*}X_{-O^*,S^*}}{n-o} \right\} \left(\frac{X^{\top}_{-O^*,S^*}X_{-O^*,S^*}}{n-o} \right)^{-1} \frac{X_{-O^*,S^*}^{\top}\xi_{-O^*}}{ n - o}\\
= & \frac{n}{n-o} \cdot\frac{1}{\sqrt n} \gamma^\top \Sigma_{S^*, S^*}^{-1} X_{-O^*,S^*}^{\top}\xi_{-O^*}  + O_p\left( \frac{ n  s \log p}{ (n-o)^{3/2}}\right) .
\end{aligned}
$$
where the last equality follows from \eqref{oracle rate 2} and a similar technique used in \eqref{eq: X Sigma}.
Define $\epsilon_i := \frac{1}{\sqrt n} \gamma^\top \Sigma_{S^*, S^*}^{-1} X^{\top}_{i,S^*}\xi_{i}$.
Clearly, $\{\epsilon_i\}_{i \in -O^*}$ are independent and zero-mean, and
\begin{equation}\label{Lyapunov 1}
\begin{aligned}
    \sum_{i\in -O^*} \text{Var}(\epsilon_i) =  \sum_{i\in -O^*}  \frac1n c_{\xi}\sigma^2\gamma^{\top}\Sigma_{S^*, S^*}^{-1} \gamma \geq \frac{c_{\xi}\sigma^2\|\gamma\|^2_2}{2M},
\end{aligned}
\end{equation}
where $c_{\xi} = \frac{\text{var}(\xi_i)}{\sigma^2}$ is a posotive constant. By the Cauchy-Schwarz inequality, we have
\begin{equation}
    \begin{aligned}
    \sum_{i\in - O^*}\mathbf{E} |\epsilon_i|^3 
    &= n^{-3/2}\sum_{i\in - O^*} \mathbf{E} \left( | \gamma^{\top} \Sigma_{ S^*,  S^*}^{-1} X^{\top}_{i, S^*} |^3 \right) \cdot \mathbf{E} |\xi_i|^3\\
    & \lesssim \frac{\|\gamma\|^3_2 \cdot \| \Sigma_{ S^*,  S^*}^{-1}\|_2^3 \cdot  \mathbf{E} (\|X^{\top}_{i, S^*} \|_2^3)  \cdot\sigma^3  }{n^{1/2}}  \\
    & \lesssim \frac{\sigma^3 s^{3/2}}{n^{1/2}} \|\gamma\|^3_2~ ,
\end{aligned}
\end{equation}
where the last inequality comes from the property of sub-Gaussian random variable. 
Since 
$$
\begin{aligned}
\frac{\sum_{i\in - O^*}\mathbf{E} |\epsilon_i|^3 }{ \left(\sum_{i\in -O^*} \text{Var}(\epsilon_i) \right)^{3/2}} \lesssim s^{3/2} n^{-1/2},
\end{aligned}
$$
by Lyapunov's central limit theorem and Slutsky's Lemma, with $n \succ s^3,~ n\gtrsim o \log n$, we conclude that
$$
 \frac{1}{\sqrt n} \gamma^\top \Sigma_{S^*, S^*}^{-1} X_{-O^*,S^*}^{\top}\xi_{-O^*}
 \overset{D}\to \mathcal{N} \left(0\ ,\ c_{\xi}\sigma^2\gamma^{\top} \Sigma_{S^*, S^*}^{-1} \gamma \right).
$$
Furthermore, combining with $\sqrt n \succ s \log p$, we get
$$
\sqrt{n}\gamma^{\top}( \tilde \beta_{S^*} - \beta^*_{S^*}) \overset{D}\to \mathcal{N} \left(0\ ,\ c_{\xi}\sigma^2\gamma^{\top} \Sigma_{S^*, S^*}^{-1} \gamma \right)
$$
for every $\gamma\in \mathbb R^s$ with $\|\gamma\|_2<\infty$.
Therefore, we complete the proof of Corollary 2.

\section{Proof of minimax lower bounds}\label{supp: minimax}
The proof of lower bounds proceeds in three steps.
We first follow the proof technique in \citet{chengaoren18aos} and get an estimation lower bound of the $\epsilon$-contamination model.
Then, we employ a Bayesian method to connect our empirical contamination model with the $\epsilon$-contamination model, thereby establishing Theorem 4.
We finally apply a similar technique to obtain the lower bound for variable selection, and complete the proof of Theorem 5.

For ease of display, we denote by $\pi_{\Sigma} (X_{i,\cdot})$ the distribution of the $i$-th observation $X_{i, \cdot} \in \mathbb R^p$ for each uncontaminated index $i$.
We also assume the noise $\xi_i$ follows from a Gaussian distribution $\mathcal N (0, \sigma^2)$ independently for each $i \in [n]$.

\subsection{Estimation lower bound related to $\epsilon'$-contamination model}\label{sec: bayes lower}
We assume $s$ is an even number, and construct two coefficient vectors as
\begin{equation}\label{eq: beta12}
\begin{aligned}
\beta^{(1)} & := (\underbrace{\lambda, \cdots, \lambda}_{s}~,~ \underbrace{0, \cdots, 0 }_{p-s})^\top \in \mathbb R^p,\\
\beta^{(2)} & := (\underbrace{0, \cdots, 0 }_{s/2}~, ~\underbrace{\lambda, \cdots, \lambda}_{s}~,~ \underbrace{0, \cdots, 0 }_{p- {3s}/2} )^\top \in \mathbb R^p,
\end{aligned}
\end{equation}
where $\lambda = \sqrt{\frac{1}{8C_{2s}}}\cdot \frac{\sigma}{\sqrt s} \cdot \frac{ o}{n}$.
Let $\mathbf P_{1}$ denote the joint distribution of $(X_{i\cdot}, Y_i)$ given $\beta^{(1)}$, where assume that $X_{i,\cdot} \in \mathbb R^{1 \times p}$ has marginal density function $\pi_\Sigma(\cdot)$ and, conditional on $X_{i\cdot}$ and $\beta^{(1)}$, $Y_i$ follows a normal distribution $\mathcal N( X_{i,\cdot} \beta^{(1)}, \sigma^2)$. 
Similarly, we define $\mathbf P_{2}$ as the joint distribution of $(X_{i\cdot}, Y_i)$ given $\beta^{(2)}$.
It is straightforward to learn that 
\begin{equation}
\begin{aligned}
0< \left\{ TV(\mathbf P_1 , \mathbf P_2 ) \right\}^2 ~\le& KL(\mathbf P_1 \| \mathbf P_2 )\\
    =& \frac1{2\sigma^2} \left( \beta^{(1)} - \beta^{(2)} \right)^\top \Sigma \left( \beta^{(1)} - \beta^{(2)} \right) \\
    \overset{(i)}\le& \frac{C_{2s}}{2\sigma^2} \cdot s \lambda^2\\
    = & \left( \frac{o}{4n} \right)^2 
    ~< ~ \left( \frac{o/(4n)}{1- o/(4n)} \right)^2,
\end{aligned}
\end{equation}
where inequality (i) follows from $\sup_{S\subset [p]: |S| \le 2s} \Lambda_{1}(\Sigma_{S,S}) \le C_{2s}$ and the construction in \eqref{eq: beta12}.
We then define $\epsilon' := \frac{TV(\mathbf P_1 , \mathbf P_2 )}{1+TV(\mathbf P_1 , \mathbf P_2 )}$, and it is directly to learn that $0 < \epsilon'  < o/(4n)$.

Define density functions
$$
p_1=\frac{\mathrm d \mathbf P_{1}}{\mathrm d\left(   \mathbf P_{1} + \mathbf P_{2}\right)}, \quad p_2=\frac{\mathrm d \mathbf P_{2}}{\mathrm d\left( \mathbf P_{1} + \mathbf P_{2}\right)}.
$$
Define $ \mathbf Q_{1} \text { and } \mathbf Q_{2}$ (both are the joint distributions of $(X_{i\cdot}, Y_i)$) by their density functions
\begin{equation}\label{eq: Qs}
\begin{aligned}
\frac{\mathrm d \mathbf Q_{1}}{\mathrm d \left( \mathbf P_{1} + \mathbf P_{2}\right)}& =\frac{\left( p_2 -p_1 \right) \cdot \mathbf1(p_2 \ge p_1) }{\operatorname{TV}\left(  \mathbf P_{1}, \mathbf P_{2} \right)} ,\\
\frac{ \mathrm  d \mathbf Q_{2}}{ \mathrm  d\left( \mathbf P_{1} + \mathbf P_{2}\right)}& =\frac{\left( p_1 -p_2 \right) \cdot \mathbf1(p_1 > p_2) }{\operatorname{TV}\left(  \mathbf P_{1}, \mathbf P_{2} \right)}.
\end{aligned}
\end{equation}
Following the proof of Theorem 5.1 in \citep{chengaoren18aos}, it is easy to check that both $ \mathbf Q_{1} \text { and } \mathbf Q_{2}$ are well-defined probability measures.
It can also be checked that 
\begin{equation}\label{eq: equal}
\begin{aligned}
\frac{\mathrm d\left( (1-\epsilon') \mathbf P_1 + \epsilon' \mathbf Q_1 \right) }{\mathrm d \left( \mathbf P_{1} + \mathbf P_{2}\right)} 
& =\left(1-\epsilon'\right) p_{1}+\epsilon^{\prime} \frac{ (p_{2}-p_{1} ) \cdot \mathbf 1(p_2 \ge p_1)}{\epsilon' / (1-\epsilon')} \\
& =\left(1-\epsilon'\right) p_2 +\epsilon^{\prime} \frac{ (p_1 - p_2 ) \cdot \mathbf 1(p_1 > p_2)}{\epsilon' / (1-\epsilon')} \\
& = \frac{\mathrm d\left( (1-\epsilon') \mathbf P_2 + \epsilon' \mathbf Q_2 \right) }{\mathrm d \left( \mathbf P_{1} + \mathbf P_{2}\right)}. 
\end{aligned}
\end{equation}
We construct the prior of $\beta$ as
\begin{equation}\label{eq: pi beta}
    \pi_\beta(\beta = \beta^{(1)}) = \pi_\beta(\beta = \beta^{(2)}) = 1/2,
\end{equation}
where recall the construction of $\beta^{(1)}$ and $\beta^{(2)}$ in \eqref{eq: beta12}.
Therefore, for any $q \in [1,2]$, we derive that 
\begin{equation}\label{eq: bayes lower bound}
\begin{aligned}
&\inf_{\hat \beta} \mathbf E_{\beta \sim \pi_\beta} \mathbf E_{Y,X | \epsilon', \beta, \mathbf Q(\beta)} \left( \| \hat \beta - \beta\|_q^q \right)\\
\ge & \frac{s\lambda^q}{2^q} \cdot \inf_{\hat \beta} \frac12 \left\{ \left[ (1-\epsilon') \mathbf P_1 + \epsilon' \mathbf Q_1 \right]\left( \|\hat \beta - \beta^{(1)} \|_q \ge \frac{s^{1/q}\lambda }{2} \right)  \right.\\
& \quad \quad \quad \quad \quad~ \left. + \left[ (1-\epsilon') \mathbf P_2 + \epsilon' \mathbf Q_2 \right]\left( \|\hat \beta - \beta^{(2)} \|_q \ge \frac{s^{1/q}\lambda }{2} \right)  \right\}\\
\overset{(i)}\ge &  \frac{s\lambda^q}{8} \cdot \inf_{\hat \beta} \left\{ \left[ (1-\epsilon') \mathbf P_1 + \epsilon' \mathbf Q_1 \right]\left( \psi^*(\hat \beta) =2 \right)  
+ \left[ (1-\epsilon') \mathbf P_2 + \epsilon' \mathbf Q_2 \right]\left( \psi^*(\hat \beta) =1 \right) \right\}\\
= & \frac{1}{8 (8 C_{2s})^{q/2}} \cdot \frac{\sigma^q s^{1-q/2} o^q}{n^q} 
~\left( = \frac{s\lambda^q}{8} \right),
\end{aligned}
\end{equation}
where $\mathbf E_{Y,X | \epsilon', \beta, \mathbf Q(\beta)}$ denotes an expectation based on the joint distribution of $(X_{i, \cdot }, Y_i)_{i \in [n]}$, with each $(X_{i, \cdot }, Y_i) \sim (1- \epsilon') \mathbf P_\ell + \epsilon' \mathbf Q_\ell$ independently for $\beta = \beta^{(\ell)}$, and we use $\mathbf Q (\beta^{(\ell)}) = \mathbf Q_\ell$ to emphasize that in our construction the contamination distribution is partially determined by $\beta$. 
Additionally, in inequality (i) we define the selector $\psi^*(\hat \beta) = \arg \min_{\ell \in \{1,2\}} \| \hat \beta  - \beta^{(\ell)}\|_q$, and this inequality holds because that $\psi^*(\hat \beta) =2 \Rightarrow  \|\hat \beta  - \beta^{(1)} \|_q \ge \|\hat \beta  - \beta^{(2)} \|_q \Rightarrow s^{1/q} \lambda = \| \beta^{(1)}  - \beta^{(2)} \|_q \le 2 \|\hat \beta  - \beta^{(1)} \|_q$.
And the last equality follows from \eqref{eq: equal}.

\subsection{Estimation lower bound in Theorem 4}\label{sec: freq lower}
Section \ref{sec: bayes lower} gives a minimax estimation lower bound under the Huber contamination model $(X_{i, \cdot }, Y_i) \sim (1- \epsilon') \mathbf P + \epsilon' \mathbf Q $.
However, this lower bound \eqref{eq: bayes lower bound} cannot be applied directly in our frequentist-setting model.
Given $\beta, n, k, \sigma$ and $\pi_\Sigma(\cdot)$, we define the distribution class of $(X_i, Y_i)_{i \in [n]}$ as 
\begin{equation}
\begin{aligned}
\mathcal P_{\beta,k} =& \left\{ (n-k) \text{ observations are drawn from } \pi_\Sigma(X_{i,\cdot}) \times \mathcal N(Y_i, X_{i,\cdot} \beta, \sigma^2) , \right.\\
& \quad k \text{ observations are drawn from arbitrary } \mathbf Q \left. \right\} .
\end{aligned}
\end{equation}
Define $\mathcal M(\beta,o) : = \bigcup_{0 \le k \le o}\mathcal P_{\beta,k}$, and in this subsection, we focus on deriving a lower bound as 
\begin{equation}\label{eq: freq lower prototype}
\inf_{\hat \beta} ~\sup_{ \beta : \| \beta\|_0 \le s}~
\sup_{\mathbf R \in \mathcal M(\beta,o)} 
~\mathbf E_{\mathbf R} \left( \| \hat \beta- \beta\|_q^q \right)
\gtrsim \sigma^q s \left(\frac{ \log (ep/s)}{n}\right)^{q/2} +\sigma^q s^{1-q/2} \frac{  o^q}{n^q} ,
\end{equation}
where $\mathbf R$ is a joint distribution of $(X_{i\cdot}, Y_i)_{i\in [n]}$.

It is straightforward to apply the uncontaminated result \citep{Raskutti2011minimax, MN19} to get 
\begin{equation}
\begin{aligned}
	&\inf_{\hat \beta} ~\sup_{ \beta : \| \beta\|_0 \le s}~
	\sup_{\mathbf R \in \mathcal M(\beta,o)}\mathbf E_{\mathbf R} \left( \| \hat \beta- \beta\|_q^q \right)\\
\ge & \inf_{\hat \beta} ~\sup_{ \beta : \| \beta\|_0 \le s}~
\sup_{\mathbf R \in \mathcal P_{\beta, k=0}}\mathbf E_{\mathbf R} \left( \| \hat \beta- \beta\|_q^q\right)
\gtrsim \sigma^q s \left(\frac{ \log (ep/s)}{n}\right)^{q/2},
\end{aligned}
\end{equation}
for any $q \in [1,2]$, therefore we only need to derive the term $\sigma^q s^{1-q/2} \frac{o^q}{n^q}$.

We first get a lower bound as 
\begin{equation}\label{eq: post mean}
\begin{aligned}
&\inf_{\hat \beta} ~\sup_{\beta: \|\beta\|_0 \le s} ~\sup_{\mathbf R \in \mathcal M (\beta,o)} \mathbf E_{(X,Y) \sim\mathbf R} \left\| \hat \beta(Y,X) - \beta \right\|_q^q \\
= & \inf_{\hat \beta} ~\sup_{\beta: \|\beta\|_0 \le s} ~\sup_{k \le o} ~ \sup_{\mathbf R \in \mathcal P_{\beta,k} } \mathbf E_{(X,Y) \sim\mathbf R} \left\| \hat \beta(Y,X) - \beta \right\|_q^q\\
\overset{(i)}\ge &\inf_{\hat \beta}~ \mathbf E_{\beta \sim \pi_{\beta}} \mathbf E_{k \sim \pi_{k}^o} \mathbf E_{\mathbf R \sim U\left(\mathcal P_{\beta,k}(\mathbf Q(\beta)) \right)} \mathbf E_{(X,Y) \sim\mathbf R} \sum_{j\in[p]} \left| \hat \beta_j (Y,X) - \beta_j \right|^q\\
\overset{(ii)}\ge &\sum_{j\in [p]}~ \inf_{\hat \beta_j(Y,X)}~ \mathbf E_{\beta ,k} \mathbf E_{X,Y|\beta,k}  \left| \hat \beta_j(Y,X) - \beta_j \right|^q \\
\overset{(iii)}= &\sum_{j\in [p]}~ \inf_{\hat \beta_j(Y,X)}~ \mathbf E_{Y,X} \mathbf E_{\beta, k | Y,X}  \left| \hat \beta_j(Y,X) - \beta_j \right|^q \\
= &\sum_{j\in [p]}  \mathbf E_{Y,X}~ \mathbf E_{\beta,k|Y,X}  \left| \hat M_j(Y,X) - \beta_j \right|^q.
\end{aligned}
\end{equation}
Some remarks on \eqref{eq: post mean}:
\begin{enumerate}
    \item In inequality (i), we first construct a distribution $ \pi_k := Binomial(n, \epsilon')$ and denote $\pi_k^o$ the conditional distribution of $\pi_k$ given $0\le k\le o$.
Once $\beta (= \beta^{(\ell)})$ and $k$ are generated, we use $\mathbf Q(\beta)$ (derived in \eqref{eq: Qs}) to represent the contaminated distribution, and denote by $\mathcal P_{\beta, k }(\mathbf Q(\beta))$ the subset of $\mathcal P_{\beta,k}$, given the contamination $\mathbf Q = \mathbf Q(\beta)$. 
It is then clear that $|\mathcal P_{\beta, k }(\mathbf Q(\beta))| = \binom nk$, and we write $\mathbf R \sim U\!\bigl(\mathcal P_{\beta,k}(\mathbf Q(\beta))\bigr)$ to denote the uniform distribution over the collection $\mathcal P_{\beta,k}(\mathbf Q(\beta))$. 
Equivalently, each subset of size $k$ is chosen as the outlier set with probability $1/\binom{n}{k}$.

\item In inequality (ii), we write $\mathbf E_{\beta,k}$ as an abbreviation of $\mathbf E_{\beta \sim \pi_{\beta}} \mathbf E_{k \sim \pi_{k}^o}$.
And we write the abbreviation $\mathbf E_{X,Y|\beta,k}$ since the joint distribution of $(X,Y)$ is fully specified given $\beta$ and $k$:
$$
\mathbb P\left( X,Y | \beta = \beta^{(\ell)}, k \right)
= \frac{1}{\binom nk} \sum_{O\subset [n]: |O|=k} 
\left\{ \prod_{i \in O} Q(X_{i\cdot}, Y_i) \cdot \prod_{i \notin O}\pi_\Sigma(X_{i\cdot})\mathcal N \left(Y_i, X_{i\cdot}\beta^{(\ell)}, \sigma^2 \right)\right\},
$$
as we discussed in inequality (i).

\item In equality (iii), we denote by $\mathbf E_{\beta, k|Y,X}$ the conditional expectation based on the probability measure 
$$
\pi^o_{\beta, k|Y,X}(\beta^{(\ell)}, k):= \frac{ \mathbb P\left( \beta = \beta^{(\ell)}, k , X,Y\right) }{  \sum_{\ell' \in \{1,2\}, 0\le k' \le o}\mathbb P\left( \beta = \beta^{(\ell')}, k=k' , X,Y\right) },
$$
where
$$
\begin{aligned}
&\mathbb P\left( \beta = \beta^{(\ell)}, k , X,Y\right)\\
:=& \frac12 \times \frac{\binom nk (\epsilon')^k (1-\epsilon')^{n-k}}{\sum_{0\le k'\le o}\binom n{k'} (\epsilon')^{k'} (1-\epsilon')^{n-k'}}\\
&~\times \frac{1}{\binom nk} \sum_{O\subset [n]: |O|=k} 
\left\{ \prod_{i \in O} Q(X_{i\cdot}, Y_i) \cdot \prod_{i \notin O } \pi_\Sigma(X_{i\cdot})\mathcal N \left(Y_i, X_{i\cdot}\beta^{(\ell)}, \sigma^2 \right)\right\}.
\end{aligned}
$$

\item In the last equality, by Theorem 1.1 on page 228 in \citep{lehmann2006theory}, there exists a Bayes estimator $\hat M_j(Y,X)$ achieving the infimum.
Additionally, this estimator is better than the zero estimator, leading
\begin{equation}\label{eq: lq}
\begin{aligned}
|\hat M_j(Y,X)| \le& \left| \mathbf E_{\beta, k|Y,X} (\hat M_j(Y,X) - \beta_j)\right| + \left| \mathbf E_{\beta, k|Y,X} \beta_j  \right|\\
\overset{\text{(Jensen)}}{\le}& \left( \mathbf E_{\beta, k|Y,X}  \left|\hat M_j(Y,X) - \beta_j \right|^q \right)^{1/q} + \left( \mathbf E_{\beta, k|Y,X} \left| \beta_j  \right|^q  \right)^{1/q}\\
\le& 2 \left( \mathbf E_{\beta, k|Y,X} \left| \beta_j  \right|^q  \right)^{1/q}.
\end{aligned}
\end{equation}

\end{enumerate}

Therefore, we bridge the $\epsilon'$-contamination model \eqref{eq: bayes lower bound} and our frequentist-setting model for any $q \in [1,2]$: 
\begin{equation}\label{eq: bridge}
\begin{aligned}
&\inf_{\hat \beta} \mathbf E_{\beta \sim \pi_\beta} \mathbf E_{Y,X | \epsilon', \beta, \mathbf Q(\beta)} \left( \| \hat \beta - \beta\|_q^q \right)\\
=& \inf_{\hat \beta} \mathbf E_{\beta \sim \pi_{\beta} } \mathbf E_{k\sim \pi_{k}}  \mathbf E_{\mathbf R \sim U\left(\mathcal P_{\beta,k}(\mathbf Q(\beta)) \right)} \mathbf E_{(X,Y) \sim\mathbf R}\sum_{j\in [p]} \left| \hat \beta_j (Y,X) - \beta_j \right|^q\\
\le &  \mathbf E_{\beta \sim \pi_{\beta} } \mathbf E_{k\sim \pi_{k}}  \mathbf E_{\mathbf R \sim U\left(\mathcal P_{\beta,k}(\mathbf Q(\beta)) \right)} \mathbf E_{(X,Y) \sim\mathbf R} \left\{ \mathbf 1 ( k\le o)\cdot \sum_{j\in [p]} \left| \hat M_j (Y,X) - \beta_j \right|^q \right\}\\
& + \mathbf E_{\beta \sim \pi_{\beta} } \mathbf E_{k\sim \pi_{k}}  \mathbf E_{\mathbf R \sim U\left(\mathcal P_{\beta,k}(\mathbf Q(\beta)) \right)} \mathbf E_{(X,Y) \sim\mathbf R} \left\{ \mathbf 1 ( k > o) \cdot \sum_{j\in [p]} \left| \hat M_j (Y,X) - \beta_j \right|^q \right\}\\ 
\overset{\text{(Jensen)}}\le & \mathbf E_{\beta \sim \pi_{\beta} } \mathbf E_{k\sim \pi_{k}^o}  \mathbf E_{\mathbf R \sim U\left(\mathcal P_{\beta,k}(\mathbf Q(\beta)) \right)} \mathbf E_{(X,Y) \sim\mathbf R} \left\{ \sum_{j\in [p]} \left| \hat M_j (Y,X) - \beta_j \right|^q \right\}\\
&+ 2^{q-1}\sum_{j\in [p]} \mathbf E_{\beta \sim \pi_{\beta} } \mathbf E_{k\sim \pi_{k}}  \mathbf E_{\mathbf R \sim U\left(\mathcal P_{\beta,k}(\mathbf Q(\beta)) \right)} \mathbf E_{(X,Y) \sim\mathbf R} \left\{  \mathbf 1 ( k > o) \cdot  \left( |\hat M_j (Y,X)|^q + |\beta_j|^q \right) \right\} \\
\le & \inf_{\hat \beta} ~\sup_{\beta: \|\beta\|_0 \le s} ~\sup_{\mathbf R \in \mathcal M (\beta,o)} \mathbf E_{(X,Y) \sim\mathbf R} \left\| \hat \beta(Y,X) - \beta \right\|_q^q
+ 15 s \lambda^q \cdot \pi_k (k > o),
\end{aligned}
\end{equation}
where the last inequality follows from \eqref{eq: post mean}, \eqref{eq: lq}, and also the fact
$$
 \mathbf E_{\beta, k|Y,X} \left( \left| \beta_j  \right|^q \right)
 \begin{cases}
\le \lambda^q  & \text{ if } j \le 3s/2, \\
=0  & \text{ if } 3s/2 <  j \le p,
\end{cases}
$$
leading by the construction of $\pi_\beta$ (\eqref{eq: beta12} and \eqref{eq: pi beta}).

By Bernstein's inequality (Appendix D.4 in \citet{FML2018}), we have 
$$
\pi_k (k > o) \le \exp\left(- \frac{\frac1n (o-n \epsilon')^2 }{2 \epsilon' + \frac23 \frac{o-n \epsilon'}{n}} \right)
\le \exp(-9o/16),
$$
where the last inequality follows from $\epsilon' \le o/(4n)$.
Therefore, combining \eqref{eq: bayes lower bound} and \eqref{eq: bridge}, for any $q \in [1,2]$, we have the lower bound
\begin{equation}
\begin{aligned}
&\inf_{\hat \beta} ~\sup_{\beta: \|\beta\|_0 \le s} ~\sup_{\mathbf R \in \mathcal M (\beta,o)} \mathbf E_{(X,Y) \sim\mathbf R} \left\| \hat \beta(Y,X) - \beta \right\|_q^q \\
\ge & s \lambda^q \left( \frac18 - 15 \exp(-9o/16) \right)\\
\ge& \frac{(8C_{2s})^{-q/2}}{80 } \cdot \sigma^q s^{1-q/2} \frac{  o^q}{n^q},
\end{aligned}
\end{equation}
where the last inequality follows from the assumption $ o \ge 9$. 
Hence, we complete the proof of Theorem 4.

\subsection{Selection lower bound in Theorem 5}
The proof of Theorem 5 is quite similar to that of Theorem 4.
Define 
$$
\mathcal B(s,a):= \left\{ \beta \in \mathbb R^p: ~ \|\beta\|_0 \le s , \min_{i: \beta_i \ne 0} |\beta_i| \ge a  \right\}.
$$
From \citet{WJM07rec} and Fano's inequality \citep{AT09}, we have
$$
\inf_{\hat S} ~\sup_{\beta \in \mathcal B \left(s, ~c_1 \sigma \sqrt{ (\log (ep/s))/n} \right) }~\sup_{\mathbf R \in \mathcal P_{\beta, k=0} } \mathbf E_\mathbf R \Big\{ \left| \hat S( X,Y) ~ \triangle~  \text{supp}(\beta)\right| \Big\} \ge c_2 s,
$$
where $\inf_{\hat S}$ denotes the infimum over all support set estimation $\hat S$ based on $(X,Y)$, and $c_1, c_2 >0$ are two absolute constants.
Therefore, we only need to prove the selection lower bound based on the parameter subspace $\mathcal B(s, \lambda)$ with $\lambda = \sqrt{\frac{1}{8C_{2s}}}\cdot \frac{\sigma}{\sqrt s} \cdot \frac{ o}{n}$.
We introduce the decoder $\eta = \eta (\beta) \in \{0,1 \}^p$, with each $\eta_i = \left\{ \eta(\beta)\right\}_i = \mathbf 1(\beta_i \ne 0)$.
Denote the selector $\hat \eta(X,Y) \in \{ 0,1\}^p$ as the estimator of $\eta(\beta)$.
By using the same technique in Section \ref{sec: bayes lower}, we have
\begin{equation}\label{eq: bayes selection lower bound}
\begin{aligned}
&\inf_{\hat \eta} \mathbf E_{\beta \sim \pi_\beta} \mathbf E_{Y,X | \epsilon', \beta, \mathbf Q(\beta)} \left(   \left\|  \hat \eta( X,Y)\ - \eta(\beta) \right\|_2^2  \right)\\
= &\inf_{\hat S} \mathbf E_{\beta \sim \pi_\beta} \mathbf E_{Y,X | \epsilon', \beta, \mathbf Q(\beta)} \left(  \left| \hat S( X,Y) ~ \triangle~  \text{supp}(\beta)\right|  \right)\\
\ge & \frac s2 \cdot \inf_{\hat \beta}~ \frac12\cdot \left\{ \left[ (1-\epsilon') \mathbf P_1 + \epsilon' \mathbf Q_1 \right]\left( \left| \hat S( X,Y) ~ \triangle~  \text{supp}(\beta^{(1)})\right| \ge \frac s2 \right)  \right.\\
& \quad \quad \quad \quad \quad~ \left. + \left[ (1-\epsilon') \mathbf P_2 + \epsilon' \mathbf Q_2 \right]\left( \left| \hat S( X,Y) ~ \triangle~  \text{supp}(\beta^{(2)})\right| \ge \frac s2 \right)  \right\}\\
\ge & \frac s4.
\end{aligned}
\end{equation}

We next analyze the selection lower bound in our setting.
Here we rewrite the prior $\pi_\beta$ as a prior on $\eta$, i.e., $\pi_\eta(\eta = \eta^{(\ell)}) = 1/2 , \quad \ell =1 ,2 $,
where 
\begin{equation}
\begin{aligned}
\eta^{(1)} & := (\underbrace{1, \cdots, 1}_{s}~,~ \underbrace{0, \cdots, 0 }_{p-s})^\top \in \mathbb R^p,\\
\eta^{(2)} & := (\underbrace{0, \cdots, 0 }_{s/2}~, ~\underbrace{1, \cdots, 1}_{s}~,~ \underbrace{0, \cdots, 0 }_{p- {3s}/2} )^\top \in \mathbb R^p.
\end{aligned}
\end{equation}
Thus $\beta^{(\ell)} = \lambda \eta^{(\ell)}$.
And we get two results similar to \eqref{eq: post mean} and \eqref{eq: bridge}:
\begin{equation}
\begin{aligned}
&\inf_{\hat \eta} ~\sup_{\beta \in \mathcal B(s, \lambda) } ~\sup_{\mathbf R \in \mathcal M (\beta,o)} \mathbf E_{(X,Y) \sim\mathbf R} \left\| \hat \eta(Y,X) - \eta(\beta) \right\|_2^2\\
\ge &\inf_{\hat \eta}~ \mathbf E_{\eta \sim \pi_{\eta}} \mathbf E_{k \sim \pi_{k}^o} \mathbf E_{\mathbf R \sim U\left(\mathcal P_{
\lambda \eta,k}(\mathbf Q(\lambda\eta)) \right)} \mathbf E_{(X,Y) \sim\mathbf R} \sum_{j\in[p]} \left( \hat \eta_j (Y,X) - \eta_j \right)^2\\
\ge &\sum_{j\in [p]}  \mathbf E_{Y,X}~ \mathbf E_{\eta,k|Y,X}  \left( \hat T_j(Y,X) - \eta_j \right)^2,
\end{aligned}
\end{equation}
and 
\begin{equation} 
\begin{aligned}
&\inf_{\hat \eta} \mathbf E_{\beta \sim \pi_\beta} \mathbf E_{Y,X | \epsilon', \beta, \mathbf Q(\beta)} \left( \| \hat \eta - \eta\|_2^2 \right)\\
\le & \mathbf E_{\eta \sim \pi_{\eta} } \mathbf E_{k\sim \pi_{k}^o}  \mathbf E_{\mathbf R \sim U\left(\mathcal P_{\lambda \eta,k}(\mathbf Q( \lambda \eta)) \right)} \mathbf E_{(X,Y) \sim\mathbf R} \left\{ \sum_{j\in [p]} \left( \hat T_j (Y,X) - \eta_j \right)^2 \right\}\\
&+ 2s \cdot  \mathbf E_{\eta \sim \pi_{\eta} } \mathbf E_{k\sim \pi_{k}}  \mathbf E_{\mathbf R \sim U\left(\mathcal P_{\lambda \eta,k}(\mathbf Q(\lambda \eta)) \right)} \mathbf E_{(X,Y) \sim\mathbf R} \left\{  \mathbf 1 ( k > o)  \right\} \\
\le & \inf_{\hat \eta} ~\sup_{\beta \in \mathcal B(s, \lambda)} ~\sup_{\mathbf R \in \mathcal M (\beta,o)} \mathbf E_{(X,Y) \sim\mathbf R} \left\| \hat \eta(Y,X) - \eta(\beta) \right\|_2^2
~+~ 2s \cdot \pi_k (k > o),
\end{aligned}
\end{equation}
where $\hat T_j (Y,X) := \mathbf E_{\eta,k|Y,X}(\eta_j) \in [0,1]$.
Therefore,
\begin{equation}
\begin{aligned}
&\inf_{\hat S} ~\sup_{\beta \in \mathcal B(s, \lambda)} ~\sup_{\mathbf R \in \mathcal M (\beta,o)} \mathbf E_{(X,Y) \sim\mathbf R} \left| \hat S( X,Y) ~ \triangle~  \text{supp}(\beta)\right| \\
\ge & \frac s4 -  2s \cdot   \exp(-9o/16) \\
\ge&  \frac s5,
\end{aligned}
\end{equation}
where the last inequality follows from $o \ge 8$, and thus we complete the proof of Theorem 5.

\section{Technical Lemmas}\label{supp: Technical lemmas}

\begin{lemma}[Support error control and element-wise error control]\label{lemma: supp error}
Assume \(\xi_i \sim SG(0,\sigma^2)\) independently for \(i\in[n]\).
Recall \(S^*=\operatorname{supp}(\beta^*)\).
Then, under Assumptions~1 and~2, we have
    \begin{equation}\label{eq: XiS*}
        \mathbf P \left(   \frac1n \left\|  X_{\cdot S^*}^\top \xi \right\|_2 \le  \sigma \sqrt{\frac{4Ms +6M \log(1/\varrho) }{n}} \right) \ge 1-2p^{-2s} - \varrho.
    \end{equation}
    Additionally, we have
\begin{equation}\label{eq: entryXi}
    \mathbf P\left( \sup_{i \in [p]} \left| \frac1n X_{\cdot i}^\top \xi \right| \le 4 \sigma \sqrt{\frac{M\log p}{n}} \right) >1-2p^{-2s} - 2 p^{-3},
\end{equation}
and 
\begin{equation}\label{eq: entryxi}
\mathbf P\left( \sup_{k \in [n]} \left| \frac1{\sqrt n} \xi_k \right| \le 3 \sigma \sqrt{\frac{ \log n}{n} }\right) > 1- 2 n^{-3}.
\end{equation}
\end{lemma}

\begin{proof}[Proof of Lemma \ref{lemma: supp error}]
In this lemma, Assumption~2 is used through Proposition~1 with \(C_1=1\), namely in the concrete form
\[
n\ge 30M^2C_\Sigma\max(s\log p,o\log n).
\]
Define 
\begin{equation}\label{eq: El1}
\mathcal E_{L1} = \left\{ \sup_{S \subset [p]:~ |S| \le s } ~ \max_{1\le k \le |S|}~\Lambda_k\left( \frac1n X_{\cdot S}^\top X_{\cdot S}\right) \le 2M\right\}
\end{equation}
Then by Proposition 1, $\mathbf P (\mathcal E_{L1}) \ge 1-2 \exp(-2s \log p)$.
The next two parts constitute the main proof. \\
\newline
{\bf For Equation \eqref{eq: XiS*}}\quad
By using Theorem 2.1 of \citet{hsu2012tail} and Assumption 1, for a given $X$, we have
\begin{equation*} 
  \mathbf P \left\{ \frac{\left\|  X_{\cdot S^*}^\top \xi \right\|_2^2}{n\sigma^2} \ge tr\left(\frac{X_{\cdot S^*}^\top X_{\cdot S^*}}n \right)+ 2\left\|\frac{X_{\cdot S^*}^\top X_{\cdot S^*}}n \right\|_F\sqrt{t}+2\Lambda_{\max}\left( \frac{X_{\cdot S^*}^\top X_{\cdot S^*}}n \right) t ~ \Bigg | X\right\} \le e^{-t}.
\end{equation*}
Combining with $|S^*|= s$, under event $\mathcal E_{L1}$,
\begin{equation*}
\begin{aligned}
tr\left(\frac{X_{\cdot S^*}^\top X_{\cdot S^*}}n \right) =&
 \sum_{k=1}^{s} \Lambda_k (\frac{X_{\cdot S^*}^\top X_{\cdot S^*}}n )
\le 2Ms,\\
\left\|\frac{X_{\cdot S^*}^\top X_{\cdot S^*}}n \right\|_F^2 = &
\sum_{k=1}^{s} \Lambda_k^2 \left( \frac{X_{\cdot S}^\top X_{\cdot S}}n \right)
\le 4M^2s.
\end{aligned}
\end{equation*} 
By taking $t= \log(1/\varrho)$ (where $\varrho \in (0,1)$), we obtain
\begin{equation*}
\begin{aligned}
&  \mathbf P \left(\frac{1}{n\sigma^2}\left\|  X_{\cdot S^*}^\top \xi \right\|_2^2 \ge 4Ms +6M \log(1/\varrho) \right)\\
\le&  \mathbf P (\mathcal E_{L1}^c) + \mathbf E_X \left\{ \mathbf1(\mathcal E_{L1}) \cdot \mathbf P \left(\frac{1}{n\sigma^2}\left\|  X_{\cdot S^*}^\top \xi \right\|_2^2 \ge 4Ms +6M \log(1/\varrho) ~\Big|~ X\right) \right\}\\
\le & 2p^{-2s} + \varrho,
\end{aligned}
\end{equation*}
which completes the proof of equation \eqref{eq: XiS*}.\\
\newline
{\bf For Equation \eqref{eq: entryXi} and \eqref{eq: entryxi}} \quad
For a given $X$, it is straightforward that $\Xi_i =  \frac1n X_{\cdot i}^\top \xi$ is sub-Gaussian random variable with sub-Gaussian parameter $ \sigma \|X_{\cdot i} \|_2 /n $ \citep{wainwright2019high}, leading that
\begin{align*}
&\mathbf P\left( \sup_{i \in [p]} \left| \frac1n X_{\cdot i}^\top \xi \right| > 4 \sigma \sqrt{\frac{M\log p}{n}} \right)\\
\le& \mathbf P (\mathcal E_{L1}^c) + \mathbf E_X \left\{ \mathbf1(\mathcal E_{L1}) \cdot \mathbf P \left( \sup_{i \in [p]} \left| \frac1n X_{\cdot i}^\top \xi \right| > 4 \sigma \sqrt{\frac{M\log p}{n}}  ~\Big|~ X\right) \right\}\\ 
\le& 2p^{-2s} + p\cdot 2e^{-4 \log p}
~=~ 2p^{-2s} + 2p ^{-3},
\end{align*}
where the second inequality applies the union bound and the fact $\max_{i \in [p]} \| X_{\cdot i} \|_2^2 \le 2Mn$ under the event $\mathcal E_{L1}$. 
A similar union bound yields that
\begin{equation*}
\mathbf P\left( \sup_{k \in [n]} \left| \frac1{\sqrt n} \xi_k \right| > 3 \sigma \sqrt{\frac{ \log n}{n} }\right)\le 2\exp\left(- 3\log n \right).
\end{equation*}
Therefore, we complete the proof of Lemma \ref{lemma: supp error}.
\end{proof}

Next, we control the error related to the oracle estimator. Recall that
\begin{align*}
    \beta^{\dag}_{S^*} &= \beta^{*}_{S^*} + (X^{\top}_{-O^*,S^*}X_{-O^*,S^*})^{-1}X_{-O^*,S^*}^{\top}\xi_{-O^*},\quad \beta^{\dag}_{-S^*} = \mathbf{0} , \\
    \theta^{\dag}_{O^*} &= \theta^{*}_{O^*} + \frac{1}{\sqrt{n}} \Big\{ 
 \xi_{O^*} - X_{O^*,S^*}(X_{-O^*,S^*}^{\top}X_{-O^*,S^*})^{-1}X^{\top}_{-O^*,S^*}\xi_{-O^*} \Big\},\quad \theta^{\dag}_{-O^*} = \mathbf{0} , \\
    \Xi_{S^*}^\dag &= \mathbf{0},\quad \xi_{O^*}^\dag = \mathbf{0},\\
    \Xi_{-S^*}^\dag &= \frac{1}{n}X_{-O^*,-S^*}^{\top} \Big\{  I_{n-o} -P_{-O^*,S^*} \Big\} \xi_{-O^*},\\
    \xi_{-O^*}^\dag &= \Big\{  I_{n-o} - P_{-O^*,S^*} \Big\} \xi_{-O^*},
\end{align*}
where we define 
\begin{equation*}
    P_{-O^*,S^*} :=  X_{-O^*,S^*}(X_{-O^*,S^*}^{\top}X_{-O^*,S^*})^{-1}X_{-O^*,S^*}^{\top} \in \mathbb R^{(n-o) \times (n-o)}.
\end{equation*}

\begin{lemma}[Oracle error control]\label{lemma: oracle error}
Assume that \(\xi_i\sim SG(0,\sigma^2)\) independently for \(i\in[n]\).
Under Assumptions~1 and~2, we have
\begin{equation}\label{eq: oracle error 1}
\begin{aligned}
\mathbf P\left( \sup_{i\in (S^*)^c}|\Xi^{\dag}_i| < 4 \sigma \sqrt{\frac{M \log p}{n} } \right) > &1- 2 p^{-3} - 2p^{-2s} , \\
\mathbf P\left( \sup_{i\in (O^*)^c}\left|\frac{\xi^{\dag}_i}{\sqrt{n}} \right| < 3\sigma \sqrt{\frac{\log n}n} \right) > & 1- 2n^{-3}.
\end{aligned}
\end{equation}
Additionally, under signal condition \eqref{eq: betathetamin}, there are element-wise estimation bounds on the oracle estimators $\beta^\dag$ and $\theta^\dag$:
\begin{equation}\label{eq: oracle error 2}
\begin{aligned}
        \mathbf P \left(\sup_{i\in S^*}\left|\beta^{\dag}_i -\beta^{*}_i\right| < 5 \sigma \sqrt{\frac{M \log p}n}  \right) &\ge 1- 4 p^{-2s} - 2 p^{-3},\\
       \mathbf P \left( \sup_{i\in O^*}\left|\theta^{\dag}_i -\theta^{*}_i\right| < 4 \sigma \sqrt{\frac{\log n}n} \right) &\ge 1- 4 p^{-2s}  - 2n^{-3} .
\end{aligned}
\end{equation}
\end{lemma}

\begin{proof}[Proof of Lemma \ref{lemma: oracle error}]
Two parts constitute the proof.\\
In this lemma, Assumption~2 is used in the concrete form
\[
n\ge C_{\mathrm{oracle}}\{s\log p+o\log n\},
\qquad
C_{\mathrm{oracle}}
\ge
\max\left\{
30M^2C_\Sigma,\,
6MC_M^2
\right\}.
\]
This condition covers the application of Proposition~1 with \(C_1=1\), and the minimum-eigenvalue bound in Lemma~\ref{lemma: rrip}.

{\bf For Equation \eqref{eq: oracle error 1}}\quad
For every $i \in (O^*)^c$, it is straightforward that $\xi_i^\dag = e_i^\top (I_{n-o} - P_{-O^*,S^*} )\xi_{-O^*}$ satisfies
\begin{equation}\label{eq: mgf referee}
\mathbf E_{\xi,X}\left( e^{\lambda\xi_i^\dag  } \right) 
\le  \mathbf E_{X}\exp\left( \frac{\lambda^2 \sigma^2}{2} \left\| e_i^\top  \left(I_{n-o} -P_{-O^*,S^*} \right) \right\|_2^2 \right) \le \exp\left( \frac{\lambda^2 \sigma^2}{2} \right)
\end{equation}
for all $\lambda \in \mathbb R$, where $e_i \in \mathbb R^{(n-o)\times 1}$ is a vector with 1 at the $i$-th position and 0 elsewhere.
Therefore, $\xi_i^\dag$ is a sub-Gaussian variable with sub-Gaussian parameter $\sigma$ \citep{wainwright2019high}, which leads that
\begin{align*}
\mathbf P_{\xi,X}\left( \sup_{i\in (O^*)^c}\left| \xi^{\dag}_i \right| \ge 3 \sigma \sqrt{ \log n } \right)
&\le \sum_{i\in (O^*)^c} 2 \exp\left(- \frac{9 \log n}{2} \right) 
~\le~ 2\exp\left(- 3\log n \right).
\end{align*}

Similarly, with a fixed design $X \in \mathbb R^{n \times p}$, for every $i \in (S^*)^c$,
$$
\mathbf E_{\xi |X}\left( e^{\lambda\Xi_i^\dag  } \right) 
\le \exp\left( \frac{\lambda^2 \sigma^2 }{2n^2} \left\| (I_{n-o}- P_{-O^*, S^*})  X_{-O^*, i } \right\|_2^2 \right) \le \exp\left( \frac{\lambda^2 \sigma^2}{2n^2} \left\| X_{\cdot, i } \right\|_2^2 \right),
$$
by the event $\mathcal E_{L1}$ (defined in \eqref{eq: El1}), we conclude that
\begin{align*}
\mathbf P\left( \sup_{i\in (S^*)^c}|\Xi^{\dag}_i| \ge 4 \sigma \sqrt{\frac{M \log p}{n} }\right)
& \le \mathbf P\left( \sup_{i\in (S^*)^c}|\Xi^{\dag}_i| \ge 4 \sigma \sqrt{\frac{M \log p}{n} } , X \in \mathcal E_{L1} \right) + \mathbf P (\mathcal E_{L1}^c) \\
&\le 2 p^{-3} + 2p^{-2s},
\end{align*}
therefore we complete the proof of \eqref{eq: oracle error 1}.\\
\newline
{\bf For Equation \eqref{eq: oracle error 2}}\quad
Similar to \eqref{eq: mgf referee}, with a fixed design $X \in \mathbb R^{n \times p}$, for every $i \in S^*$ and $\lambda \in \mathbb R$ we have
\begin{align*}
\mathbf E_{\xi|X}\left( e^{\lambda(\beta^{\dag}_i -\beta^{*}_i )} \right) 
\overset{(i)}\le& \exp\left( \frac{\lambda^2 \sigma^2}{2} \left\| e_i^\top (X^{\top}_{-O^*,S^*}X_{-O^*,S^*})^{-1}X_{-O^*,S^*}^{\top} \right\|_2^2 \right)\\
\le& \exp\left( \frac{\lambda^2 \sigma^2}{2}  \left\| (X^{\top}_{-O^*,S^*}X_{-O^*,S^*})^{-1}X_{-O^*,S^*}^{\top} \right\|_2^2 \right)\\
\overset{(ii)}=& \exp\left( \frac{\lambda^2 \sigma^2}{2}  \left\| (X^{\top}_{-O^*,S^*}X_{-O^*,S^*})^{-1} \right\|_2 \right)\\
=& \exp\left( \frac{\lambda^2 \sigma^2}{2}  \frac1{\Lambda_{\min} (X^{\top}_{-O^*,S^*}X_{-O^*,S^*}) } \right),
\end{align*}
where inequality (i) follows from the relationship \eqref{oracle solution}, inequality (ii) follows from $\|A\|_2^2 = \| A A^\top \|_2$.

Define the event
\[
\mathcal E_{L2} :=
\left\{
\left\| X_{O^*,S^*} \right\|_2
\le
C_M \sqrt{s \log p + o \log n },
\quad
\Lambda_{\min}\left(
X_{-O^*, S^*}^\top X_{-O^*, S^*}
\right)
\ge
\frac n{3M}
\right\}.
\]
Then, by Proposition~1 and Lemma~\ref{lemma: rrip}, we have
$
\mathbf P(\mathcal E_{L2}) \ge 1-4p^{-2s}.
$
Therefore, by the probability union bound, we have 
\begin{align*}
\mathbf P\left( \sup_{i\in S^*}\left|\beta^{\dag}_i -\beta^{*}_i\right| \ge 5 \sigma \sqrt{\frac{M \log p}n} \right)
&\le \mathbf P\left( \sup_{i\in S^*}\left|\beta^{\dag}_i -\beta^{*}_i\right| \ge 5 \sigma \sqrt{\frac{M \log p}n} ~ \big| \mathcal E_{L2} \right)\mathbf P( \mathcal E_{L2}) + \mathbf P( \mathcal E_{L2}^c )\\ 
& \le 2 p^{-3} +4 p^{-2s}.
\end{align*}

Additionally, with a fixed design $X \in \mathbb R^{n \times p}$, for every $i \in O^*$ and $\lambda \in \mathbb R$, we have
\begin{align*}
\mathbf E_{\xi|X}\left( e^{\lambda(\theta^{\dag}_i -\theta^{*}_i )} \right) 
=& \left\{ \mathbf E_\xi~ e^{\lambda \xi_i/\sqrt n} \right\} \cdot 
  \left\{  \mathbf E_{\xi|X}~ \exp\left( -\frac{\lambda}{\sqrt n} e_i^\top X_{O^*,S^*}(X_{-O^*,S^*}^{\top}X_{-O^*,S^*})^{-1}X^{\top}_{-O^*,S^*} \xi_{-O^*}\right)\right\} \\
\le& \exp\left( \frac{\lambda^2 \sigma^2}{2n}\right) \cdot \exp\left( \frac{\lambda^2 \sigma^2}{2n} \left\| e_i^\top X_{O^*,S^*} (X^{\top}_{-O^*,S^*}X_{-O^*,S^*})^{-1}X_{-O^*,S^*}^{\top} \right\|_2^2 \right)\\
\le& \exp\left\{ \frac{\lambda^2 \sigma^2}{2n} \left(1+ \left\| X_{O^*,S^*} \right\|_2^2 \left\|(X^{\top}_{-O^*,S^*}X_{-O^*,S^*})^{-1}\right\|_2 \right) \right\}\\
=& \exp\left\{ \frac{\lambda^2 \sigma^2}{2n} \left(1+ \frac{\left\| X_{O^*,S^*} \right\|_2^2}{\Lambda_{\min}\left(X^{\top}_{-O^*,S^*}X_{-O^*,S^*} \right) } \right) \right\}
\end{align*}

Therefore, under the assumption $n \ge 3 M C_M^2 (s \log p + o \log n)$, we have
\begin{align*}
\mathbf P\left( \sup_{i\in O^*}\left|\theta^{\dag}_i -\theta^{*}_i\right| \ge 4 \sigma \sqrt{\frac{\log n}n} \right)
& \le \mathbf P\left( \sup_{i\in O^*}\left|\theta^{\dag}_i -\theta^{*}_i\right| \ge 4 \sigma \sqrt{\frac{\log n}n} ~,~ X \in \mathcal E_{L2}\right) + \mathbf P\left(\mathcal E_{L2}^c \right) \\
&\le  2 n^{-3} + 4p^{-2s}.
\end{align*}
Therefore, we complete the proof of \eqref{eq: oracle error 2}, and hence the proof of Lemma \ref{lemma: oracle error} is completed.

\end{proof}

\begin{lemma}[Restricted minimum eigenvalue control]\label{lemma: rrip}
Recall $S^{*}=\operatorname{supp}(\beta^{*})$ and
$O^*=\operatorname{supp}(\theta^*)$.
Under Assumptions~1 and~2,
\[
\mathbf P\left(
\Lambda_{\min}\left(
\frac{X_{-O^*, S^*}^\top X_{-O^*, S^*}}{n}
\right)
\ge
\frac1{3M}
\right)
\ge
1-4p^{-2s}.
\]
\end{lemma}

\begin{proof}[Proof of Lemma~\ref{lemma: rrip}]
In this proof, Assumption~2 is used in the concrete form
\[
n\ge \max\left\{
30M^2C_\Sigma,\,
6MC_M^2
\right\} \cdot (s\log p+o\log n).
\]
This condition covers both the application of Proposition~1 with \(C_1=1\) and the bound
\begin{equation}\label{eq: lemma 3 sample}
C_M^2\frac{s\log p+o\log n}{n}
\le
\frac{1}{6M}.
\end{equation}
By Weyl's inequality and the decomposition $X_{\cdot,S^*}^\top X_{\cdot,S^*}
=X_{O^*,S^*}^\top X_{O^*,S^*}+
X_{-O^*,S^*}^\top X_{-O^*,S^*}$,
with a probability greater than $1- 4p^{-2s}$, we have
\[
\begin{aligned}
\Lambda_{\min}\left(
\frac{X_{-O^*,S^*}^\top X_{-O^*,S^*}}{n}
\right)
&\ge
\Lambda_{\min}\left(
\frac{X_{\cdot,S^*}^\top X_{\cdot,S^*}}{n}
\right)
-
\Lambda_{\max}\left(
\frac{X_{O^*,S^*}^\top X_{O^*,S^*}}{n}
\right) \\
&\ge
\frac{1}{2M}
-
\frac{\|X_{O^*,S^*}\|_2^2}{n} \\
&\ge
\frac{1}{2M}
-
C_M^2\frac{s\log p+o\log n}{n} \\
&\ge
\frac{1}{3M},
\end{aligned}
\]
where the second inequality follows from the restricted isometry in Proposition~1, the third inequality follows from the restricted incoherence in Proposition~1,  and the last inequality follows from the condition \eqref{eq: lemma 3 sample}.
This completes the proof of Lemma~\ref{lemma: rrip}.
\end{proof}

\section{Additional Details about the Extension to GLM}\label{supp: GLM}
In the Discussion section, we extend the two-stage AC-IHT algorithm to the setting of generalized linear models (GLMs). 
Compared to the linear model case, only the gradient descent step requires modification. Therefore, we begin by analyzing the decomposition of $H_{\beta}^{t+1}$ and $H_{\theta}^{t+1}$.

\subsection{Preliminary}
Let $X^{(i)} \in \mathbb R^{p \times 1}$ be the transformation of the $i$th row of X.
We have 
\begin{equation}\label{eq:GLMH_beta}
\begin{aligned}
H^{t+1}_{\beta} :=& \beta^t - \frac{ \eta }{n} \sum_{i=1}^n X^{(i)}\left(b'(\zeta_i^t) - Y_i \right)\\
=& \beta^* -(\beta^*-\beta^t) - \frac{ \eta }{n} \sum_{i=1}^n X^{(i)}\Big\{ \left(b'(\zeta_i^t) - b'(\zeta_i^*)\right)
    -\left(Y_i - b'(\zeta_i^*)\right) \Big\}\\
=& \beta^* + \left(\frac{ \eta }{n} \sum_{i=1}^n b''( \tilde\zeta_i^t ) X^{(i)} (X^{(i)})^{\top} - \mathbf I_p \right) \left( \beta^* -\beta^t\right) \\
    &+ \frac{{ \eta }}{\sqrt{n}}\left( \sum_{i=1}^n b''( \tilde\zeta_i^t)X^{(i)}e_i^\top\right)(\theta^* - \theta^t)+\frac{ \eta }{n} \sum_{i=1}^n X^{(i)}\left(Y_i-b'(\zeta_i^*)\right) \\
=& \beta^* + \tilde{\Phi}^t \left( \beta^* -\beta^t\right) + \frac{{ \eta }}{\sqrt{n}}(\tilde{X}^{t})^{\top}\left( \theta^* - \theta^t\right) +  { \eta }\tilde \Xi ,
\end{aligned}
\end{equation}

\begin{equation}\label{eq:GLMH_theta}
\begin{aligned}
H^{t+1}_{\theta} :=& \theta^t - \frac{ \eta }{\sqrt{n}} \sum_{i=1}^n e_i\left(b'(\zeta_i^t) - Y_i \right)\\
=& \theta^* -(\theta^*-\theta^t) - \frac{ \eta }{\sqrt{n}} \sum_{i=1}^n e_i\Big\{ \left(b'(\zeta_i^t) - b'(\zeta_i^*)\right)
    -\left(Y_i - b'(\zeta_i^*)\right) \Big\}\\
=& \theta^* + \frac{ \eta }{\sqrt{n}}\left( \sum_{i=1}^n b''( \tilde\zeta_i^t ) e_i (X^{(i)})^\top \right) \left( \beta^* -\beta^t\right) 
    + \left( { \eta }\sum_{i=1}^n b''( \tilde\zeta_i^t)e_ie_i^\top -\mathbf I_n\right)(\theta^* - \theta^t)\\
    &+\frac{ \eta }{\sqrt{n}} \sum_{i=1}^n e_i\left(Y_i-b'(\zeta_i^*)\right) \\
=& \theta^* + \frac{ \eta }{\sqrt{n}}\tilde{X}^{t} \left( \beta^* -\beta^t\right) + \tilde{D}^t\left( \theta^* - \theta^t\right) + \frac{ \eta }{\sqrt{n}}\tilde\xi ,
\end{aligned}
\end{equation}
where $\zeta^t_i := (X^{(i)})^\top \beta^t + 
\sqrt n e_i^\top\theta^t$, $\tilde\zeta_i^t = \lambda^t \zeta^t_i + (1-\lambda^t) \zeta^*_i $ for some $\lambda^t \in (0,1)$, and $\tilde\Phi^t := \frac{ \eta }{n} \sum_{i=1}^n b''(\tilde\zeta_i^t) X^{(i)} (X^{(i)})^\top -I_p \in \mathbb R^{p \times p}$, $\tilde{X}^t := \sum_{i=1}^n b''( \tilde\zeta_i^t)e_i(X^{(i)})^\top\in\mathbb R^{n \times p}$, $\tilde \Xi := \frac1{n} \sum_{i=1}^n X^{(i)}\left(Y_i-b'(\zeta^*_i)\right) \in \mathbb R^p$, \( e_i \in \mathbb{R}^{n\times 1} \) be the \( i \)-th standard basis vector, $\tilde{D}^t := { \eta }\text{diag}(b''( \tilde\zeta_1^t),\ldots,b''( \tilde\zeta_n^t)) - I_n = { \eta } \sum_{i=1}^n b''( \tilde\zeta_i^t) e_i e_i^\top - I_n  \in \mathbb R^{n \times n}$ and $\tilde \xi := \sum_{i=1}^n e_i\left(Y_i-b'(\zeta_i^*)\right)$. 
The decomposition of \( H^{t+1}_\beta \) and \( H^{t+1}_\theta \) reveals that the structure of \( H^{t+1}_\beta \) is similar to \eqref{eq: decomposition} in the proofs of Theorems~1 and~2.

\subsection{Useful lemmas}
The following lemmas establish that \( \tilde{X}^t \), \( \tilde{\Phi}^t \), \( \tilde \Xi \), and \( \tilde \xi \) retain the same properties as those in the proofs of Theorems~1 and~2.
\begin{lemma}[Proposition 1, GLM version]\label{lemma: ell2}
Under Assumptions~1, 2, and~5, for any fixed constant \(C_1>0\), with Assumption~2 imposed in the concrete form
\[
n\ge 30M^2C_1C_\Sigma\max(s\log p,o\log n),
\]
the following properties hold with probability greater than \(1-4\exp(-2C_1s\log p)\):
\begin{enumerate}
\item \textbf{(Restricted isometry)} For every index set $S \subset [p]$ with $|S| \le C_1 s$, the sample covariance matrix satisfies:
\begin{equation}\label{eq: rii GLM}
 \frac{L}{2M} \|u\|_2^2 ~\le ~ u^\top \left(  \frac{1}{n} \sum_{i=1}^n b''(\tilde\zeta_i^t) X^{(i)} (X^{(i)})^\top \right)_{S,S} u ~\le ~2MU \|u\|_2^2, ~\text{  for each } u \in \mathbb R^{|S|}.
\end{equation}
\item \textbf{(Restricted incoherence)} There exists a constant $C_M>0$ depending only on $M$ such that:
\begin{equation}\label{eq: inco GLM}
\sup_{S\subset[p]:~ |S| \le { C_1} s} ~
\sup_{O\subset[n]:~ |O| \le {C_1}o}
\left\| \tilde{X}^t_{O,S} \right\|_2 \le U { \sqrt{C_1}} C_M \sqrt{s \log p + o \log n}.
\end{equation}
\end{enumerate}
\end{lemma} 

\begin{proof}[Proof of Lemma \ref{lemma: ell2}]
Under Assumption 5, for every index set $S \subset [p]$ with $|S| \le C_1 s$, we have
$$
\begin{aligned} 
\frac{L}{n} \sum_{i=1}^n  u^\top \left( X^{(i)} (X^{(i)})^\top \right)_{S,S} u \le  \frac{1}{n} \sum_{i=1}^n b''(\tilde\zeta_i^t)~u^\top \left( X^{(i)} (X^{(i)})^\top \right)_{S,S} u
\le \frac{U}{n} \sum_{i=1}^n  u^\top \left( X^{(i)} (X^{(i)})^\top \right)_{S,S} u.
\end{aligned}
$$
Therefore, following \eqref{eq: weyl}, with a probability greater than $1 - 4p^{-2C_1 s}$ we get 
$$
\begin{aligned} 
\left\| \frac{1}{n} \sum_{i=1}^n b''(\tilde\zeta_i^t) \left( X^{(i)} (X^{(i)})^\top \right)_{S,S} \right\|_2
& \le U \left( \| \Sigma\|_2 + \frac1{2M} \right) \le 2MU .
\end{aligned}
$$
A similar technique shows that 
$$
\begin{aligned} 
\Lambda_{\min} \left( \frac{1}{n} \sum_{i=1}^n b''(\tilde\zeta_i^t) \left( X^{(i)} (X^{(i)})^\top \right)_{S,S} \right)
& \ge L \left( \Lambda_{\min} (\Sigma) - \frac1{2M} \right) \ge \frac{L}{2M},
\end{aligned}
$$
therefore we complete the proof of \eqref{eq: rii GLM}.

By the definition of $\tilde X^t$, for every index set $S \subset [p]$ with $|S| \le C_1 s$ and set $O \subset [n]$ with $|O| \le C_1 o$, we have
$$
\begin{aligned}
    \left\| \tilde{X}^t_{O,S} \right\|_2^2 
    &=  \left\| (\tilde{X}^t_{O,S})^\top \tilde{X}^t_{O,S}  \right\|_2 
    = \left\|\sum_{i\in O} (b''( \tilde\zeta_i^t))^2 X^{(i)}_S(X^{(i)}_S)^\top \right\|_2 \\
    &\overset{(i)}\le U^2 \left\| X_{O,S}^\top X_{O,S} \right\|_2\\
    &\le U^2 C_1 C_M^2 \cdot (s\log p + o \log n),
\end{aligned}
$$
where inequality (i) follows from Assumption 5, and the last inequality follows from the restricted incoherence result in Proposition 1, which holds with a probability greater than $1- 4p^{-2C_1 s}$. Therefore, we complete the proof of Lemma \ref{lemma: ell2}.
\end{proof}

\begin{lemma}\label{lemma:subgaussian}
Assume that GLM (5.1) holds with $\zeta_i^* = (X^{(i)})^\top \beta^* + \sqrt n e_i^\top\theta^* ,~\text{for every}\ i \in [n]$. Then based on Assumption 5, $\tilde\xi_i : =Y_i - b'(\zeta_i^*)$ is sub-Gaussian with zero mean and sub-Gaussian parameter $\sigma = \sqrt{aU}$, that is,
\begin{equation}
    \mathbf{P}\left( |Y_i- b'(\zeta_i^*)| \ge t\right)
    \le 2 \exp\left( -\frac{t^2}{2aU} \right), ~\text{for all}\ t \ge 0.
\end{equation}
\end{lemma} 
\begin{proof}[Proof of Lemma \ref{lemma:subgaussian}]
Note that for GLM (5.1), we have $\mathbf E_{Y_i | X^{(i)}}(Y_i) = b'(\zeta_i^*)$. From Theorem 5.10 of \citet{lehmann2006theory}, we also have
\begin{equation}
    \mathbf E_{Y_i | X^{(i)}}\left( \exp(\lambda Y_i) \right) = \exp\left( \frac{b(\zeta_i^* + \lambda a )- b(\zeta_i^*) }{a}\right),~ \forall \lambda \in \mathbb R,
\end{equation}
which leads to
\begin{equation}
\begin{aligned}
\mathbf E_{Y_i | X^{(i)}} \left( e^{\lambda (Y_i-b'(\zeta_i^*))} \right) 
&= \exp\left( \frac{b(\zeta_i^* + \lambda a )- b(\zeta_i^*) -\lambda a\cdot b'(\zeta_i^*)}{a}\right)\\
&\overset{(i)}{=} \exp\left( \frac{ \lambda^2 a b''(\tilde\zeta_i)}{2 }\right)\\
&\le \exp\left( \frac{ \lambda^2 a U}{2 }\right), ~\forall \lambda \in \mathbb R,
\end{aligned}
\end{equation}
where in equality (i), $\tilde \zeta_i$ is between $\zeta_i$ and $\zeta_i+ \lambda a$ based on Taylor's Theorem, and the last inequality follows from Assumption 5.
Hence, we prove that $Y_i - b'(\zeta_i^*)$ is sub-Gaussian with zero mean and sub-Gaussian parameter $\sigma = \sqrt{aU}$, and it satisfies the concentration inequality (see, e.g., \citet{wainwright2019high}):
\begin{equation}
    \mathbf{P}\left( |Y_i- b'(\zeta_i^*)| \ge t\right)
    \le 2 \exp\left( -\frac{t^2}{2aU} \right), ~\forall t \ge 0,
\end{equation}
which completes the proof of Lemma \ref{lemma:subgaussian}. 
\end{proof}

\subsection{Proof of Theorem 6}

We are now ready to present the proof of Theorem~6.
The proof is divided into three steps: \textbf{Step 1}: establishing the convergence property of the first-stage algorithm (Theorem 1, GLM version); \textbf{Step 2}: establishing the signal adaptive result of the second-stage algorithm (Theorem 2, GLM version).
Define $M ' : =\max \left( M , ~\frac{M}{L},~MU \right)$.

\textbf{Step 1 (First stage Algorithm 1, GLM version)}\quad 
Define the event
\begin{equation*} 
\mathcal E_{GLM}: = \left\{  \left\| \tilde \Xi  \right\|_\infty < 4\sqrt{\frac{aUM'\log p}{n}} , \text{ and }  \left\| \frac1{\sqrt n} \tilde \xi  \right\|_\infty < 3\sqrt{\frac{ aU \log n}{n} }\right\},
\end{equation*}
and
{\small
$$
\mathcal E_{X,GLM}:= \begin{Bmatrix}
\frac{1}{2M'} \le \Lambda_k\left( \frac{1}{n} \sum_{i=1}^n b''(\tilde\zeta_i^t) X^{(i)}_S (X^{(i)}_S)^\top \right) \le 2M',~ S \subset [p]:~ |S| \le (2B'+1)s \text{ and } 1\le k \le |S|,  \\
\sup_{S\subset[p]:~ |S| \le (B'+1) s} ~
\sup_{O\subset[n]:~ |O| \le (B'+1)o}
\left\| \tilde X_{O,S}^t \right\|_2 \le {(\sqrt {B'}+1)} f' \sqrt n 
\end{Bmatrix},
$$
}

\noindent
where
\begin{equation*}\label{eq: B delta GLM}
B':= \left(\frac{\kappa'+\delta'}{\kappa'- \delta'}\right)^2>1, \quad \delta' := \frac{4M'^2}{4M'^2 +1}\in (0,1), \quad f':= C_M U \sqrt{\frac{s \log p + o \log n}n},
\end{equation*}
and the decay rate $\kappa' \in \left(  \delta',1 \right)$. The event \(\mathcal E_{X,GLM}\) follows from Lemma~\ref{lemma: ell2} with \(C_1=2B'+1\) for the restricted isometry part and \(C_1=B'+1\) for the restricted incoherence part. These choices are fixed once \(M,L,U\) and \(\kappa'\) are fixed. 
By Lemmas \ref{lemma: ell2} and \ref{lemma:subgaussian}, we have $\mathbf P(\mathcal E_{GLM} \cap \mathcal E_{X,GLM}) \ge 1- O\left(p^{-2} + n^{-3}\right)$. 
Under the event $\mathcal E_{GLM} \cap \mathcal E_{X,GLM}$, we choose some proper parameters
\begin{equation}\label{eq: glm lambda}
\begin{aligned}
\lambda_{\beta, \infty} &\asymp \sqrt{aU} \left( \sqrt{\frac{ \log p}{n}} + \frac{o \log n}{n \sqrt s} \right), \\
\lambda_{\theta, \infty} &\asymp \sqrt{aU} \left(  \sqrt{\frac{\log n}{n}} + \frac{s \log p}{n \sqrt o} \right).
\end{aligned}
\end{equation}
Then, by choosing the learning rate $\eta' \in \left[\frac{2M'}{4M'^2 + 1}, \frac{4M'}{4M'^2 + 1} \right]$ and applying a proof process similar to that in Section \ref{sec: MI}, we ensure that, with probability at least $1 - O(p^{-2}+n^{-3})$, both estimators are $\ell_0$-sparse, i.e.,
$
\|\widehat{\beta}^{GLM}\|_0 \lesssim s,~\|\widehat{\theta}^{GLM}\|_0 \lesssim o,
$
and satisfies
\begin{equation*}
\begin{aligned}
\| \widehat{\beta}^{GLM} - \beta^* \|_2^2 \lesssim & aU\left( \frac{s \log p}{n} + \frac{o^2 \log^{2} n}{n^2} \right),\\
\| \widehat{\theta}^{GLM} - \theta^* \|_2^2 \lesssim& aU \left( \frac{o\log n}{n} + \frac{s^2 \log^{2} p}{n^2} \right). 
\end{aligned}
\end{equation*}

\textbf{Step 2 (Second-stage Algorithm 2, GLM version)}\quad  
Following Lemmas \ref{lemma: supp error} and \ref{lemma:subgaussian}, we have
\begin{align*}
\mathbf P\left\{\underbrace{ \frac1n \left\|  X_{\cdot S^*}^\top \tilde \xi \right\|_2 \le \sqrt{aUM} \sqrt{\frac{4 s +6 \log(1/\varrho)}n} }_{=:\mathcal E'_{GLM}} \right\}
\ge 1- O\left(p^{-2s} + \varrho \right).
\end{align*}
We fix the threshold parameters $\lambda_\beta$, $\lambda_\theta$ as in \eqref{eq: glm lambda}. 
Then, under the signal condition 
\begin{equation}\label{eq: betamin in th6}
\min_{i \in S^*} |\beta_i^*| \ge  \left( \frac{4 \kappa'}{\kappa'-\delta'} + \frac{\kappa'-\delta'}{4 \kappa'} \right) \cdot \lambda_\beta 
\asymp  \sqrt{aU} \left( \sqrt{\frac{ \log p}{n}} + \frac{o \log n}{n \sqrt s} \right),
\end{equation}
by applying a proof technique similar to that in Section \ref{sec: prove of th2-3}, we get the sharper bound 
\begin{equation*}
\| \tilde \beta^{GLM} - \beta^*\|_2 
\lesssim \sqrt{\frac{s +\log (1/\varrho)}n} + \frac{s\log p+ o \log n}{n},
\end{equation*}
with a probability greater than $1 - \varrho - O(p^{-2}+n^{-3})$.
Therefore, we complete the proof of Theorem 6, demonstrating the signal adaptivity of our procedure under the GLM setting.

\section{Proof of Theorem 7}\label{supp: heavy tail}
We first establish the estimation error bound for the regime where $\delta \in (0, 1)$.
Before proceeding to the main analysis, we provide some necessary probabilistic inequalities.

\paragraph{Concentration of $|\Xi_j'|$.} Define 
$$
\Xi_j' := \frac1n \sum_{i=1}^n X_{ij} \xi_i =  \frac1n \sum_{i=1}^n X_{ij} \psi_\tau(\epsilon_i), \text{ for every } j \in [p].
$$
For every $k\ge 2$, by the definition of $\psi_\tau(\cdot)$, we have the moment inequality
$$
\begin{aligned}
\mathbf E( |\xi_i|^k) =& \mathbf E\left\{ \tau^k \mathbf 1 ( |\epsilon_i| > \tau) \right\} + \mathbf E\left\{ |\epsilon_i|^k \mathbf 1 ( |\epsilon_i| \le \tau) \right\}\\
\le &  \mathbf E\left\{ \tau^k \frac{|\epsilon_i|^{1+\delta}}{\tau^{1+\delta}}\cdot \mathbf 1 ( |\epsilon_i| > \tau) \right\} + \mathbf E\left\{ |\epsilon_i|^k \frac{\tau^{k-1-\delta}}{|\epsilon_i|^{k-1-\delta}} \cdot \mathbf 1 ( |\epsilon_i| \le \tau) \right\}\\
\le& \tau^{k-1-\delta} \mathbf E\left\{ |\epsilon_i|^{1+\delta} \mathbf 1 ( |\epsilon_i| > \tau) + |\epsilon_i|^{1+\delta}\mathbf 1 ( |\epsilon_i| \le \tau) \right\}\\
\le & v_\delta \tau^{k-1-\delta}.
\end{aligned}
$$
By the sub-Gaussian property, we obtain
$$
\mathbf E ( |X_{ij}|^k ) 
= \int_{t \in \mathbb R^+} \mathbf P(|X_{ij}|^k > t) \mathrm d t
\le (2\Sigma_{jj})^{k/2} k \Gamma(k/2).
$$
Then, for every $\lambda$ satisfying $|\lambda| \le \frac{1}{2 \tau \sqrt{ 2 \Sigma_{jj} } }$, we get a Bernstein-type bound: 
$$
\begin{aligned}
\mathbf E(e^{\lambda X_{ij} \xi_j }) 
\le & 1 + \frac12 \lambda^2 \mathbf EX_{ij}^2 \xi_j^2 + \sum_{k\ge 3} \frac{\lambda^k}{k!} \mathbf E |X_{ij} \xi_j|^k\\
\le &  1 +  2\lambda^2\Sigma_{jj} v_\delta \tau^{1-\delta} + \sum_{k\ge 3} \frac{\lambda^k v_\delta \tau^{k-1-\delta} \cdot (\sqrt{2 \Sigma_{jj}})^k ~k!}{k!} \\
\le & 1+ 4\lambda^2\Sigma_{jj} v_\delta \tau^{1-\delta} \\
\le& \exp(4\lambda^2\Sigma_{jj} v_\delta \tau^{1-\delta}), \text{ for each } (i,j) \in [n] \times [p].
\end{aligned}
$$
Therefore, by the probability union bound, we get the concentration inequality
\begin{equation}\label{eq: heavy Xi}
\mathbf P \left( \max_{j \in [p]}\left[ |\Xi_j'| - \frac{v_\delta \sqrt{2 \Sigma_{jj}} }{\tau^\delta} - \frac{6 \sqrt{2 \Sigma_{jj}}\cdot \tau \log p }{n}\right] \ge 0 \right)   \le 2 p^{-2}.
\end{equation}

\paragraph{Concentration of $\|\theta^*\|_0$.} 
We next bound the term $\|\theta^*\|_0 = \sum_{i \in [n]} \mathbf 1(\epsilon_i - \psi_\tau(\epsilon_i) \ne 0 ) = \sum_{i \in [n]} \mathbf 1( |\epsilon_i| > \tau )$. 
Define $q:= \mathbf P( |\epsilon_i| > \tau ) \le \frac{v_\delta}{\tau^{1+\delta}}$.
By Bernstein's inequality (Appendix D.4 in \citet{FML2018}), we have 
\begin{equation}\label{eq: heavy o}
\begin{aligned}
\mathbf P\left( \| \theta^*\|_0 \ge \frac{2nv_\delta}{\tau^{1+\delta}}  \right)
\le &\mathbf P\left\{ \frac1n \sum_{i \in [n]} \mathbf1(|\epsilon_i| > \tau) \ge q+ \frac{v_\delta}{\tau^{1+\delta}} \right\} \\
\le &\exp\left( -\frac{n (v_\delta/\tau^{1+\delta})^2}{2q(1-q) + 2/3 \cdot v_\delta/\tau^{1+\delta}}\right)
\le \exp\left( -\frac{n v_\delta}{ 3\tau^{1+\delta} }\right).
\end{aligned}
\end{equation}

Take $\tau = \left( \frac{n v_\delta}{6 \log p}\right)^{\frac1{1+\delta}}$.
We now replace the event $\mathcal E$ \eqref{eq: good event of Th1} (in the proof of Theorem 1) by the newly defined $\mathcal E_{heavy}$:
$$
\mathcal E_{heavy}:= \begin{Bmatrix}
\left\| \Xi' \right\|_\infty \le \frac{v_\delta \sqrt{2 \max_j \Sigma_{jj}} }{\tau^\delta} + \frac{6 \sqrt{2 \max_j \Sigma_{jj}}\cdot \tau \log p }{n}, \\
o \le \frac{2nv_\delta}{\tau^{1+\delta}}, \text{ and }~\left\| \xi \right\|_\infty \le\tau
\end{Bmatrix}.
$$ 
By \eqref{eq: heavy Xi} and \eqref{eq: heavy o}, we have $\mathbf P(\mathcal E_{heavy}) \ge 1- 3p^{-2}$.

\paragraph{Near-optimal error bound.}

We first construct a suitable $\lambda_{\theta,0}$. According to \eqref{eq: lambda0s}, it suffices to ensure that
\begin{equation}\label{eq: lambda 0 heavy}
    \sqrt o \lambda_{\theta,0} > \|\theta^* \|_2 \vee (\sqrt o \lambda_{\theta, \infty}).
\end{equation}
Following the proof of Theorem 1, we aim to choose a proper $\lambda_{\theta, \infty}$ such that
$$
\sqrt o \lambda_{\theta, \infty}
\gtrsim \sqrt{o} \left\| \frac1{\sqrt n} \xi \right\|_\infty + \sqrt{\frac{s\log p + o \log n}{n}} \cdot \sqrt s \| \Xi' \|_\infty.
$$
Additionally,
$$
\begin{aligned}
&\sqrt{o} \left\| \frac1{\sqrt n} \xi \right\|_\infty + \sqrt{\frac{s\log p + o \log n}{n}} \cdot \sqrt s \| \Xi' \|_\infty\\
\lesssim & \tau \sqrt{\frac on} + \sqrt{\frac{s\log p + o \log n}{n}} \cdot \sqrt s \left( \frac{v_\delta  }{\tau^\delta} + \frac{\tau \log p }{n} \right)\\
\asymp& \tau \sqrt{\frac on} \left( 1 + \sqrt{\frac{s\log p}{n}} \sqrt{\frac{s\log p + o \log n}{n}} \right)\\
\asymp& \tau \sqrt{\frac on},
\end{aligned}
$$
where the second equality uses $o \asymp \log p$, and the last equality follows from the sample size condition $n \gtrsim (s+\log n)\log p$.
Therefore, it is enough to choose
$$
\lambda_{\theta,0} \gtrsim \frac1{\sqrt n} \left( \|\epsilon  \|_\infty \vee \tau \right) \asymp  \frac1{\sqrt n} \left( \|\epsilon  \|_\infty \bigvee \left( \frac{n v_\delta}{\log p}\right)^{\frac1{1+\delta}} \right), 
$$
which guarantees \eqref{eq: lambda 0 heavy}.
Moreover, by the union bound and the Markov inequality, for every $r >0$, 
$$ 
\mathbf P ( \|\epsilon  \|_\infty > r)
\le n \mathbf E\left\{  \mathbf1 \left( |\epsilon|^{1+\delta}>r^{1+\delta} \right)\right\} \le \frac{n v_\delta}{r^{1+\delta}} .
$$
Hence, by taking $\lambda_{\theta,0} \gtrsim n^{-1/2} (nv_\delta/\varrho)^{\frac1{1+\delta}}$, we ensure that, with probability at least $1-\varrho-O(p^{-2})$, the initial threshold $\lambda_{\theta,0}$ satisfies \eqref{eq: lambda 0 heavy}.

Following the proof of Theorems 1 and 2, with a probability greater than $1-\varrho- O(p^{-2})$, the output of our AC-IHT algorithm satisfies
$$
\begin{aligned}
\| \tilde \beta  - \beta^*\|_2 \lesssim& \sqrt s \|\Xi' \|_\infty + \sqrt{\frac{s\log p + o \log n}{n}} \cdot \sqrt{o} \left\| \frac1{\sqrt n} {\xi} \right\|_\infty\\
\lesssim & \frac{\sqrt s v_\delta}{\tau^\delta} + \frac{\sqrt s \log p}{n}\tau + \sqrt{\frac{v_\delta s \log p}{n} \tau^{1-\delta} } + \frac{v_\delta \sqrt{\log n}}{\tau^\delta} \\
\lesssim& \frac{ v_\delta \sqrt{s+\log n}}{\tau^\delta} + \frac{\sqrt{s+\log n} \log p}{n}\tau \\
\asymp & v_\delta^{\frac{1}{1+\delta}} \sqrt{s+\log n} \left( \frac{\log p}{n} \right)^{\frac{\delta}{1+\delta}},
\end{aligned}
$$
where the first two inequalities follow from Assumption 1, which imposes constant upper and lower bounds on the spectrum of $\Sigma$.

In the case $\delta \ge 1$, by the moments inequality, we have $v_1^{1/2} := (\mathbf E |\epsilon_i|^2)^{1/2} \le (\mathbf E |\epsilon_i|^{1+\delta})^{1/(1+\delta)} \le v_\delta^{1/(1+\delta)} $.
Then we take $\delta'=1$ and $v_{\delta'} = v_1$.
Following the proof sketch provided above, with a probability at least $1 - \varrho -O(p^{-2})$, we can choose $\lambda_{\theta,0} \gtrsim  \sqrt{v_1/\varrho }$ and get the upper bound
$$ 
\| \tilde \beta  - \beta^*\|_2 
\lesssim
v_{\delta'}^{\frac{1}{1+{\delta'}}} \sqrt{s+\log n} \left( \frac{\log p}{n} \right)^{\frac{{\delta'}}{1+{\delta'}}} 
= \sqrt{\frac{v_1 (s+ \log n) \log p}{n}}.
$$
Therefore, we complete the proof of Theorem 7, demonstrating the applicability of our algorithm to the high-dimensional heavy-tailed regression.

\end{appendix}

\end{document}